\documentclass[10pt,journal,compsoc]{IEEEtran}

\ifCLASSOPTIONcompsoc
\usepackage[nocompress]{cite}
\else
\usepackage{cite}
\fi

\ifCLASSINFOpdf
\else
\fi

\usepackage{bm,url}
\usepackage{times}
\usepackage{epsfig,subfigure}
\usepackage{graphicx}
\usepackage{amsmath,amsthm}
\usepackage{amssymb}
\usepackage{enumerate}
\usepackage[algoruled,vlined]{algorithm2e}
\usepackage{color}
\usepackage{tabu,multicol}
\usepackage{makecell,tikz}
\usepackage[switch,columnwise]{lineno}

\tikzset{near start abs/.style={xshift=1cm}}
\usetikzlibrary{shapes.geometric}
\usetikzlibrary[arrows,arrows.meta,decorations.pathmorphing,backgrounds,positioning,fit,petri]

\renewcommand{\raggedright}{\leftskip=0pt \rightskip=0pt plus 0cm}
\newcommand{\tabincell}[2]{\begin{tabular}{@{}#1@{}}#2\end{tabular}}
\usepackage{multirow}
\newtheorem{lemma}{Lemma}
\newtheorem{remark}{Remark}
\newtheorem{corollary}{Corollary}
\newtheorem{definition}{Definition}
\newtheorem{assumption}{Assumption}
\newtheorem{proposition}{Proposition}

\newcommand{\A}{\mathbf{A}}
\newcommand{\B}{\mathbf{B}}
\newcommand{\D}{\mathbf{D}}
\renewcommand{\d}{\mathbf{d}}
\newcommand{\e}{\mathbf{e}}

\newcommand{\I}{\mathbf{I}}
\newcommand{\M}{\mathbf{M}}
\newcommand{\T}{\mathcal{T}}
\newcommand{\U}{\mathbf{U}}

\newcommand{\Dr}{\operatorname{D}}
\renewcommand{\P}{\mathbf{P}}
\newcommand{\G}{\mathbf{G}}
\renewcommand{\H}{\mathbf{H}}
\newcommand{\tr}{\operatorname{tr}}
\newcommand{\grad}{\operatorname{grad}}
\newcommand{\X}{\mathbf{X}}
\newcommand{\x}{\mathbf{x}}
\newcommand{\y}{\mathbf{y}}
\newcommand{\W}{\mathbf{W}}
\newcommand{\Y}{\mathbf{Y}}
\newcommand{\Z}{\mathbf{Z}}
\newcommand{\epsb}{\boldsymbol{\epsilon}}
\newcommand{\phib}{\boldsymbol{\phi}}
\newcommand{\Phib}{\boldsymbol{\Phi}}
\newcommand{\Psib}{\boldsymbol{\Psi}}
\newcommand{\Pib}{\boldsymbol{\Pi}}
\newcommand{\xib}{\boldsymbol{\xi}}
\newcommand{\Xib}{\boldsymbol{\Xi}}
\newcommand{\Thetab}{\boldsymbol{\Theta}}

\renewcommand{\L}{\mathfrak{P}}
\newcommand{\J}{\mathfrak{Q}}
\newcommand{\Mf}{\mathfrak{M}}
\renewcommand{\S}{\mathfrak{S}}
\newcommand{\St}{\mathfrak{St}}
\newcommand{\Gr}{\mathfrak{Gr}}

\newtheorem{innercustomgeneric}{\customgenericname}
\providecommand{\customgenericname}{}
\newcommand{\newcustomtheorem}[2]{%
	\newenvironment{#1}[1]
	{%
		\renewcommand\customgenericname{#2}%
		\renewcommand\theinnercustomgeneric{##1}%
		\innercustomgeneric
	}
	{\endinnercustomgeneric}
}
\newcustomtheorem{customprop}{Proposition}

\begin{document}
\title{
	Trace Quotient with Sparsity Priors for Learning Low Dimensional Image Representations
}

\author{
Xian~Wei,~\IEEEmembership{Member,~IEEE,}
Hao~Shen,~\IEEEmembership{Member,~IEEE,}
and~Martin~Kleinsteuber
\IEEEcompsocitemizethanks{
\IEEEcompsocthanksitem X. Wei is with with Fujian Institute of Research 
on the Structure of Matter, Chinese Academy of Sciences (CAS), China.
 \protect
E-mail: xian.wei@tum.de
%
\IEEEcompsocthanksitem H. Shen is with the Technical University of Munich, Germany
and fortiss GmbH, Munich, Germany. \protect
E-mail: shen@fortiss.de
\IEEEcompsocthanksitem M. Kleinsteuber is with the Technical University of Munich, Germany
and Mercateo AG, Munich, Germany. \protect
E-mail: kleinsteuber@tum.de
\IEEEcompsocthanksitem This work has been supported 
by the German Research Foundation (DFG) under Grant 
No. KL 2189/9-1 and the CAS Pioneer Hundred Talents Program (Type C) under Grant No.2017-122.
%
}
}

\markboth{IEEE Transactions on Pattern Analysis and Machine 
Intelligence, Submitted Version}%
{Shell \MakeLowercase{\textit{et al.}}: Bare Demo of IEEEtran.cls for Computer Society Journals}

\IEEEtitleabstractindextext{
\begin{abstract}
	This work studies the problem of learning appropriate low 
    dimensional image representations.
	We propose a generic algorithmic framework, which leverages two classic 
    representation learning paradigms, i.e., sparse representation and
    the trace quotient criterion, to disentangle underlying factors of variation
    in high dimensional images.
	Specifically, we aim to learn simple representations of low dimensional, discriminant
	factors by applying the trace quotient criterion 
	to well-engineered sparse representations. 
	%
	We construct a unified cost function, coined as the SPARse 
    LOW dimensional representation (\emph{SparLow}) function, 
    for jointly learning both a sparsifying dictionary and 
    a dimensionality reduction transformation.
	The \emph{SparLow} function is widely applicable for developing 
	various algorithms in three classic machine learning scenarios, namely, 
	unsupervised, supervised, and semi-supervised learning.
	In order to develop efficient joint learning algorithms for maximizing the 
	\emph{SparLow} function, we deploy a framework of sparse coding with appropriate 
	convex priors to ensure the sparse representations to be locally differentiable.
	Moreover, 
	we develop an efficient geometric conjugate 
	gradient algorithm to maximize the \emph{SparLow} function on its underlying 
	Riemannian manifold.
	Performance of the proposed \emph{SparLow} algorithmic framework is investigated 
	on several image processing tasks, such as 3D data visualization, 
	face/digit recognition, and object/scene categorization.
\end{abstract}

}

\maketitle

\IEEEdisplaynontitleabstractindextext
\IEEEpeerreviewmaketitle


\IEEEraisesectionheading{\section{Introduction}\label{sec:introduction}}
\IEEEPARstart{F}{inding} appropriate low dimensional representations of data
is a long-standing challenging problem in data processing and machine learning.
Particularly, suitable low dimensional image representations have demonstrated 
prominent capabilities and conveniences in various image processing applications, 
such as image visualization \cite{hint:science06,rowe:science00}, 
segmentation \cite{elha:pami13}, clustering \cite{elha:pami13,hexf:pami05}, 
and classification \cite{belh:pami97,yans:pami07}.
Recent development in representation learning confirms that proper 
data representations are the key to success of modern machine 
learning algorithms.
%
%
One of its major challenges is how to automatically extract suitable representations 
of data by employing certain general-purpose learning mechanisms to promote solutions to 
machine learning problems \cite{beng:pami13,lecu:nature15}.
In this work, we aim to develop an effective two-layer representation
 learning paradigm for constructing low 
dimensional image representations, which is capable of revealing task-specific 
information in image processing.


\subsection{Related Work}
\label{sec:11}
\emph{Sparse representation} is a well-known powerful tool to 
explore structure of images for specific 
learning tasks \cite{elad:book10}.
Images of interest are assumed to admit sparse 
representations with respect to a collection of atoms, known
as a \emph{dictionary}.
Namely, each image can be constructed as a linear combination 
of only a few atoms.
With such a model, atoms are explanatory factors that 
are capable of describing intrinsic structures of the images \cite{beng:pami13}.
Prominent applications in image processing include image 
reconstruction, super-resolution, denoising, and inpainting, 
e.g., \cite{ahar:tsp06,yang:tip10,hawe:tip13}.
The corresponding methods are often referred to as 
\emph{data-driven sparse representation}.
Moreover, for a given dictionary, sparse 
representations of images can be interpreted as extracted 
features of the images, which capture sufficient information to 
reconstruct the original images.
By feeding the sparse representations directly to classifiers or 
other state of the art methods, improving performance has been observed 
in various image processing applications, such as 
face recognition \cite{wrig:pami09}, motion segmentation \cite{elha:pami13}
and object categorization \cite{gaos:pami13, srin:tip15, yang:cvpr09}.
These observations suggest that appropriate sparse representations of images are capable 
of facilitating specific learning tasks.
We refer to these approaches as \emph{task-driven sparse representation}.

Performance of task-driven sparse representation methods is known 
to depend significantly on the construction of dictionaries.
For example, for solving the problem of image classification,
class-specific dictionaries can be constructed by 
directly selecting images from each class, either randomly 
\cite{wrig:pami09} or according to certain structured priors 
\cite{yang:cvpr09,srin:tip15}, to achieve better results than state of the arts.
Dictionaries can also be learned with respect to specific criteria, such as locality 
of images \cite{wang:cvpr10, gaos:pami13}, class-specific discrimination 
\cite{perr:pami08}, optimal Fisher discrimination criterion \cite{yang:ijcv14}, 
and maximal mutual information \cite{qiuq:pami14}, to further improve
performance in image classification.
Alternatively, sparse representations can also be combined with the 
classical expected risk minimization formulation, such as 
\emph{least squares loss} 
\cite{zhan:cvpr10, jian:pami13, Zhou:tip17}, \emph{logistic loss} \cite{mair:pami12},
and \emph{square hinge loss} \cite{yang:cvpr10}.
%
%
In the literature, these methods are referred to as the \emph{task-driven 
dictionary learning (\emph{TDDL})} \cite{mair:pami12}. 
%
%
Note, that these methods are mainly devoted to supervised learning.

All aforementioned approaches often solve a problem of learning both 
a dictionary and a task-specific parameter, either sequentially or simultaneously.
Sparse representations of images are treated as inputs to task-specific learning 
algorithms, such as classifiers or predictors.
Early work in \cite{gkio:nips11} applies a linear projection to 
extract low dimensional features that preserve inner products of 
pair-wise sparse coefficients for facilitating image classification.
A similar approach in \cite{gaoj:prl12} shows that low dimensional features of sparse
representations of images, obtained by applying Principal Component Analysis (PCA) 
to sparse representations, are capable of enhancing performance in image visualization and clustering.
More recently, applying the spectral clustering framework to sparse representations 
of images yields the so-called \emph{Sparse Subspace Clustering} method 
\cite{elha:pami13}, leading to promising results
in motion segmentation and face clustering.
All these observations indicate that low dimensional features of sparse 
representations can be more task-specific in achieving higher performance.
Such a phenomenon can be studied in a more general framework of 
\emph{representation learning}, which aims to disentangle underlying 
factors of variation in data \cite{beng:pami13,lecu:nature15,beng:ftml09}.
%
%
%
Specifically, sparse representations of images can be 
regarded as a tool for modelling such factors in a first instance. 
In a second instance, a task-dependent combination of the sparse coefficients 
then has the potential to further improve the images' representations. 
%
So far, such representation learning paradigms are often constructed 
as a separated two-stage unsupervised encoding scheme,
i.e., the further disentangling instrument is 
independent from the layer of sparse representation. 
Although sparse coding can also been utilized to construct a deep learning 
architecture for extracting more abstract representations \cite{beng:pami13},
such an approach often requires special computing hardware and a very 
large amount of training images, 
resulting in prohibitive training efforts compared to two-layer learning approaches.

\subsection{Motivation and Main Contributions}
%
%
The motivation of this work is to develop a generic two-layer 
representation joint learning framework that 
allows to extract low dimensional, discriminant image representations 
by linearly projecting high-dimensional sparse factors of variation 
onto a low dimensional subspace. 
Among various low dimensional learning instruments for 
discriminant factors of variation, the Trace 
Quotient (TQ) criterion is a simple but powerful, linear framework. 
For a two-class problem, perfect separation can be achieved via a
TQ minimization, when the classes are linearly separable \cite{koki:nlaa10}.
This generic criterion is shared by various classic dimensionality reduction (DR) 
methods, including
PCA, Linear Discriminant Analysis (LDA) \cite{belh:pami97},
Linear Local Embedding (LLE) \cite{rowe:science00}, Marginal Fisher Analysis (MFA) \cite{yans:pami07},
Orthogonal Neighborhood Preserving Projection (ONPP) \cite{koki:nlaa10}, 
Locality Preserving Projections (LPP) \cite{hexf:pami05}, Orthogonal LPP (OLPP) \cite{caid:tip06}, 
Spectral Clustering (SC) \cite{elha:pami13}, semi-supervised LDA (SDA) 
\cite{caid:iccv07}, etc. 

Recently, the authors of this paper proposed to employ the TQ criterion to further 
disentangle sparse representations of image data to facilitate unsupervised learning 
tasks \cite{weix:cvpr16}, and developed a joint disentangling framework, coined as 
SPARse LOW dimensional representation learning (\emph{SparLow}).
It has been further tested on supervised learning tasks \cite{weix:icpr16}.
Although numerical evidences of the proposed
framework have been provided in these two works, a complete investigation
of the \emph{SparLow} learning framework has not been systematically conducted.
Hence, the goal of this work is to fully investigate the potential of the \emph{SparLow} 
framework as a fundamental representation learning instrument in a broad 
spectrum of machine learning, i.e., \emph{supervised}, \emph{unsupervised}, and 
\emph{semi-supervised learning}, and to further explore its capacity in image processing applications.
The main contributions are as follows:
\begin{enumerate}
	\item The difficulty of constructing an efficient algorithm to optimize the
		\emph{SparLow} cost function lies in the differentiability of sparse 
		representations with respect to a given dictionary. 
		To address this issue, we consider the sparse coding problem by 
		minimizing a quadratic reconstruction error with appropriate convex sparsity 
		priors. 
		We show that when sparse representations of all data samples are unique, then
		the sparse representations can be interpreted as a locally differentiable
		function with respect to the dictionary, and the first derivative of such sparse
		representation has a closed-form expression.
		 
	\item The joint disentanglement problem results in solving an optimization 
		problem that is defined on an underlying Riemannian manifold.
		Although optimization on Riemannian manifolds is nowadays well established, 
		derivation of a Riemannian Conjugate Gradient (CG) \emph{SparLow} algorithm remains a sophisticated task.
		In this work, we present a generic framework of Riemannian CG 
		\emph{SparLow} algorithms with necessary technical details, so that both 
		theorists and practitioners can benefit from it.

	\item Finally, sensitivity of the \emph{SparLow} framework with respect to 
		its parameters is investigated by numerical experiments.
\end{enumerate}

Compared to state of the art  \emph{TDDL}  methods, the proposed \emph{SparLow} framework
shares the following three merits:
i) It introduces a generic formulation for learning both a dictionary
and an orthogonal DR transformation in unsupervised, supervised
and semi-supervised learning settings, in contrast to the existing
supervised \emph{TDDL} approaches \cite{yang:cvpr10, zhan:cvpr10, mair:pami12, jian:pami13} and 
two-layer unsupervised learning framework \cite{gaoj:prl12, elha:pami13};
ii) Compared to the popular class-wise sparse coding approaches, which often learn
one dictionary for each class \cite{perr:pami08}, and employ a set
of binary classifiers in either ``one-versus-all" or ``one-versus-one"
scheme for multiclass classification \cite{mair:pami12, yang:cvpr10}, \emph{SparLow}
only learns a compact dictionary and an orthogonal projection
for all classes. It significantly reduces the computational complexity 
and memory burden for classifications problems with many classes;
iii) Different to the \emph{TDDL} methods 
that compute the sparse representation
by using the sparsity prior of Lasso or Elastic Net \cite{zouh:rssb05},  \emph{SparLow}
allows for more general convex sparsity priors. 

The paper is organized as follows. 
Section~\ref{sec:02} provides a brief review on both sparse representations and 
the trace quotient optimization.
In Section~\ref{sec:03}, we construct a generic cost function for learning both a
sparsifying dictionary and an orthogonal DR transformation,
and then discuss its several exemplifications in Section~\ref{sec:04}.
%
A geometric CG algorithm is developed in Section~\ref{sec:05},
together with their experimental evaluations presented in 
Section~\ref{sec:06}.
Finally, conclusions and outlook are given in Section~\ref{sec:07}.

\section{Sparse Coding and Trace Quotient}
\label{sec:02}
In this section, we first briefly review some state of the art results of sparse
representations with convex priors, which formulate the first layer of the \emph{SparLow}.
Then, the TQ optimization based dimensionality reduction is presented,
which is used to construct the second layer of the \emph{SparLow}.
The presented hypothesis, restrictions and reformulations of both sparse
representations and dimensionality reduction 
 enable us to develop the joint learning paradigm of the \emph{SparLow}, which is fully
 investigated in Sections \ref{sec:03}, \ref{sec:04} and \ref{sec:05}.

%

We start with an introduction to notations and definitions used in the paper. 
In this paper, we denote sets and manifolds with fraktur letters, such as $\L$, $\S$, and $\Gr$,
and $|\L|$ the cardinality of the set $\L$.
Matrices are written as boldface capital letters like $\X$, $\Phib$, 
column vectors are denoted by boldfaced small letters, e.g., $\x$, 
$\d$, whereas scalars are either capital or small letters, such as $n$ and $N$. 
%
%
%
%
We denote by 
$\I_{n}$ the $n \times n$-identity matrix, 
$(\cdot)^\top$ the matrix transpose, 
$\operatorname{tr}(\cdot)$ the trace of a square matrix.
%
Furthermore, $\| \cdot \|_{1}$ and $\| \cdot \|_{2}$ denote
the $\ell_{1}$-, $\ell_{2}$-norm of a vector, and  $\| \cdot \|_{F}$ the Frobenius norm of a matrix.
%
%
\subsection{Sparse Coding with Convex Priors}
\label{sec:21}
Let $\X := [\x_{1}, \ldots, \x_{n}] \in \mathbb{R}^{m \times n}$
be a collection of $n$ data points in $\mathbb{R}^{m}$.
Sparse coding aims to find a collection of \emph{atoms} $\d_{i} \in
\mathbb{R}^{m}$ for $i = 1, \ldots, r$, such that each data point can be
approximated by a linear combination of a small subset of atoms. 
Specifically, by denoting $\D := [\d_{1}, \ldots, \d_{r}] 
\in \mathbb{R}^{m \times r}$, referred to as a \emph{dictionary}, all data
samples $\x_{i}$ for all $i=1, \ldots, n$ are assumed to be modeled as
\begin{equation}
\label{eq_DL_framework}
    \x_{i} = \D {\phib}_{i} + \epsb_{i},
\end{equation}
where ${\phib}_{i} \in \mathbb{R}^{r}$ is the corresponding sparse 
representation of $\x_{i}$, and $\epsb_{i} \in \mathbb{R}^m$ is a small additive 
residual, such as noise.
%

One popular solution of the sparse coding problem above 
is given by solving the following minimization problem
\begin{equation}
\label{eq:sparse_coding}
	\operatorname*{min}_{\D \in \mathbb{R}^{m \times r}, \Phib \in \mathbb{R}^{r \times n}}
	\sum_{i = 1}^{n} \tfrac{1}{2} \| \x_{i} - \D \phib_{i} \|_{2}^{2} + g(\phib_{i}),
\end{equation}
with $\Phib := [\phib_1, \ldots, \phib_n] \in \mathbb{R}^{r \times n}$.
Here, the first term penalizes the reconstruction error of sparse representations, and
the second term is a sparsity promoting regularizer.
Often, the function $g$ is chosen to be separable, i.e., its evaluation is computed 
as the sum of functions of the individual components of its argument.
\begin{definition}[Separable sparsity regularizer]
    Let $\phib := [\varphi_{1}, \ldots, \varphi_{r}]^{\top} \in \mathbb{R}^{r}$. 
    A funciton $g \colon \mathbb{R}^{r} \to [0, +\infty)$ is
    a separable sparsity regularizer if 
    \begin{equation}
    	g(\phib) := \sum\limits_{i=1}^{r} g_{i}(\varphi_{i}),
    \end{equation}
    with $g_{i}(\varphi_{i}) \ge 0$ and $g_{i}(0) = 0$.
\end{definition}
\noindent There are many choices for $g$ in the literature, such as the $\ell_{0}$-(quasi-)norm 
and its variations. 
It is important to notice that an optimization procedure to solve the 
problem as in Eq.~\eqref{eq:sparse_coding} will force the norm of columns of 
$\D$ to infinity, and consequently drive the value $g(\phib_{i})$ to zero.
To avoid such trivial solutions, it is common to restrict all atoms 
$\d_{i} \in \mathbb{R}^{m}$ to have unit norm, i.e., the set of dictionaries is
a product manifold of $r$ times the $(m-1)$-dimensional unit sphere, i.e.,
\begin{equation}
	\mathfrak{S}(m,r) := \left\{ \D \in \mathbb{R}^{m \times r} 
	\big| \| \d_{i}\|_{2} = 1 \right\}.
\end{equation}

If the dictionary $\D$ is fixed, Problem~\eqref{eq:sparse_coding} degenerates into a collection 
of decoupled sample-wise \emph{sparse regression} problems.
Specifically, for each sample $\x$, we have
\begin{equation}
\label{eq:sparse_regression_1}
	\min_{\phib \in \mathbb{R}^{r}} 
    f_{\x}(\phib) := \tfrac{1}{2} \| \x - \D \phib \|_{2}^{2} + g(\phib).
\end{equation}
If the sparsity regularizer $g$ is \emph{strictly convex},
then the \emph{sparse regression} problem~\eqref{eq:sparse_regression_1} has a unique solution.
Thus, the solution of the sample-wise sparse regression problem 
can be treated as a function in $\x$, i.e., 
\begin{equation}
\label{eq:sparse_regression}
	\phib_{\D}(\x) := \operatorname*{argmin}_{\phib \in \mathbb{R}^{r}}
	f_{\x}(\phib).
\end{equation}
With sparse representations of all samples being calculated, specific learning 
algorithms can be directly applied to these sparse coefficients to extract further
representations.

By choosing the function $g$ to be the \emph{elastic net 
regularizer} \cite{zouh:rssb05}, i.e.,
\begin{equation}
\label{eq:elastic}
	g_{e}(\phib) := \lambda_{1} \| \phib \|_{1} + \lambda_{2} \| \phib \|_{2}^{2},
\end{equation}
where parameters $\lambda_{1},\lambda_{2} > 0$ are chosen to ensure stability and 
uniqueness of the sparse solution,
for a given dictionary $\D$, the solution $\phib_{\D}(\x_{i})$ as in 
Eq.~\eqref{eq:sparse_regression} can be considered as a 
function in $\D$, i.e., $\phib_{\D} \colon \mathbb{R}^{m} \to \mathbb{R}^{r}$, with
a closed-form expression \cite{mair:pami12}.
Moreover, the sparse representation $\phib_{\D}$ is locally differentiable with 
respect to the dictionary $\D$.
Such a convenient result has lead to a joint learning approach to optimize the cost function 
associated with the \emph{TDDL} methods, e.g.,~\cite{yang:cvpr10, mair:pami12}.

Unfortunately, for a general choice of $g$, there is no guarantee to have a closed-form 
expression of the sparse representation. 
Recent work in \cite{bagn:nips09} proposes the choice of $g$ to be a full unnormalized 
Kullback-Leibler (KL) divergence.
Although the corresponding sparse representation has no closed-form expression, 
its first derivative does have an explicit formula, which enables the development of 
gradient-based optimization algorithms.
It is worth noticing that both the elastic net regularizer and the full unnormalized 
KL divergence regularizer belong to the category of convex sparsity regularizer.
Hence, we hypothesize that the choice of 
an appropriate convex sparsity prior can 
facilitate the development of efficient joint learning algorithms for discovering 
abstract representations of sparse coefficients of images.
Such a hypothesis is proven to be true in Section~\ref{sec:51}.

\subsection{Optimization of the Trace Quotient Criterion}
Classic DR methods aim to find a lower-dimensional
representation $\y_{i} \in \mathbb{R}^{l}$ of given
data samples $\x_{i} \in \mathbb{R}^{m}$ with $l<m$, via a 
mapping $\mu \colon \mathbb{R}^{m} \to \mathbb{R}^{l}$, which captures certain application dependent 
properties of the data.
Many classic DR methods restrict the mapping $\mu$ to be an orthogonal
projection.
Let us denote the set of $m \times l$ orthonormal matrices by
\begin{equation}
	\mathfrak{St}(l,m) := \big\{ \U \in \mathbb{R}^{m\times l}| \U^{\top} \U = \I_{l} \big\}.
\end{equation}
Specifically, in this work, we confine ourselves to the form of orthogonal projections as
$\mu(\x) := \U^{\top} \x$.
%
%
This model covers a wide range of classic supervised and unsupervised learning
methods, such as LDA, MFA, PCA, OLPP, and ONPP.
Further details are given in Section~\ref{sec:04}.

One generic algorithmic framework to find optimal $\U \in \St(l,m)$
is formulated as a maximization problem of the so-called \emph{trace quotient} 
or \emph{trace ratio}, i.e.,
\begin{equation}
\label{eq:trace_quotient_st}
	\operatorname*{argmax}_{\U \in \St(l,m)} \;
	\frac{\operatorname{tr}(\U^{\top} \A \U)}{\operatorname{tr}(\U^{\top} \B \U) + \sigma},
\end{equation}
where matrices $\A, \B\in \mathbb{R}^{m \times m}$ are often symmetric positive semidefinite, 
and constant $\sigma > 0$ is chosen to prevent the denominator from being zero.
%
%
Both matrices $\A$ and $\B$ are constructed to measure the ``similarity'' between data points
	according to the specific problems,  e.g.,~\cite{cunn:jmlr15, koki:nlaa10}.
Specifically, in following sections, they are represented as smooth functions to measure 
the discrepancy between sparse coefficient pairs $(\phib_{\D}(\x_{i}), \phib_{\D}(\x_{j})), \forall i, j$.
The methodological details and examples will be given and discussed in Sections~\ref{sec:03} and \ref{sec:04}.
%
%

It is obvious that solutions of the problem in Eq.~\eqref{eq:trace_quotient_st} are 
rotation invariant, i.e., let $\U^{*} \in \St(l,m)$ be a solution of the problem, then 
so is $\U^{*}\Thetab$ for any $\Thetab \in \mathbb{R}^{l \times l}$ being orthogonal.
In other words, the solution set of the problem in Eq.~\eqref{eq:trace_quotient_st} is the 
set of all $l$-dimensional linear subspaces in $\mathbb{R}^{m}$.
In order to cope with this structure, we employ the Gra{\ss}mann manifold, which can be
alternatively identified as the set of all $m$-dimensional rank-$l$ orthogonal 
projectors, i.e.,
\begin{equation}
\label{eq:grassmann}
	\mathfrak{Gr}(l,m) := \left\{ \U \U^{\top} | \U \in \St(l,m) \right\}.
\end{equation}
Thus, the trace quotient problem can be formulated as
\begin{equation}
\label{eq:trace_quotient_gr}
	\operatorname*{argmax}_{\P \in \Gr(l,m)} \;
  	\frac{\operatorname{tr}(\P \A)}{\operatorname{tr}(\P \B) + \sigma}.
\end{equation}
Although various efficient optimization algorithms have been developed to solve the 
trace quotient problem, see~\cite{ngot:simax10,koki:nlaa10,cunn:jmlr15}, 
the construction described in the next section requires further nontrivial, 
constructive development.

\section{A Joint Disentangling Framework}
\label{sec:03}
In this section, we construct the \emph{SparLow} cost function, 
	which adopts the construction of sparse coding with convex 
	priors in the framework of TQ maximization to extract low dimensional 
	representations of sparse codings of images.
	
%
The \emph{SparLow} function allows to jointly learn 
both a sparsifying dictionary and an orthogonal projection in the framework of 
TQ maximization.
Let us denote by $\Phib(\D,\X) := [\phib_{\D}(\x_{1}),\ldots,\phib_{\D}(\x_{n})] \in 
\mathbb{R}^{r \times n}$ the sparse representation of the data 
$\X = [\x_{1},\ldots,\x_{n}]$ for a given dictionary $\D$ computed by solving the
sparse regression problems as in Eq.~\eqref{eq:sparse_regression}.
Let $\mathcal{A} \colon \mathbb{R}^{r \times n} \to \mathbb{R}^{r \times r}$ and 
$\mathcal{B} \colon \mathbb{R}^{r \times n} \to \mathbb{R}^{r \times r}$ be two smooth 
functions that serve as generating functions for the matrices $\A$ and $\B$ in 
the trace quotient in Eq.~\eqref{eq:trace_quotient_gr}.
Constructions of the two structure matrix-valued function $\mathcal{A}$ and $\mathcal{B}$ 
are according to the specific learning tasks and exemplified in Section~\ref{sec:04}.
We can define a generic trace quotient function on sparse representations as
%
\begin{equation}
\label{eq_main_cost_no_regular}
\begin{split}
	f \colon \mathfrak{S}(m,r) \times \mathfrak{Gr}(l,r) & \to \mathbb{R} \\
	f(\D,\P) & := \frac{ \operatorname{tr} \left(\P \mathcal{A}( \Phib(\D,\X)) \right)  }
	{  \operatorname{tr} \left( \P\, \mathcal{B}(\Phib(\D,\X)) \right) + \sigma  }.
\end{split}
\end{equation}
From the perspective of learning representations, the projection $\P$ aims to capture low-dimensional 
discriminant features in sparse representations of images.
It is important to notice that the function $f$ is not necessarily differentiable,
unless the structure functions $\mathcal{A}$ and $\mathcal{B}$ are differentiable 
in the dictionary $\D$, i.e., the sparse representations $\Phib(\D,\X)$ are differentiable in 
$\D$.
This issue is further discussed in Section~\ref{sec:51}.

In order to prevent solution dictionaries from being highly coherent, which is necessary 
for guaranteeing the local smoothness of sparse solutions \cite{hawe:tip13}, we employ a 
log-barrier function on the scalar product of all dictionary columns to 
control the mutual coherence of the learned dictionary $\D$, 
i.e., for dictionary $\D = [\d_{1}, \ldots, \d_{r}] \in \mathbb{R}^{m \times r}$,
we define
\begin{equation}
\label{eq_mutual_coherence0}
	g_{c}(\D) := - \operatorname*{\sum}_{1 \leq i<j\leq r} 
	\tfrac{1}{2} \operatorname*{log}\big(1-( \d_{i}^{\top} \d_{j})^2 \big).
\end{equation}
It is worth noticing that the interplay between the exactness of sparse 
representations of images $\Phib(\D,\X)$ and the measure of discrimination 
of TQ is indirect. 
As observed in preliminary experiments of this work, it is very difficult 
to ensure maximization of the function $f$ regularized by $g_{c}$ to extract
good low-dimensional representations of images for the learning tasks.
Namely, sparse representations $\Phib(\D,\X)$ are data-driven 
disentangling factors of variation in the first layer, which carry data 
information independently from the (potentially) task-specific second layer.
Such an observation might also be interpreted as overfitting of sparse representations 
to the second layer's target.
In order to deal with this problem, we propose to adopt the \emph{warm start}
strategy from an optimal data-driven dictionary, and restrict the new dictionary to lie
in a neighborhood of the warm start to explicitly balance both 
data-driven and task-driven information.
Specifically, we propose the following regularizer on the dictionary as
\begin{equation}
\label{eq:reg_h}
	g_{d}(\D) := \tfrac{1}{2} \| \D - \D^{*}\|_{F}^2,
\end{equation}
where $\D^{*}$ is the optimal data-driven dictionary learned from the data $\X$.
In the rest of the paper, we refer to it as the \emph{data regularizer}.
It measures the distance between an estimated dictionary $\D$ and the dictionary $\D^{*}$ 
in terms of the Frobenius norm.
Practically, we set $\D^{*}$ to be a dictionary produced by state of the art 
methods, such as K-SVD \cite{ahar:tsp06}.
Our experiments have verified that $g_{d}$ guarantees stable performance of 
\emph{SparLow} algorithms, see Section~\ref{sec:521}.

To summarize, we construct the following cost function to jointly learn both a 
sparsifying dictionary and an orthogonal transformation, i.e.,
\begin{equation}
\label{eq_main_function}
\begin{split}
	J \colon \S(m,r) \times & \Gr(l,r) \to \mathbb{R} \\[0.5mm]
	J(\D, \P) := & f(\D, \P) - \mu_{1} g_{c}(\D) - \mu_{2}g_{d}(\D),
\end{split}
\end{equation}
where the two weighting factors $\mu_{1} > 0$ and $\mu_{2} > 0$ 
control the influence of the two regularizers on the final solution.
In this work, we refer to it as the \emph{SparLow} function.

\section{Exemplifications of the \emph{SparLow} Model}
\label{sec:04}
In the previous section, we construct the \emph{SparLow} function 
for extracting low dimensional representations of sparse codings of images.
In what follows, we exemplify counterparts of several classic unsupervised, 
supervised and semi-supervised learning methods 
by constructing various structure smooth functions 
$\mathcal{A}$ and $\mathcal{B}$ in Eq.~\eqref{eq_main_cost_no_regular}. 

\subsection{Unsupervised \emph{SparLow}}
\label{sec:341}
We firstly introduce three unsupervised learning methods within the \emph{SparLow} framework.
\subsubsection{PCA-like \emph{SparLow}}
\label{sec:3411}
The standard PCA method computes an orthogonal transformation $\U \in \St(l,m)$, so that 
the variance of the low dimensional representations of the data $\X$ is maximized, 
i.e., ${\U}$ is the solution of the following maximization problem
\begin{equation}
\label{eq_pca}
\operatorname*{\max}_{\U \in \St(l,m)} \operatorname{tr}
\left( \U^\top \X \Pib_{n} \X^{\top} \U \right),
\end{equation}
where ${\Pi}_{r} := \I_{r} - \tfrac{1}{r}{\mathbf{1}}_{r} {\mathbf{1}}^{\top}_r$ 
is the centering matrix in $\mathbb{R}^{r}$ with $\mathbf{1}_{r} = 
[1,\cdots,1]^{\top} \in \mathbb{R}^{r}$.
In the framework of trace quotient, the denominator is trivially a constant, i.e.,
\begin{equation}
\operatorname{tr}\big( \U^{\top} \B_{\text{pca}} \U \big) = 
\operatorname{tr}\big( \X \Pib_{n} \X^{\top} \big),
\end{equation}	
with $\B_{\text{pca}} = \operatorname{tr}\big( \X \Pib_{n} \X^{\top} \big) \I_{n}$.
By adopting the sparse representations $\Phib({\D}, \X)$, by $\Phib$ for short, 
the two structure functions are defined as
\begin{equation}
\left\{ \!\!\! \begin{array}{rcl}
\mathcal{A}_{\text{pca}}({\Phib}) &\!\!\!\!\! :=\!\!\!\!\! & 
\Phib \Pib_{n} \Phib^{\top} \\[1mm]
\mathcal{B}_{\text{pca}}(\Phib) &\!\!\!\!\! :=\!\!\!\!\! & 
\operatorname{tr}(\Phib \Pib_{n} \Phib^{\top}) \I_{r}.
\end{array} \right.
\end{equation}
We refer to the corresponding \emph{SparLow} algorithm as the \emph{PCA-SparLow} algorithm.

\subsubsection{LLE-like \emph{SparLow}}
\label{sec:3412}
The original LLE algorithm aims to find low dimensional representations of the data 
via fitting directly the barycentric coordinates of a point based on its 
neighbors constructed in the original data space \cite{rowe:science00}.
It is well known that the low dimensional representations in the LLE method can only 
be computed implicitly. 
Therefore, the so-called ONPP method introduces an explicit orthogonal 
transformation between the original data and its low dimensional representation \cite{koki:nlaa10}.
Specifically, the ONPP method solves the following problem
\begin{equation}
\label{eq_lle}
\operatorname*{\min}_{\U \in \St(l,m)} \operatorname{tr}\left( \U^\top \X 
\Z_{\text{lle}} \X^\top \U \right),
\end{equation}
where $\Z_{\text{lle}} = (\I_{n} - \W)^\top (\I_{n}-\W)$ with $\W \in \mathbb{R}^{n \times n}$ 
being the matrix of barycentric coordinates of the data.
Similar to the construction for PCA-\emph{SparLow}, we construct the following 
functions for an LLE-like \emph{SparLow} approach
\begin{equation}
\left\{ \!\!\! \begin{array}{rcl}
\mathcal{A}_{\text{lle}}(\Phib) &\!\!\!\!\! :=\!\!\!\!\! & 
\Phib \Z_{\text{lle}} \Phib^{\top} \\[1mm]
\mathcal{B}_{\text{lle}}(\Phib) &\!\!\!\!\! :=\!\!\!\!\! & 
\operatorname{tr}( \Phib \Z_{\text{lle}} \Phib^{\top}) \I_{r}.
\end{array} \right.
\end{equation}

\subsubsection{Laplacian \emph{SparLow}}
\label{sec:3413}
Another popular category of unsupervised learning methods are the ones involving a 
Laplacian matrix of data, e.g., Locality Preserving Projection (LPP) 
\cite{hexf:pami05}, Orthogonal LPP (OLPP) \cite{caid:tip06}, Linear Graph Embedding 
(LGE) \cite{yans:pami07}, and Spectral Clustering \cite{elha:pami13}.
Let us denote by 
${d}_{ij} := \exp (-\|{\x}_i - {\x}_j\|_2^2/t)$
the Laplacian similarity
between two data points ${\x}_i$ and ${\x}_j$ with constant $t>0$.
Similar to the approaches applied above, we adopt a simple formulation by setting
\begin{equation}
\label{eq_laplace_A}
\left\{ \!\!\! \begin{array}{rcl}
\mathcal{A}_{\text{lap}}(\Phib)&\!\!\!\!\! :=\!\!\!\!\! & 
\Phib \Z_{\text{lap}} \Phib^{\top} \\[1mm]
\mathcal{B}_{\text{lap}}( \Phib ) &\!\!\!\!\! :=\!\!\!\!\! & 
\Phib \Y_{\text{lap}} \Phib^{\top},
\end{array} \right.
\end{equation}
with $\Z_{\text{lap}} := \{z_{ij}\} \in \mathbb{R}^{n\times n}$ being a real 
symmetric matrix measuring the similarity between data pairs $({\x}_i, {\x}_j)$, and
$\Y := \{y_{ij}\} \in \mathbb{R}^{n \times n}$ being a diagonal matrix having
$y_{ii} := \sum_{j \neq i} z_{ij}$ for all $i,j$.
%
Specifically, the similarity matrix $\Z_{\text{lap}}$ can be computed by applying a
Gaussian kernel function on the distance between two data samples, i.e.,
$z_{ij} = {d}_{ij}$
if $\phib_{i}$ and $\phib_{j}$ are adjacent, $z_{ij} = 0$ otherwise.
%

\subsection{Supervised \emph{SparLow}}
\label{sec:342}
In this subsection, we focus on the supervised \emph{SparLow} learning for 
solving classification problems.
Assume that there are $c$ classes of images. 
Let $\X_{i} = [{\x}_{i1}, \ldots, {\x}_{in_{i}}] \in 
\mathbb{R}^{m \times n_{i}}$ for $i = 1, \ldots, c$ with $n_{i}$ being 
the number of samples in the $i^{\mathrm{th}}$ class. 
The corresponding sparse coefficients are denoted by 
${{\Phib}}_{i} := [{\phib}_{i1}, \ldots, {\phib}_{in_{i}}] \in \mathbb{R}^{r \times n_{i}}$, 
and ${{\Phib}} := [{{\Phib}}_{1}, \ldots, {{\Phib}}_{c}] \in \mathbb{R}^{r \times n}$
with $n = \sum\limits_{i=1}^{c} n_{i}$.  

\subsubsection{LDA \emph{SparLow}}
\label{sec:3421}
The classic LDA algorithm \cite{belh:pami97} aims to find low-dimensional representations of 
the high dimensional data, so that the between-class scatter is maximized, while the within-class scatter is minimized.          
Let us define by $\overline{\phib}_{i} \in \mathbb{R}^{r}$ the center of the $i$-th class.
The within-class scatter matrix is computed as
\begin{equation}
\begin{split}
\mathcal{B}_{\text{lda}}({{\Phib}}(\D,\X)) = & \sum\limits_{i=1}^{c} \sum\limits_{j=1}^{n_{i}} 
({\phib}_{ij} - \overline{{\phib}}_{i})({\phib}_{ij} - \overline{{\phib}}_{i})^{\top} \\
= & \sum\limits_{i=1}^{c} {{\Phib}}_{i} {\Pi}_{n_{i}} {{\Phib}}_{i}^{\top} \; 
=  {{\Phib}} \mathbf{L}^{w} {{\Phib}}^{\top},
\end{split}
\end{equation}
with $\mathbf{L}^{w} := \operatorname{diag}({\Pib}_{n_{1}}, \cdots, {\Pib}_{n_{c}}) 
\in \mathbb{R}^{n\times n}$
being a block diago\-nal matrix, whose diagonal blocks are the centering matrices in 
$\mathbb{R}^{n_{i}}$, associated with the corresponding classes.
%
Let $\overline{{\phib}} \in \mathbb{R}^{r}$ be the centre of all classes.
Then, we can define the between-class scatter matrix as
\begin{equation}
\begin{split}
\mathcal{A}_{lda}& (\Phib) =  \sum\limits_{i=1}^{c} n_i
(\overline{{\phib}}_{i} - \overline{{\phib}})(\overline{{\phib}}_{i} - \overline{{\phib}})^{\top} \\
= & \left[ {{\phib}}_{1}\! \tfrac{\mathbf{1}_{n_{1}}}{\sqrt{n_{1}}}, \ldots, {{\phib}}_{c}\! 
\tfrac{\mathbf{1}_{n_{c}}}{\sqrt{n_{c}}} \right] 
{\Pib}_{c}\!
\left[ {{\phib}}_{1}\! \tfrac{\mathbf{1}_{n_{1}}}{\sqrt{n_{1}}}, \ldots, {{\phib}}_{c}\! 
\tfrac{\mathbf{1}_{n_{c}}}{\sqrt{n_{c}}} \right]^{\top} \\
= &\, {{\Phib}} {\mathbf{L}}^{b} {{\Phib}}^{\top},
\end{split}
\end{equation}
where $\mathbf{L}^{b} := \mathbf{C}^b{\Pib}_{c}(\mathbf{C}^b)^{\top}$ 
with $\mathbf{C}^b = \operatorname{Bdiag}\big(\frac{\mathbf{1}_{n_1}}{\sqrt{n_1}}, \cdots, \frac{\mathbf{1}_{n_c}}{\sqrt{n_c}}\big) \in \mathbb{R}^{n\times c}$.

\subsubsection{MFA \emph{SparLow}}
\label{sec:3422}
Marginal Fisher Analysis (MFA) \cite{yans:pami07}, also known as Linear 
Discriminant Embedding (LDE), is the supervised version of the Laplacian Eigenmaps \cite{koki:nlaa10}.
The main idea is to maintain the original neighbor relations of points from the 
same class while pushing apart the neighboring points of different classes. 

Let $\mathfrak{N}_{k_1}^{+}({\phib}_i)$ denote the set of $k_1$ nearest neighbors which share the same label with ${\phib}_i$, 
and $\mathfrak{N}_{k_2}^{-}({\phib}_i)$ denote the set of $k_2$ nearest neighbors among the data points whose labels 
are different to that of ${\phib}_i$.
We construct two matrices $\Z_{\text{mfa}}^{+} := \{z_{ij}^{+}\} \in \mathbb{R}^{n\times n}$
and $\Z_{\text{mfa}}^{-} := \{z_{ij}^{-}\} \in \mathbb{R}^{n\times n}$ with
\begin{equation}
z_{ij}^{+} = \left\{\!\!
\begin{array}{ll}
1, & {\phib}_j\in \mathfrak{N}_{k_1}^{+}({\phib}_i) ~\text{or}~ 
{\phib}_i\in \mathfrak{N}_{k_1}^{+}({\phib}_j), \\
0, & \text{otherwise},
\end{array}
\right.
\end{equation}
and
\begin{equation}
z_{ij}^{-} = \left\{\!\!
\begin{array}{ll}
1, & {\phib}_j\in \mathfrak{N}_{k_1}^{-}({\phib}_i) ~\text{or}~ 
{\phib}_i\in \mathfrak{N}_{k_1}^{-}({\phib}_j), \\
0, & \text{otherwise}.
\end{array}
\right.
\end{equation}
Then, the Laplacian matrices for characterizing the inter-class and  
intra-class locality are defined as
\begin{equation}
\mathbf{L}_{\text{mfa}}^{-} = \Y_{\text{mfa}}^{-} - \Z_{\text{mfa}}^{-}, \quad\text{and}
\quad \mathbf{L}_{\text{mfa}}^{+} = \Y_{\text{mfa}}^{+} - \Z_{\text{mfa}}^{+},
\end{equation}
%
where $\Y^{+}$ and $\Y^{-}$ are two diagonal matrices defined as
%
\begin{equation}
y^{+}_{ii} = \sum_{j \neq i} z_{ij}^{+}, \quad\text{and}\quad
y^{-}_{ii} = \sum_{j \neq i} z_{ij}^{-}.
\end{equation}
Then, we construct the following functions for MFA-like \emph{SparLow} approach, i.e.,
\begin{equation}
\left\{ \!\!\! \begin{array}{rcl}
\mathcal{A}_{\text{mfa}}(\Phib) &\!\!\!\!\! :=\!\!\!\!\! & 
\Phib \mathbf{L}_{\text{mfa}}^{-} \Phib^\top \\[1mm]
\mathcal{B}_{\text{mfa}}(\Phib) &\!\!\!\!\! :=\!\!\!\!\! & 
\Phib \mathbf{L}_{\text{mfa}}^{+} \Phib^\top.
\end{array} \right.
\end{equation}
%
\subsubsection{MVR \emph{SparLow}}
\label{sec:3423}
Many challenging problems, e.g., multi-label classification, can 
be modeled as multivariate ridge regression (MVR) by solving the 
following minimization problem
\begin{equation}
\label{eq_lossFunc_leastSueres}
\min_{\D,\U,\W} \|\mathbf{Z} - \W^\top \U^\top {{\Phi}}(\D,\X) \|_F^2  + \mu \|\W\|_F^2 ,
\end{equation}
where $\mathbf{Z}\in \mathbb{R}^{d\times n}$ is the target matrix, 
$\U\in \St(l,r)$, $\W\in \mathbb{R}^{l\times d}$ and $\mu \in \mathbb{R}^{+}$. 
By freezing both $\mathbf{Z}$ and $\U$, a solution to the problem as in 
Eq.~\eqref{eq_lossFunc_leastSueres} with respect to $\W$ has a closed form expression as
\begin{equation}
\label{eq_W}
\W = \left( \U^\top ( \Phib \Phib^\top  + 
\mu \I_{r} ) \U \right)^{-1} \U^\top \Phib {\Z}^\top.		
\end{equation}
Using this closed expression to substitute $\W$ in Eq.~\eqref{eq_lossFunc_leastSueres}, 
we can rewrite Eq.~\eqref{eq_lossFunc_leastSueres} in the form of the \emph{SparLow} as
\begin{equation}
\label{eq_main_cost_MVR}
\left\{ \!\!\! \begin{array}{rcl}
\mathcal{A}_{\text{mvr}}(\Phib) &\!\!\!\!\! :=\!\!\!\!\! & 
- \Phib \Z_{\text{mvr}}^{\top} \Z_{\text{mvr}} \Phib^{\top} \\[1mm]
\mathcal{B}_{\text{mvr}}(\Phib) &\!\!\!\!\! :=\!\!\!\!\! & 
\left( \Phib \big( \Phib^{\top} + \mu \I_{r} \big) \right).
\end{array} \right.
\end{equation}
Therein, $\Z_{\text{mvr}}$ can be constructed as the binary class labels of input signals, which is usually coded as $\mathbf{Z}_{\text{mvr}} \in \mathbb{R}^{c\times n}$ with
$\mathbf{z}_i = [z_{i1},\cdots, z_{ic}]^\top$, $z_{ij} = 1$ if $\mathbf{z}_i$ is in class $j$, $z_{ij} = 0$ otherwise. 
Alternatively, $\mathbf{Z}_{\text{mvr}}$ can also be a handcrafted 
indicator matrix according to labels, e.g., the ``discriminative'' sparse codes in \cite{jian:pami13}. 

\subsection{Semi-supervised \emph{SparLow}}
\label{sec:343}
In this subsection, we demonstrate that the \emph{SparLow} model is also well suited to exploit unlabeled data in a 
semi-supervised setting. 
%
%
Assume that there are $n_{l}$ labeled samples $\X_{l} \in \mathbb{R}^{m\times n_{l}}$ 
and $n_{u}$ unlabeled samples $\X_{u} \in \mathbb{R}^{m\times n_{u}}$, with 
$n = n_{l} + n_{u}$.
For a given dictionary $\D$, we denote by $\Phib_{l} \in \mathbb{R}^{r \times n_{l}}$ and 
$\Phib_{u} \in \mathbb{R}^{r \times n_{u}}$ the corresponding sparse coefficients.
The first assumption to support \emph{semi-supervised \emph{SparLow} model} is that the learned dictionary 
for specific class is also effective for learning good sparse features from unlabeled data \cite{mair:pami12}.
Secondly, we follow the way that learning semi-supervised DR settings associated with preserving the global data manifold structure,
namely, nearby points will have similar lower-dimensional representations \cite{caid:iccv07, yugx:pr12} 
or labels \cite{zhou:nips04, belk:jmlr06}.

Given the whole dataset $\X$, let us define the graph Laplacian matrix 
$\mathbf{L} = \Y - \Z \in \mathbb{R}^{n\times n}$,  
where $\Z := \{z_{ij}\}$ with $z_{ij}$ weighting the edge between adjacency 
data pairs $({\phib}_i, {\phib}_j)$,
$z_{ij} = 0$ otherwise.
$\Y$ is diagonal with $y_{ii} := \sum_{j \neq i} z_{ij}$ for all $i,j$.

%
\subsubsection{Semi-supervised LDA \emph{SparLow}}
\label{sec:3431}
For labeled dataset $\X_{l}$, we adopt the criterion of LDA and 
compute matrices $\mathbf{L}^{w}$ and $\mathbf{L}^{b}$ as same to Section \ref{sec:3421}. 
Hence the total scatter matrix can be written as 
\begin{equation}
S_{t} ( {{\Phib}}(\D,\X_{l}) ) := 
\Phib\mathbf{L}^{t} \Phib^{\top},
\end{equation}
with $\mathbf{L}^{t} := \mathbf{L}^{w} + \mathbf{L}^{b} \in \mathbb{R}^{n_{l}\times n_{l}}$.
%
Similar to the constructions of Section \ref{sec:3421}, we construct the formulations of
semi-supervised LDA \emph{SparLow} as 
\begin{equation}
\left\{ \!\!\! \begin{array}{rcl}
\mathcal{A}_{\text{sda}}({\Phib}) &\!\!\!\!\! :=\!\!\!\!\! & 
\Phib  \widetilde{\mathbf{L}}^b \Phib^{\top}, \\[1mm]
\mathcal{B}_{\text{sda}}({\Phib}) &\!\!\!\!\! :=\!\!\!\!\! & 
\Phib ( \widetilde{\mathbf{L}}^t + \alpha {\mathbf{L}}) \Phib^{\top},
\end{array} \right.
\end{equation}
where $\alpha \in \mathbb{R}$ controls the influence of labeled Laplacian matrix $\widetilde{\mathbf{L}}^t \in \mathbb{R}^{n\times n}$ 
and global Laplacian matrix $\mathbf{L} \in \mathbb{R}^{n\times n}$.
Therein, $\widetilde{{\mathbf{L}}}^b$ and $\widetilde{{\mathbf{L}}}^t$ are the augmented matrices of 
$\mathbf{L}^b$ and $\mathbf{L}^t$, namely,
\begin{equation*}
\widetilde{\mathbf{L}}^b \!=\! \left[\!\!
\begin{array}{cc}
\mathbf{L}^b                   &\!\!\!\! \mathbf{0}_{n_{u}\times n_{u}} \\
\mathbf{0}_{n_{u}\times n_{u}} &\!\!\!\! \mathbf{0}_{n_{u}\times n_{u}}
\end{array}
\!\!\right]
\text{, and}~~
\widetilde{\mathbf{L}}^t \!=\! \left[\!\!
\begin{array}{cc}
\mathbf{L}^t                   &\!\!\!\! \mathbf{0}_{n_{u}\times n_{u}} \\
\mathbf{0}_{n_{u}\times n_{u}} &\!\!\!\! \mathbf{0}_{n_{u}\times n_{u}}
\end{array}
\!\!\right].
\end{equation*}

%
\subsubsection{Semi-supervised Laplacian \emph{SparLow}}
\label{sec:3432}
We now consider the semi-supervised version of Laplacian \emph{SparLow}.
For the dataset $\X$, let us define 
the nonlocal graph Laplacian matrix $\mathbf{L}^N = \Y^{{N}} - \Z^N$ in $\mathbb{R}^{n \times n}$ with 
$z^N_{ij}$ weighting the edge between non-adjacency data pairs $({\phib}_i, {\phib}_j)$, 
$z^N_{ij} = 0$ otherwise.
$\Y^N$ is diagonal with $y^N_{ii} := \sum_{j \neq i} z^N_{ij}$ for all $i,j$.
For the labeled dataset $\X_{l}$, we adopt the settings of Section \ref{sec:3422}
and define interclass local Laplacian matrix $\mathbf{L}^{-} \in \mathbb{R}^{n_{l} \times n_{l}}$ 
and intraclass local Laplacian matrix $\mathbf{L}^{+} \in \mathbb{R}^{n_{l} \times n_{l}}$.
Hence, we construct the formulations of Semi-supervised Laplacian \emph{SparLow} as 
\begin{equation}
\left\{ \!\!\! \begin{array}{rcl}
\mathcal{A}_{slap}({\Phi}) &\!\!\!\!\! :=\!\!\!\!\! & 
\Phib ( \widetilde{\mathbf{L}}^{-} + \alpha_1\mathbf{L}^N ) \Phib^{\top} \\[1mm]
\mathcal{B}_{slap}({\Phi}) &\!\!\!\!\! :=\!\!\!\!\! & 
\Phib ( \widetilde{\mathbf{L}}^{+}+ \alpha_2 \mathbf{L}^{N}) \Phib^{\top},
\end{array} \right.
\end{equation}
with $\alpha_1 \in \mathbb{R}^{+}, \alpha_2 \in \mathbb{R}^{+}$ control
the influence of labeled Laplacian matrix and unlabeled Laplacian matrix. 
Similar to the setting of Section \ref{sec:3431},
let $\widetilde{\mathbf{L}}^{-} \in \mathbb{R}^{n \times n}$ and 
$\widetilde{\mathbf{L}}^{+} \in \mathbb{R}^{n \times n}$ 
denote the augmented matrices of $\mathbf{L}^{-}$ and $\mathbf{L}^{+}$.

%
\subsubsection{Semi-supervised MVR \emph{SparLow}}
\label{sec:3433}
The supervised linear (label-based) regression (e.g., SVM) associated with a manifold 
regularisation \cite{belk:jmlr06}, is another popular framework for resolving 
semi-supervised learning problems.
In this section, we adopt an MVR model associated with a manifold regularization as
\begin{equation}
\begin{split}
\label{eq_leastSueres_semi}
& \min_{\U,\W}  \|\Z_{l} - \W^\top \U^\top \Phib_{l} \|_F^2  + \rho_1 \|\W\|_F^2 \\
&\qquad + \rho_2 \operatorname{tr} \left( \W^\top \U^\top \Phib \mathbf{L} \Phib^{\top} \U \W \right),
\end{split}
\end{equation}
where $\Z_{l} \in \mathbb{R}^{d \times n_{l}}$ is the target matrix for $\X_{l}$, 
$\U\in \St(l,r)$, $\W\in \mathbb{R}^{l\times d}$ and $\rho_1, \rho_2 \in \mathbb{R}^{+}$.
By fixing $\U$, minimization of the problem as in Eq.~\eqref{eq_leastSueres_semi} with respect to $\W$ leads to a closed-form expression
\begin{equation}
\label{eq_Semi_W}
\W = \big( \U^\top \big( \Phib_{l} \Phib_{l}^\top  + 
\rho_1 \I_{r} + \rho_{2} \Phib L \Phib^{\top} \big) \U \big)^{-1} \U^{\top} 
\Phib_{l} \Z_{l}^\top.
\end{equation}	 
Using this closed-form expression to substitute the $\W$ in Eq.~\eqref{eq_leastSueres_semi}, 
we can rewrite Eq.~\eqref{eq_leastSueres_semi} in the form of \emph{SparLow} as
\begin{equation}
\left\{ \!\!\! \begin{array}{cl}
\mathcal{A}_{\text{smvr}}({\Phib})&\!\!\!:= -\Phib_{l} \Z_{l}^{\top} \Z_{l} \Phib_{l}^{\top} \\
\mathcal{B}_{\text{smvr}}({\Phib})&\!\!\!:= \Phib_{l} \Phib_{l}^\top  + \rho_1 \I_{r} 
+ \rho_{2} \Phib L \Phib^{\top}.
\end{array} \right.
\end{equation}

\section{Optimization Algorithm for \emph{SparLow}}
\label{sec:05}
In this section, we firstly investigate the differentiability of the \emph{SparLow} 
function, and then present a geometric conjugate gradient algorithm that maximizes the 
\emph{SparLow} cost function on the underlying Riemannian manifold.


\subsection{Differentiability of the \emph{SparLow} Function}
\label{sec:51}
%
%
In this subsection, we investigate the differentiability of the \emph{SparLow} function $J$, 
and derive its Euclidean gradient in the embedding space of the product manifold 
$\S(m,r) \times \Gr(l,r) \subset \mathbb{R}^{m \times r} \times 
\mathbb{R}^{r \times r}$, which is the building block for computing the Riemannian gradient 
of $J$ in Section~\ref{sec:52}. 

By the construction of the \emph{SparLow} function being differentiable in the orthogonal
projection $\P$, we can compute the Euclidean gradient of $J$ with respect to $\P$ as
\begin{equation}
\label{eq:eu_grad_p}
	\nabla_{J}(\P) = \frac{\mathcal{A}(\Phi(\D,\X)) - f(\D,\P) \cdot \mathcal{B}(\Phi(\D,\X))}
	{\operatorname{tr}(\P \mathcal{B}(\Phi(\D,\X)))+\sigma}.
\end{equation}
The Euclidean gradient of $J$ with respect to $\D$ consists of three components,
i.e.,
\begin{equation}
\label{eq:eu_grad_d}
	\nabla_{J}(\D) = \nabla_{f}(\D) - \mu_{1} \nabla_{g_{c}}(\D) - 
	\mu_{2} \nabla_{g_{d}}(\D),
\end{equation}
two of which are the Euclidean gradients of the two regularizers, which can be simply 
computed as
\begin{equation}
	\nabla_{g_{d}}(\D) = \D - \D^{*},
\end{equation}
and
\begin{equation}
	\nabla_{g_{c}}(\D) = \D \sum\limits_{1\leq i< j\leq r} 
	\frac{2 \d_{i}^\top \d_{j} }
	{ 1-(\d_{i}^\top \d_{j})^2 } 
	\left(\e_{i}\e_{j}^{\top} + \e_{j}\e_{i}^{\top} \right),
\end{equation}
with $\e_{i} \in \mathbb{R}^{r}$ being the $i$-th basis vector of $\mathbb{R}^{r}$.
The computation of the first component $\nabla_{f}(\D)$ requires the 
differentiability of the sparse representation $\Phib(\D,\X)$.

Given images $\x_{i}$ and dictionary $\D$, let $\phib_{i}^{*} := [\varphi_{1}^{*}, 
\ldots, \varphi_{r}^{*}]^{\top} \in \mathbb{R}^{r}$ be the sparse 
representation given by solving the sparse regression problems as in 
Eq.~\eqref{eq:sparse_regression}. We denote the set of indexes of non-zero 
entries of $\phib_{i}^{*}$, known as the \emph{support} of $\phib_{i}^{*}$, by
\begin{equation}
\label{eq:support}
    \L(\x_{i},\D) := \big\{j \in \{1,\ldots,r\} |
    \varphi^{*}_{j} \neq 0 \big\}.
\end{equation}
The differentiability of the sparse representation $\Phib(\D,\X)$ 
requires the following assumptions.
\begin{assumption}
\label{ass:01}
	A separable regularizer $g$ is strictly convex, and
	each component-wise function $g_{i}$ for $i = 1, \ldots,r$
	is differentiable everywhere
    except at the origin, i.e., $g_{i}'(x) \neq 0$ for $x \neq 0$.
\end{assumption}
\begin{remark}
	The strict convexity of $g$ ensures the uniqueness of 
    solutions of the sample-wise sparse regression, and also enables
    explicit characterizations of the unique global minimum.
	It is easy to show that most popular convex sparsifying regularizers, e.g., 
    the $\ell_{1}$-regularizer, the elastic net regularizer, and
    full unnormalized Kullback-Leibler (KL) divergence, fulfill this 
    assumption.
\end{remark}

\begin{assumption}
\label{ass:02}
	For an arbitrary $\D \in \S(m,r)$, the union of the supports of all
	sparse representations $\phib_{i}^{*}$ for all $i = 1, \ldots,n$ is the complete 
	set of indices of atoms, i.e.,
	\begin{equation}
		\bigcup_{i=1}^{n} \L(\x_{i},\D) = \{1,\ldots,r\}.
	\end{equation}
\end{assumption}
\noindent Assumption~\ref{ass:02} ensures all atoms in the dictionary are updated
at any dictionary $\D$.
Then, we can derive the following proposition about the differentiability of the sparse representation $\Phib(\D,\X)$.
\begin{proposition}
\label{prop:01}
	If both Assumption~\ref{ass:01} and \ref{ass:02} hold true, 
	then the sparse representation $\Phib(\D,\X)$ is differentiable on 
    $\S(m,r) \times \Gr(l,r)$.
\end{proposition}
\noindent
The proof of the proposition is given in Appendix. The proposition leads 
straightforwardly the following corollary about the differentiability of 
the \emph{SparLow} function.
\begin{corollary}
\label{coro:01}
	If both Assumption~\ref{ass:01} and \ref{ass:02} hold true, 
	then the \emph{SparLow} function $J$ defined in Eq.~\eqref{eq_main_function} 
	is differentiable on $\S(m,r) \times \Gr(l,r)$.
\end{corollary}

Finally, in order to develop gradient-based algorithms, we need the first 
derivative of $\Phib(\D,\X)$ to admit closed-form expression.
We refer to Appendix for more details and the proof of the following result.
\begin{proposition}
\label{prop:02}
	Let a separable regularizer $g$ satisfy Assumption~\ref{ass:01}.
    If each component-wise function $g_{i}$ has non-degenerate Hessian except
	at the origin, i.e., $g_{i}''(x) > 0$ for $x \neq 0$,
    where $g_{i}''(x)$ denotes the second derivative of function $g_{i}$,
    then the first derivative of $\Phib(\D,\X)$ has a closed-form expression.
\end{proposition}

In order to compute the Euclidean gradient of $f$ with respect to $\D$, we need to compute 
the first derivative of $f$ at $\D \in \mathbb{R}^{m \times r}$ in direction
$\Xib \in \mathbb{R}^{m \times r}$, as
\begin{equation}
\label{eq:derg}
\begin{split}
	\Dr_{1}\!f(\D) \Xib = \operatorname{tr}\!\bigg( &
	\frac{\P \big(\Dr\!\mathcal{A}(\Phib)\circ \Dr \Phi(\D,\X) \Xib \big)}
	{\operatorname{tr}\big(\P \mathcal{B}(\Phi(\D,\X))\big)+\sigma} \\
	& - 
	\frac{f(\D,\P) \!\cdot\! \Dr\!\mathcal{B}(\Phib) \circ \Dr \Phi(\D,\X) \Xib}
	{\operatorname{tr}\big(\P \mathcal{B}(\Phi(\D,\X))\big)+\sigma} \bigg),\!
\end{split}
\end{equation}
where $\Dr\! \mathcal{A}(\Phib) \colon \mathbb{R}^{m \times r} \to \mathbb{R}^{r \times r}$ and 
$\Dr\! \mathcal{B}(\Phib) \colon \mathbb{R}^{m \times r} \to \mathbb{R}^{r \times r}$ are
the directional derivatives of $\mathcal{A}(\cdot)$ and $\mathcal{B}(\cdot)$, respectively.
Its calculation is dependent on the concrete construction of the two matrix-valued 
functions $\mathcal{A}$ and $\mathcal{B}$.
Although there are various constructions of $\mathcal{A}(\Phib)$ and $\mathcal{B}(\Phib)$
for different learning paradigms, see~Section~\ref{sec:04},
there are two basic forms, namely, $\Phib \Z \Phib^{\top}$ and $\operatorname{tr}(\Phib \Z 
\Phib^{\top}) \I_{r}$ with $\Z \in \mathbb{R}^{n \times n}$ being some structure matrix 
specified in Section~\ref{sec:04}.
The tedious computation of Eq.~\eqref{eq:derg} can be generalized in the following 
form 
\begin{equation}
\label{eq:derf}
	\operatorname{D}_{1}\!f(\D) \Xib
	= \operatorname{tr} \!\bigg(\! \underbrace{
	\frac{ \widetilde{\mathcal{A}}(\Phib)\!-\!f(\D,\P) \widetilde{\mathcal{B}}(\Phib)}
	{\operatorname{tr}\!\big(\P \mathcal{B}(\Phi(\D,\X))\big)\!+\!\sigma}
	}_{=:\Z \in \mathbb{R}^{n \times n}}
	\Dr \Phi(\D,\X) \Xib \bigg),\!\!
\end{equation}
where $\Dr \Phi(\D,\X) \Xib \in \mathbb{R}^{r \times n}$ is the first derivative of 
$\Phib$ with respect to the dictionary $\D$.
Here, $\widetilde{\mathcal{A}}$ and $\widetilde{\mathcal{B}}$ are two matrix-valued functions
dependent on the specific choice of $\mathcal{A}(\cdot)$ and $\mathcal{B}(\cdot)$.
We summarize the general formula for $\widetilde{\mathcal{A}}$ and $\widetilde{\mathcal{B}}$
for Eq.~\eqref{eq:derf}.
When $\mathcal{A}(\Phib) = \Phib \Z \Phib^{\top}$, we have
\begin{equation}
	\widetilde{\mathcal{A}}(\Phib) = (\Z \Phib \P)^{\top} + \Z \Phib^{\top} \P.
\end{equation}
When $\mathcal{B}(\Phib) = \operatorname{tr}(\Phib \Z \Phib^{\top}) \I_{r}$, we have
\begin{equation}
	\widetilde{\mathcal{B}}(\Phib) = \Z^{\top} \Phib^{\top} + \Z \Phib^{\top}.
\end{equation}

Let $\L_{i}$ be a shorthand notation for $\L(\x_{i},\D)$ and $k_{i} := |\L_{i}|$ 
denote the cardinality of $\L_{i}$.
We denote further ${\phib}_{\L_{i}} := \{\varphi_{i,j}\} \in \mathbb{R}^{k_{i}}$ with $j \in \L_{i}$ 
and ${\D}_{\L_{i}} \in \mathbb{R}^{m \times k_{i}}$ being the subset of ${\D}$, in which 
the index of atoms (columns) fall into the support $\L_{i}$.
Finally, by recalling the first derivative of the generic sparse coding $\phib$ 
as computed in 
Eq.~(12) in Appendix, 
we compute the Euclidean gradient as
\begin{equation}
\begin{split}
	\nabla_{f}(\D) = \sum\limits_{i=1}^{n} \mathcal{V} \Big( &
	\x_{i} \Z_{\L_{i}} K_{i}^{-1} - \D_{\L_{i}} \phib_{\L_{i}} 
	\Z_{\L_{i}} K_{i}^{-1} \\[-1mm]
	& - \D_{\L_{i}} K_{i}^{-1} \Z_{\L_{i}}^{\top} \phib_{\L_{i}} \Big),
\end{split}
\end{equation}
where $K_{i} := \D_{\L_{i}}^{\top} \D_{\L_{i}} + \mathsf{H} 
g(\phib_{\L_{i}})$ with $\mathsf{H}g(\phib_{\L_{i}})$ being the Hessian matrix
of the regularizer $g$ defined on $\phib_{\L_{i}}$.
Here, $\mathcal{V} \colon \mathbb{R}^{m \times k_{i}} \to \mathbb{R}^{m \times r}$ 
produces a matrix by replacing columns of the zero matrix in $\mathbb{R}^{m \times r}$ 
with columns of matrix $\mathbf{Y} \in \mathbb{R}^{m \times k_{i}}$ according to 
the support $\L_{i}$.

\subsection{A Geometric CG \emph{SparLow} Algorithm}
\label{sec:52}
In this subsection, we present a geometric CG algorithm on the product manifold 
$\Mf := \S(m,r) \times \Gr(l,r)$ to maximize the \emph{SparLow} 
function $J$ as defined in Eq.~\eqref{eq_main_function}.
It is well known that CG algorithms offer prominent properties, such as a superlinear rate of 
convergence and the applicability to large scale optimization problems with low 
computational complexity, e.g., in sparse recovery \cite{hawe:tip13}.
We refer to \cite{absi:book08, hawe:tip13} for further technical details for these computations. 
\begin{algorithm}[t!]
\caption{A CG-\emph{SparLow} Framework.} 
\label{algo:sparlow} 
\SetAlgoNoLine
	\SetKwHangingKw{IN}{Input~:}
	\SetKwHangingKw{Sa}{Step~1:}
	\SetKwHangingKw{Sb}{Step~2:}
	\SetKwHangingKw{Sc}{Step~3:}
	\SetKwHangingKw{Sd}{Step~4:}
	\SetKwHangingKw{Se}{Step~5:}
	\SetKwHangingKw{Sf}{Step~6:}
	\SetKwHangingKw{Sg}{Step~7:}
	\SetKwHangingKw{OUT}{\hspace{-1.6mm}Output:}

	\IN{$\X \in \mathbb{R}^{m\times n}$ and functions $\mathcal{A} \colon 
		\mathbb{R}^{r \times n} \to \mathbb{R}^{r \times r}$ and $\mathcal{B} 
		\colon \mathbb{R}^{r \times n} \to \mathbb{R}^{r \times r}$ 
		as specified in Section~\ref{sec:04} \vspace{1mm}}
		
	\OUT{Accumulation point $(\D^{*}\!,\P^{*}) \!\in\! \S(m,r) 
		\!\times\! \Gr(l,r)\!$ \vspace{1mm}}			
	
	\Sa{Given an initial guess $\D^{(0)} \in \S(m,r)$ and $\P^{(0)} \in 
		\Gr(l,r)$
		\vspace{1mm}}
	
	\Sb{Set $j=j+1$, let $\big(\D^{(j)},\P^{(j)}\big) = \big(\D^{(j-1)},\P^{(j-1)}\big)$, 
		and compute the Riemannian gradient \vspace{0.1mm}
	
		\hspace{20mm} $\G^{(j)} = \H^{(j)} = \operatorname{grad}_{\!J}\big(\D^{(j)},\P^{(j)}\big)$
		\vspace{1mm}}
	
	\Sc{Set $\M^{(j)} = \big(\D^{(j)},\P^{(j)}\big)$ \vspace{1mm}}			
	\begin{minipage}{0.43\textwidth}
		\begin{enumerate}[(i)]
			\item Update $\M^{(j)} \gets \Gamma_{\M^{(j)},\H^{(j)}}(t^{*})$, 
				where \vspace{2mm}
			
			\hspace{10mm} $t^{*} = \displaystyle \operatorname*{argmax}_{t \in \mathbb{R}}
				J \circ \Gamma_{\M^{(j)},\H^{(j)}}(t)$; \\[0mm]

			\item Compute $\G^{(j+1)} = \operatorname{grad}_{\!J}(\M^{(j)})$; \\[-2mm]

			\item Update $\H^{(j+1)} \!\gets\! \G^{(j+1)} \!+\! 
			\beta~\T_{\M^{(j)},t^{*}\H^{(j)}}(\H^{(j)})$, where $\beta$ is chosen 
			such that $\T_{\M^{(j)},t^{*}\H^{(j)}}(\G^{(j)})$ and $\H^{(j+1)}$ 
			conjugate with respect to the Hessian of $J$ at $\M^{(j)}$. \\[-2mm]
		\end{enumerate}
	\end{minipage}
	
	\Sd{If $\big\|\M^{(j+1)}-\M^{(j)}\big\|$ is small enough, stop. Otherwise, go to Step 2
		\vspace{1mm}}
\end{algorithm}
 \begin{figure}
	\vspace{-8mm}
	\begin{center}
		\begin{tikzpicture}[scale=0.9, inner sep=0pt,dot/.style={fill=black,circle,
			minimum size=7pt},scale=0.3]
		\draw[line width=1pt] (-9,0) .. controls (0,10) and (11,10) .. (22,-2);
		\draw[line width=1pt] (-9,0) .. controls (-3,1) and (1,-3) .. (2,-6);
		\draw[line width=1pt] (2,-6) .. controls (3,-4) and (8,1) .. (22,-2);
		\node[] at (2.2,-4) {$\Mf$};
		\draw[fill=green, opacity=0.25] (-7,-3) -- (-4,5) -- (5.5,7) -- (3,-1) -- cycle;
		\draw[fill=green, opacity=0.25] (8, -1.5) -- (7,7) -- (16.5, 4) -- (18.5,-3.75) -- cycle;
		\node[] at (-2,1.2) {${\M}^{(j)}$};
		\node[] at (1,-0.5) {$T_{{\M}^{(j)}} \Mf$};
		\node[dot,scale=0.5] (xk) at (-2,2) {};
		\draw[-{>[scale=2.5, length=2,width=2]},line width=0.7pt] (-2,2) -- (-0,4.5);
		\node[] at (0.5,5) {\footnotesize ${\H}^{(j)}$};
		\draw[-{>[scale=2.5, length=2,width=2]},red,line width=0.7pt] (-2,2) .. controls (1,5) and (6,5.1) .. (10,4);
		\node[] at (5.3,3.4) {\footnotesize 
			$\Gamma_{{\M}^{(j)},{\H}^{(j)}}(t^{(j)})$
		};
		\node[] at (10.2,5.0) {${\M}^{(j+1)}$};
		\node[] at (15.4,-2.25) {$T_{{\M}^{(j+1)}} \Mf$};
		\node[dot,scale=0.5] (xk1) at (10,4) {};
		\draw[-{>[scale=2.5, length=2,width=2]},line width=0.7pt] (10,4) -- (14,2.7);
		\node[] at (18.15, 3.5) {\footnotesize 
			$\mathcal{T}_{  {\M}^{(j)},{t}^{(j)} {\H}^{(j)} }({\H}^{(j)})$
		};
		\draw[-{>[scale=2.5, length=2,width=2]},line width=0.7pt] (10,4) -- (15.6, 5.6);
		\node[] at (16.6, 6.6) {\footnotesize $\nabla_{J}({\M}^{(j+1)})$};
		\draw[-{>[scale=2.5, length=2,width=2]},line width=0.7pt] (10,4) -- (11,1.2);
		\node[] at (8.5, 0.4) {\footnotesize $\grad_{J}({\M}^{(j+1)})$ };
		\draw[-{>[scale=2.5, length=2,width=2]},line width=0.7pt] (10,4) -- (15,0);
		\draw[dotted] (14,2.7) -- (15,0);
		\draw[dotted] (11,1.2) -- (15,0);
		\node[] at (16.3,-0.3) {\footnotesize ${\H}^{(j+1)}$};
		\draw[-{>[scale=2.5, length=2,width=2]},red,line width=0.7pt] (10,4) .. controls (11,3) and (12,2) .. (13,-0.5);
		\end{tikzpicture}
	\end{center}
	\vspace{-3mm}
	\caption{ A CG update from the point ${\M}^{(j)}$ to 
		      the point ${\M}^{(j+1)}$ on a manifold $\Mf$.
		Tangent space at $\M \in \Mf$: $T_{{\M}} \Mf$;  
		a CG search direction at ${\M}$: ${\H} \in T_{{\M}} \Mf$;
		the Euclidean gradient $\nabla_{J}({\M})$ and its induced Riemannian gradient: $\grad_{J}({\M}) \in T_{{\M}}\Mf$;
		retraction: $\Gamma_{{\M}} \colon T_{{\M}}\Mf \to \Mf$;
		vector transport: $\mathcal{T}_{{\M},{\H}} \colon T_{{\M}}\Mf \to  T_{\Gamma_{\M,\H}(t)}\Mf $.
	}
	\label{fig_3_CGD}
	\vspace{-5mm}
\end{figure}

Classic geometric CG algorithms require the concepts of geodesic and parallel transport, 
which are often more computationally demanding.
In this work, we adopt an alternative approach based on the concept of retraction and its
corresponding vector transport.
A generic framework of our CG-SparLow algorithm is summarized in Algorithm~\ref{algo:sparlow}, 
and further illustrated in Fig.~\ref{fig_3_CGD}. 
In the rest of this section, we explain the key technical details of the CG-SparLow algorithm.

%
%
Firstly, we recall some basic geometry of product manifold $\Mf$. 
%
%
We denote the tangent space of $\Mf$ at $\M := (\D,\P)$ by
\begin{equation}
	T_{\M}\Mf := T_{\D}\S(m,r) \times T_{\P}\Gr(l,r),
\end{equation}
where $T_{\D}\S(m,r)$ and $T_{\P}\Gr(l,r)$ denote the tangent space of the
product of spheres $\S(m,r)$ and the Grassmann manifold $\Gr(l,r)$, respectively.
We endow the manifold $\Mf$ with the Riemannian metric inherited from 
the surrounding Euclidean space, i.e.,
\begin{equation}
\label{eq:riemannianmetric}
	\langle (\D_1, \P_1), (\D_2, \P_2) \rangle_R := \tr (\D_1 \D_2^\top)+
	\tr(\P_1 \P_2^\top),
\end{equation}
with $\D_1 \in \mathbb{R}^{m \times r}$ and $\P_1 \in \mathbb{R}^{r \times r}$.

In each sweep of the CG algorithm (from Step 2 to Step 4 in Algorithm~\ref{algo:sparlow}),
after computing the Euclidean gradient of $J$ as computed in Eq.~\eqref{eq:eu_grad_p} 
and \eqref{eq:eu_grad_d}, we compute the Riemannian gradient of $J$ with respect to 
the Riemannian metric as defined in Eq.~\eqref{eq:riemannianmetric} via the associated
projection on the tangent space $T_{\M}\Mf$. Concretely, we have
\begin{equation} 
	 \grad_{J}(\D,\P) := \big( \grad_{1}\!J(\D,\P), \grad_{2}\!J(\D,\P) \big),
\end{equation}
with $\grad_{1}\!J(\D,\P)$ and $\grad_{2}\!J(\D,\P)$ being the Riemannian gradients of 
$J$ with respect to the first and the second parameter, respectively.
Specifically, we have
\begin{equation}
	\grad_{1}\!J(\D,\P) := \nabla_{J}(\D) - \D \operatorname{ddiag}(\D^{\top} 
	\nabla_{J}(\D))
\end{equation}
with $\operatorname{ddiag}(\cdot)$ putting the diagonal entries of a square matrix into 
a diagonal matrix form, and
\begin{equation}
	\grad_{2}\!J(\D,\P) := \P \nabla_{\!J}(\P) + \nabla_{\!J}(\P) \P - 2 \P 
	\big( \nabla_{J}(\P) \big) \P.
\end{equation}

A retraction $\Gamma_{\M} \colon T_{\M} \Mf \to \Mf$ is a smooth mapping
from the tangent space $T_{\M} \Mf$ to the manifold, such that the evaluation 
$\Gamma_{\M}(0) =  \M$ and the derivative $\operatorname{D} \Gamma_{\M}(0) \colon T_{\M} \Mf \to T_{\M} \Mf$ is the identity mapping. 
For the unit spheres, we restrict ourselves to the following retraction, 
for $\d \in \S(m,1)$ and $\xib \in T_{\d} \S(m,1)$
\begin{equation}
\label{eq:sphere_retraction}
	\gamma_{\d, \xib}(t) := \frac{{\d} + 
	t{\xib}}{\| \d + t{\xib} \|_{2}} \in \S(m,1).
\end{equation}
For constructing a retraction on the Grassmann manifold $\Gr(l,m)$, we need a map 
with $\P \in \Gr(l,m)$ and $\Psib \in T_{\P}\Gr(l,m)$
\begin{equation}
	\zeta_{\P,\Psib}(t) := \big( \I_{r} + t(\Psib \P - \P \Psib) \big)_{Q},
\end{equation}
where $t > 0$ is the step size, and $(\cdot)_{Q}$ is the unique QR decomposition of an 
invertible matrix, i.e., all diagonal entries of the upper triangular part are 
positive.
Then we define the following retraction on $\Gr(l,m)$ as
\begin{equation}
\label{eq:grassman_retraction}
	\gamma_{\P, \Psib}(t) := \zeta_{\P,\Psib}(t) \P (\zeta_{\P,\Psib}(t))^{\top}
	\in \Gr(l,m).
\end{equation}
In concatenation, we construct a retraction on $\Mf := \S(m,r) \times \Gr(l,r)$ with
$\M := (\D,\P) \in \Mf$ and $\H := (\Xib,\Psib) \in T_{\M}\Mf$ 
as
\begin{equation}
	\Gamma_{\M,\H}(t) := \big( [\gamma_{\d_{i},\xib_{i}}(t)]_{i=1,\ldots,k},
	\gamma_{\P, \Psib}(t) \big),
\end{equation}
which is used for implementing a line search algorithm on $\Mf$ 
in Step~3-(i), see~\cite{hawe:tip13}.
Finally, by recalling the vector transport on $\S(m,1)$ with respect
to the retraction $\gamma_{\d, \xib}(t)$ in Eq.~\eqref{eq:sphere_retraction} as
\begin{equation}
	\tau_{\d,t \xib}(\widetilde{\xib}) := \frac{1}{\| {\d} + t{\xib} \|_{2}} 
	\left( \I_{n} + \frac{(\d + t{\xib} )(\d + 
	t{\xib})^{\top}}{\| \d + t{\xib} \|_{2}^{2}} \right) \widetilde{\xib},
\end{equation}
and the vector transport on $\Gr(l,m)$ with respect to the retraction 
$\gamma_{\P, \Psib}(t)$ in Eq.~\eqref{eq:grassman_retraction} as
\begin{equation}
	\tau_{\P,t \Psib}(\widetilde{\Psib}) := 
	\zeta_{\P,\Psib}(t) \widetilde{\Psib} (\zeta_{\P,\Psib}(t))^{\top},
\end{equation}
we define the vector transport of $\widetilde{\H} := (\widetilde{\Xib},
\widetilde{\Psib}) \in T_{(\D,\P)}\Mf$ with respect to the retraction 
$\Gamma_{\M,\H}(t)$ in the direction $t \H = (\Xi,\Psib) \in T_{(\D,\P)}\Mf$, 
denoted by $\T_{\M,\H} \colon T_{\M}\Mf \to T_{\Gamma_{\M,\H}(t)}\Mf$, as
\begin{equation}
	\T_{\M,t\H}(\widetilde{\H}) := \big( [\tau_{\d_{i},t\xib_{i}}
	(\widetilde{\xi}_{i})]_{i=1,\ldots,k}, \tau_{\P,t \Psib}(\widetilde{\Psib}) \big).
\end{equation}
For updating the direction parameter $\beta$ in Step~3-(iii),
we employ a formula proposed in \cite{klei:icassp07}
\begin{equation}
\label{eq:88}
	\beta^{KH} = \frac{\big\langle \G^{(j+1)},~\G^{(j+1)} - 
	\mathcal{T}_{\M^{(j)},t \H^{(j)}} \G^{(j)} \big\rangle_{R}}
	{\big\langle \H^{(j)},~\G^{(j)} \big\rangle_{R}}.
\end{equation}

\begin{figure*}[!ht]
	\centering
	\subfigure[\emph{Unsupervised \emph{SparLow} with or without} $g_{d}$]{
		\label{fig:regu1}
		\includegraphics[width=0.3\textwidth,height=0.2\textwidth]
		{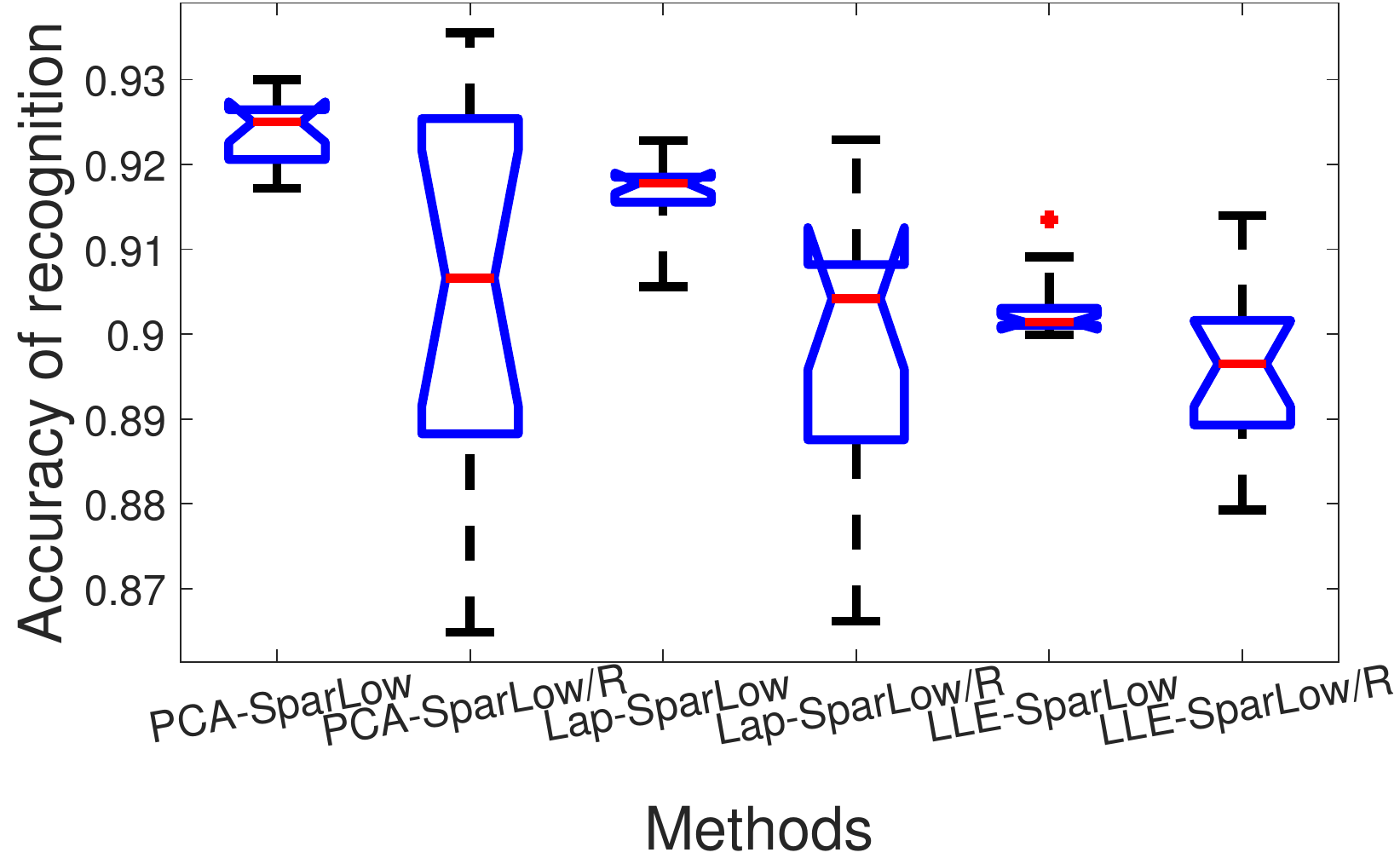}
	}
	\hspace{3mm}
	\subfigure[\emph{PCA-SparLow}]{
		\label{fig:regu2}
		\includegraphics[width=0.3\textwidth]
		{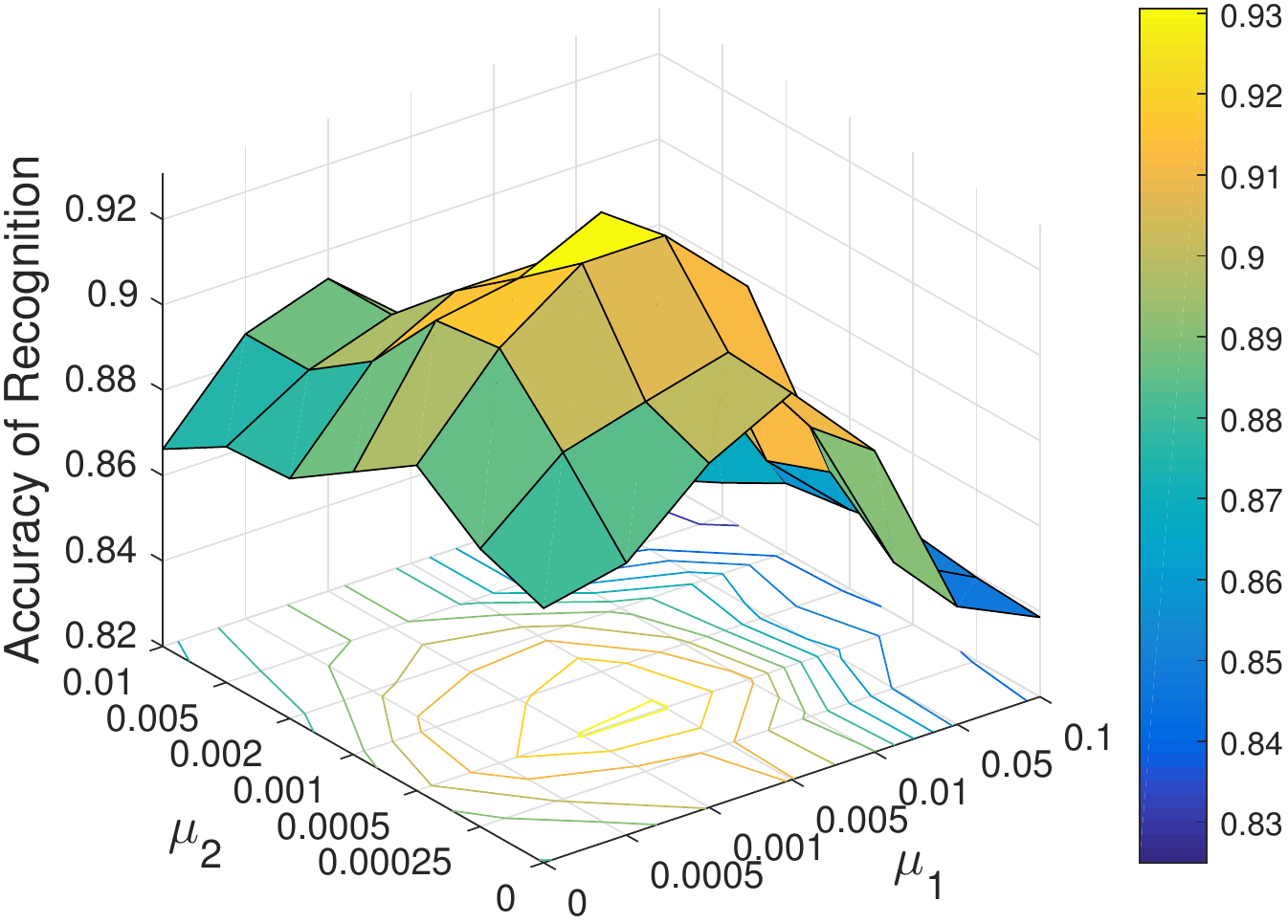}
	}
	\hspace{3mm}
	\subfigure[\emph{LDA-SparLow}]{
		\label{fig:regu3}
		\includegraphics[width=0.3\textwidth]
		{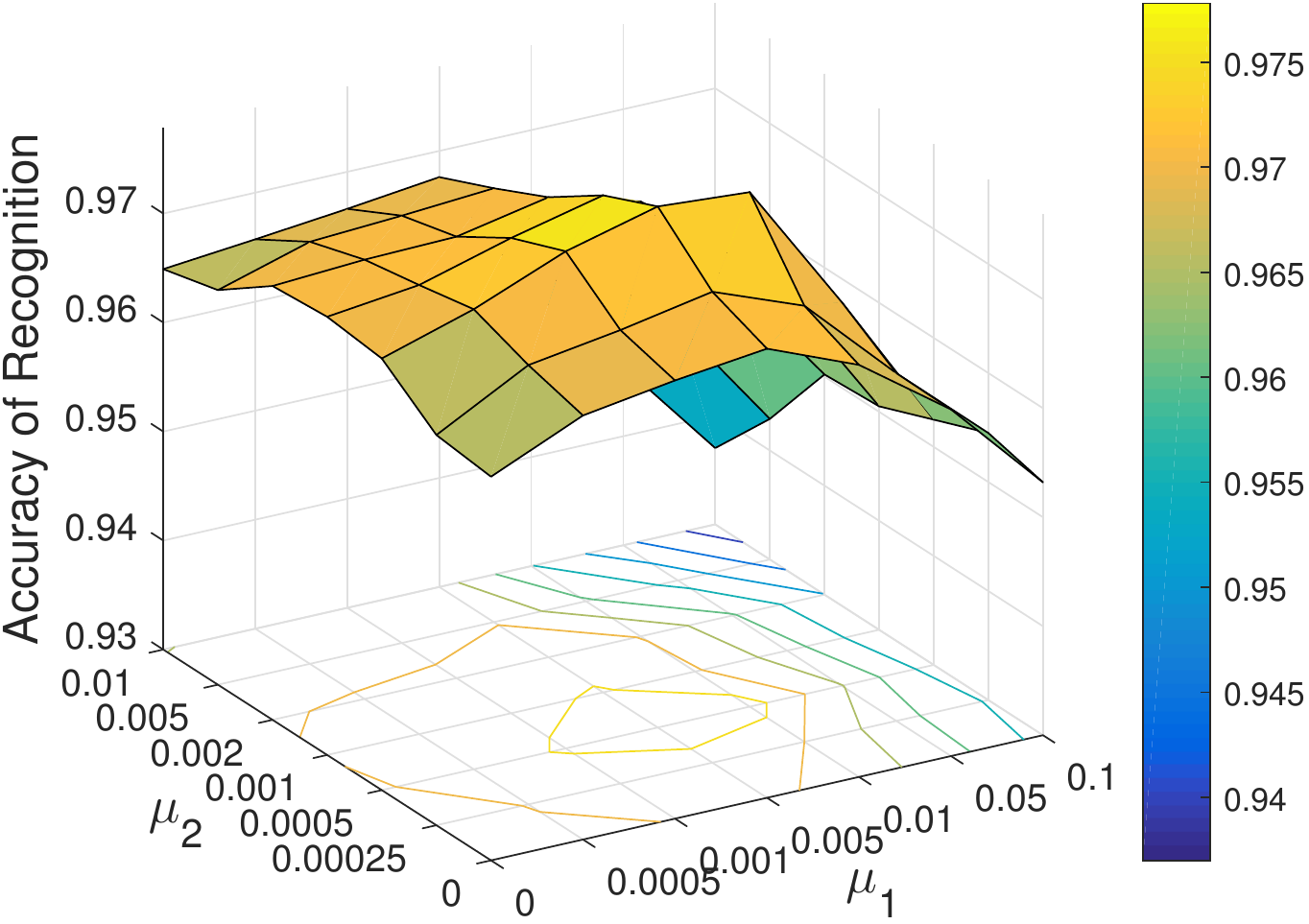}
	}
	\vspace{-1mm}
	\caption{
		\label{fig:regu}
		Impact of the regularizers to the recognition rate on the USPS digits
		(\emph{SparLow/R} refers to \emph{PCA-SparLow} methods without $g_{d}$).
	}
\end{figure*}
\vspace{-1.5mm}

\section{Experimental Evaluations}
\label{sec:06}
In this section, we investigate performance of our proposed \textit{SparLow} 
framework in several image processing applications.
For the convenience of referencing, we adopt the following fashion to 
name the algorithms in comparison: for example, the PCA-like \emph{SparLow} 
algorithm described in Section~\ref{sec:3411} is referred to as 
the \emph{PCA-SparLow} and its sequential learning counterpart, which
directly applies a PCA on the corresponding sparse representations, as \emph{SparPCA}.
\subsection{Experimental Settings}
\label{sec:61}
In unsupervised learning experiments, we employ the K-SVD algorithm 
\cite{ahar:tsp06} to compute an empirically optimal data-driven dictionary,
then initialize \emph{SparLow} algorithms with its column-wise normalized copy,
as required by the regularizer $g_{d}$.
For both supervised and semi-supervised learning, we adopt the same approach to generate
a sub-dictionary for each class, and then concatenate all sub-dictionaries to form 
a common dictionary.
%
\begin{figure}[!t]
	\centering
	\includegraphics[width=0.35\textwidth]
	{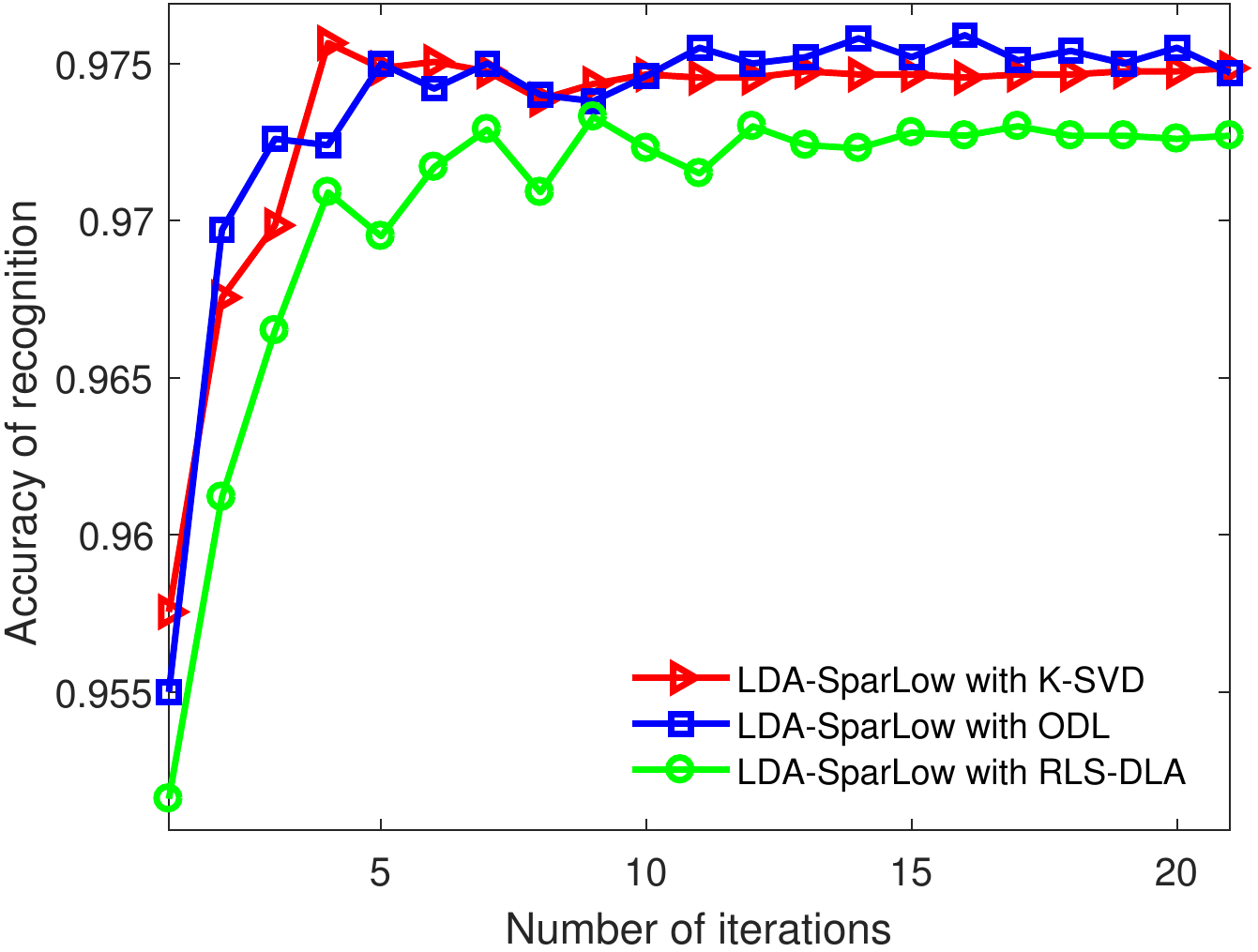}
	\vspace{-2mm}
	\caption{
		\label{fig:diff_init}
	 Trace of performance over optimization process initialized
        with different sparse coding methods on 15-Scenes dataset.
	}
	\vspace{-3mm}
\end{figure}
 In order to compare the performance of the \emph{SparLow} system with different initial dictionaries,
	Fig.~\ref{fig:diff_init} depicts the optimization process of \emph{LDA-SparLow} performed on the 15-Scenes dataset,
	with the data-driven dictionary being learned by K-SVD, ODL \cite{mair:jmlr10} and RLS-DLR \cite{skre:tsp10}, respectively.
%
%
Furthermore, Fig.~\ref{fig_SupervisedSparLow_convergence} depicts the optimization process of supervised \emph{SparLow} methods
 with the data-driven dictionary being learned by K-SVD.


%
With the initial dictionary $\D^{(0)} \in \S(m,r)$ being given, the initial 
orthogonal projection $\P^{(0)} \in \Gr(l,r)$ can be directly obtained by applying 
classic TQ maximization algorithms on the sparse representations of the
samples with respect to $\D^{(0)}$. 
Certainly, when the number of training samples is huge, it is unnecessary to
perform a TQ maximization in order to generate an initialization.
Instead, we employ only a selection of random samples to compute the initial 
orthogonal projection $\P^{(0)}$.

In all experiments, we choose $\sigma = 10^{-3}$ by hand in Eq.~\eqref{eq_main_cost_no_regular}.
 The parameters for $\mu_1, \mu_2, r, l$ in Eq.~\eqref{eq_main_function} and 
$\lambda_{1}, \lambda_{2}$ in Eq.~\eqref{eq:elastic} or $\lambda$ in \textit{KL-divergence} \cite{bagn:nips09}
could be well tuned via performing cross validation.
Images are presented as $m$-dimensional vectors, and normalized to have unit norm.
%
%
For datasets without a pre-construction of training set and testing set,
all experiments are repeated ten times with different randomly constructed training set
and test set, and the average of per-class recognition rates is recorded for each run. 
For most of our experiments, we employ the \textit{elastic net} method 
\cite{zouh:rssb05} to solve the sparse coding problem \eqref{eq:sparse_regression}.
An alternative solution based on \textit{KL-divergence} is also evaluated in the 
application of large scale image processing in Section~\ref{sec:54}.

{Fig.}~\ref{fig_pie_dic_size} plots the recognition 
rates of \emph{LDA-SparLow}, \emph{MFA-SparLow}, SRC, FDDL 
\cite{yang:ijcv14}, and LC-KSVD with varying dictionary sizes (number of atoms). 
In all cases, the proposed methods perform better than SRC and FDDL, 
and {give} significant improvement to LC-KSVD and TDDL. 
This also confirms that an increasing dimension of spare representation
can enhance linear separability for image classification, as
observed in \cite{tehy:jmlr03, ranz:nips06}.

%
\begin{figure}[!t]
	\centering
	\includegraphics[width=0.35\textwidth]
	{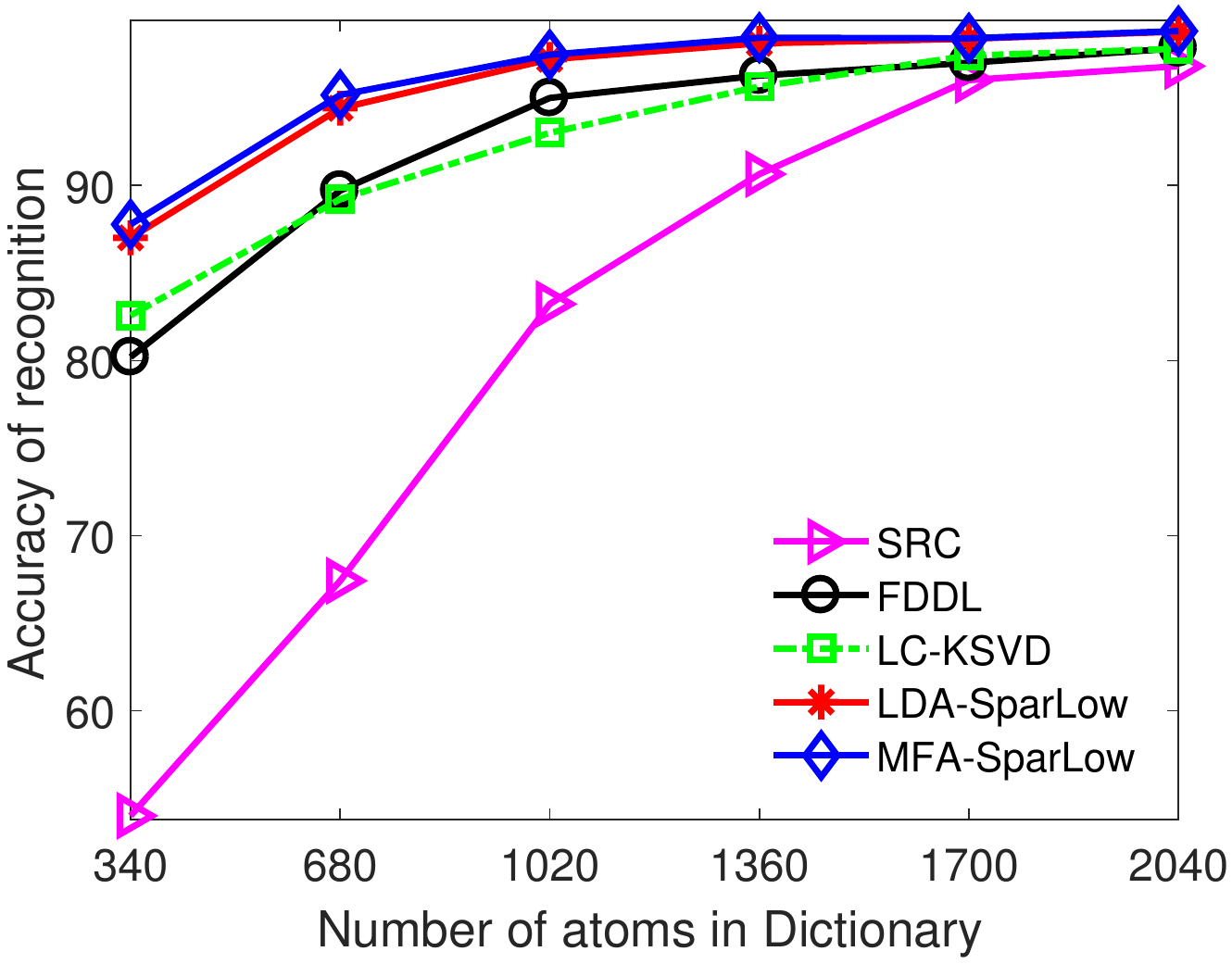}
	\vspace{-2mm}
	\caption{
		\label{fig_pie_dic_size}
		Comparison on recognition results with different dictionary sizes for PIE faces. 
			The classifier is 1NN.
	}
	\vspace{-3mm}
\end{figure}
%
\subsection{Tuning of Parameters}
\label{sec:62}
Training a \emph{SparLow} model can be computationally expensive.
We firstly investigate the impact of various factors of the learning model.
All experiments in this subsection were conducted on the USPS dataset \cite{hull:pami94},
which contains $7291$ training images and $2007$ testing images.
After applying \textit{SparLow} models on the images to produce the corresponding
low dimensional representations, we employ the one-nearest neighbor ($1$NN) 
method to test the performance of the \textit{SparLow} in terms of classification.

\subsubsection{Weighing the Regularisers}
\label{sec:521}
Here, we investigate the impact of the two regularizers $g_{c}$ and $g_{d}$ 
on the performance of the \emph{SparLow}, i.e., the inference of weighing parameters
$\mu_{1}$ and $\mu_{2}$ in Eq.~\eqref{eq_main_function}.
Firstly, we test a special case that $\mu_{2} = 0$, i.e., without the \emph{data
regularizer}.
Fig.~\ref{fig:regu1} shows the box plot of results of applying the $1$NN classification 
ten times on the USPS database with random initializations.
As usual, the recognition accuracy is chosen as 
	the lowest one of recognition results after the 
	algorithm running $20$ iterations of each run, i.e., the converged value of recognition in this work.
The results suggest that the regularizer $g_{d}$ has the capability of ensuring 
good reconstruction, and achieving stable discriminations.
	Fig.~\ref{fig_SupervisedSparLow_convergence} depicts the trace of performance over optimization process 
	of supervised \emph{SparLow} on CMU PIE faces with $n_{\mathrm{train}} = 120$ and USPS, respectively.
	All sub-figures in Fig.~\ref{fig_SupervisedSparLow_convergence} and Fig.~\ref{fig:regu1} show that
	the regularizers $g_{c}$ and $g_{d}$ can highly improve the stability of recognition accuracy
	after convergence.
\begin{figure*}[!ht]
	\centering
	\subfigure[\emph{Unsupervised \emph{SparLow}}]{
		\label{fig:redu1}
		\includegraphics[width=0.3\textwidth]
		{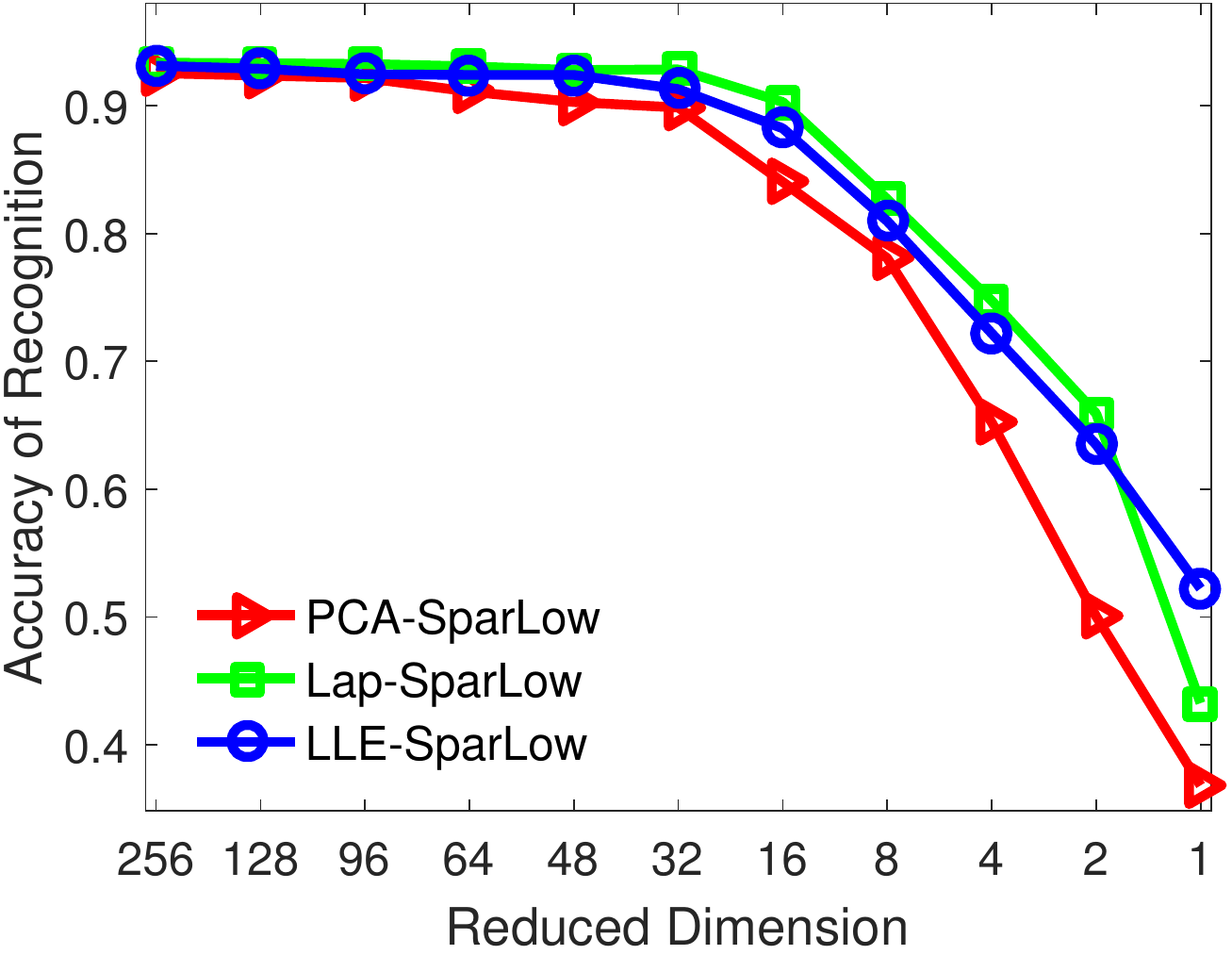}
	}
	\hspace{3mm}
	\subfigure[\emph{Supervised \emph{SparLow}}]{
		\label{fig:redu2}
		\includegraphics[width=0.3\textwidth]
		{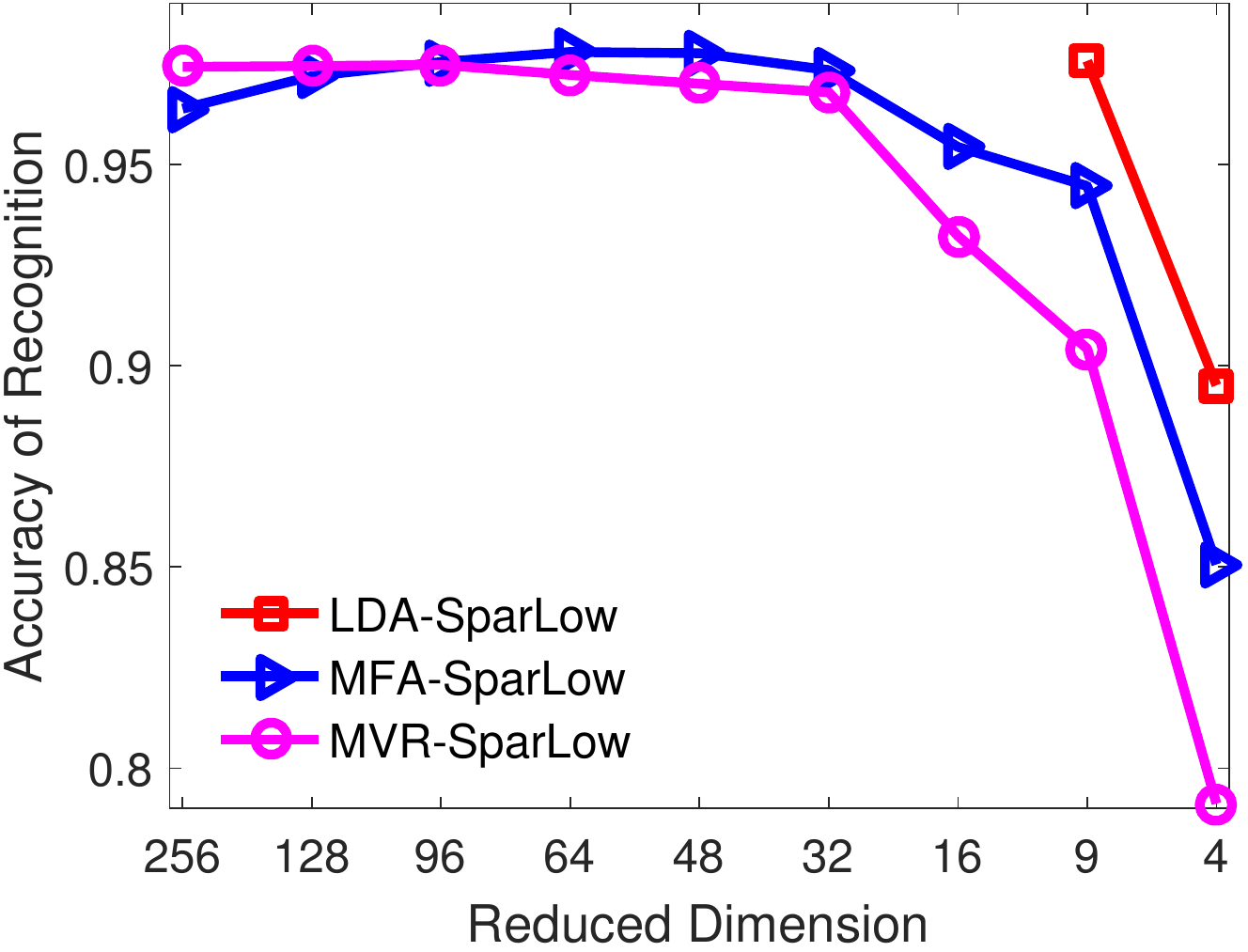}
	}
	\hspace{3mm}
	\subfigure[\emph{Semi-supervised \emph{SparLow}}]{
		\label{fig:semi3}
		\includegraphics[width=0.288\textwidth]
		{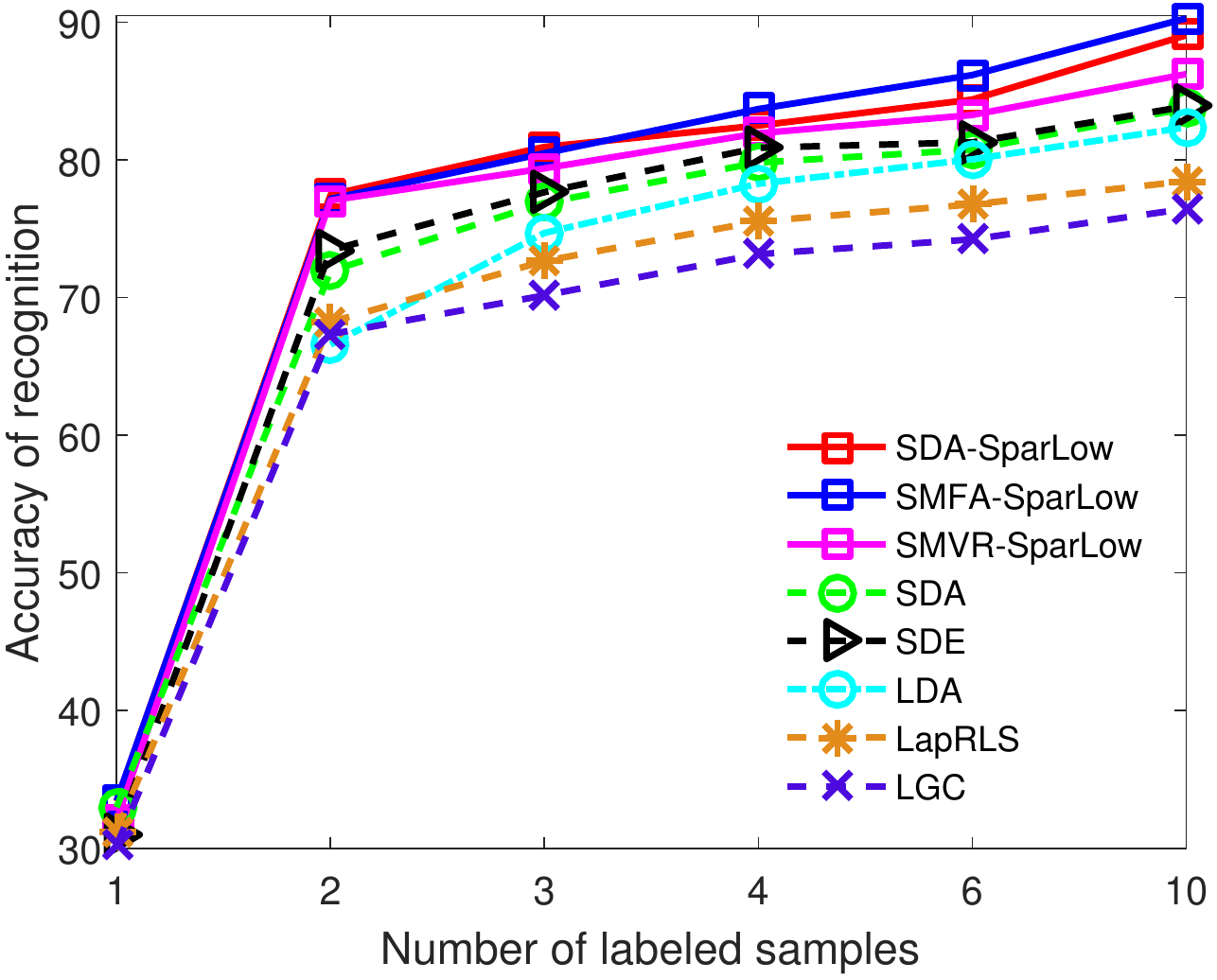}
	}
	\vspace{-1mm}
	\caption{
		\label{fig:redu_semi}
		Impact of targeted low dimensionality and number of labelled samples 
		to the recognition rate of $1$NN classification on the USPS digits.
	}
\end{figure*}
\begin{figure*}[htb!]
	\centering
	\subfigure[PIE: with $g_{c}$ and $g_{d}$ ]{
			\label{fig_pie_results1} 
		\includegraphics[width=1.65in]{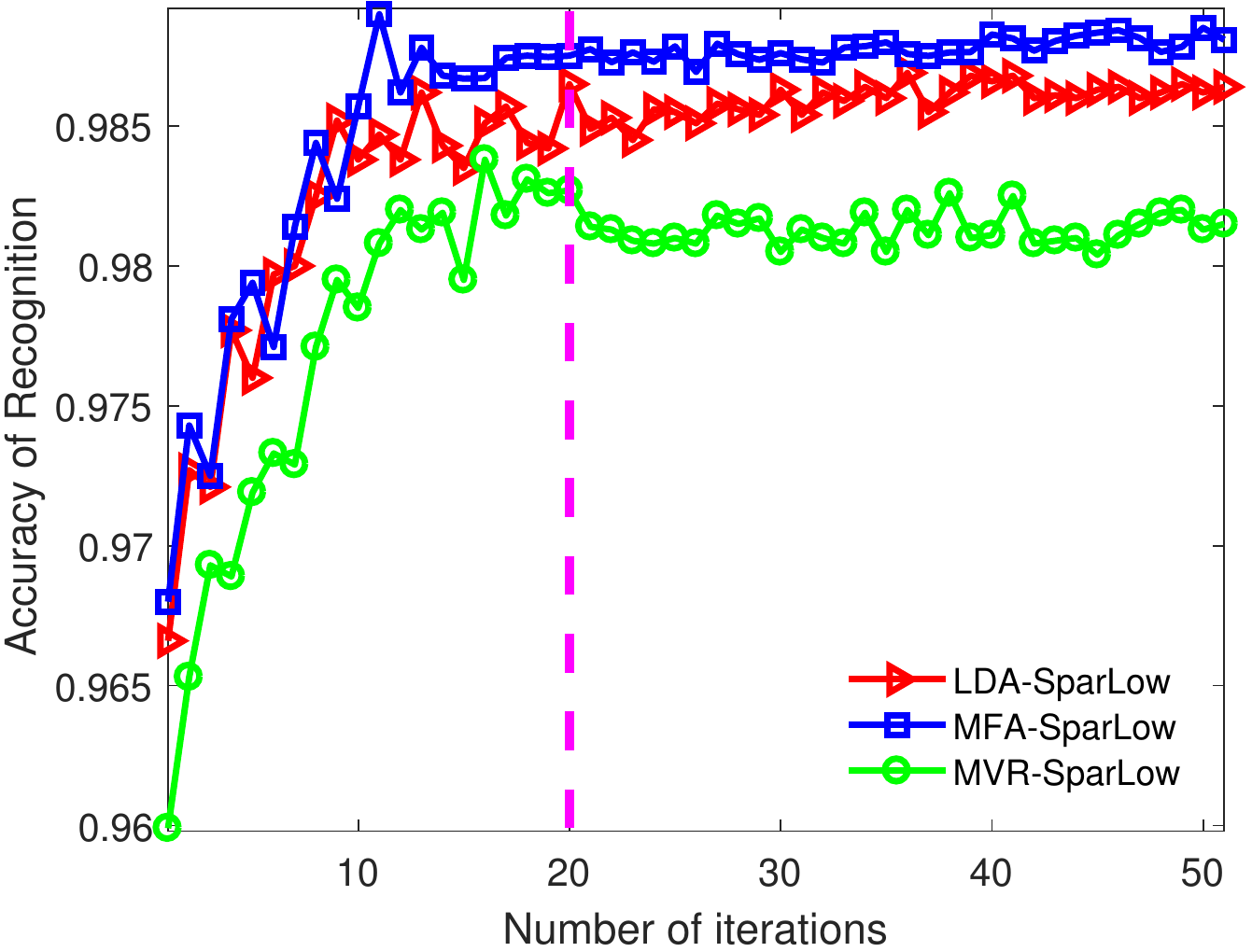}} 
	\subfigure[PIE: without $g_{c}$ and $g_{d}$]{
		\label{fig_pie_results2} 
		\includegraphics[width=1.65in]{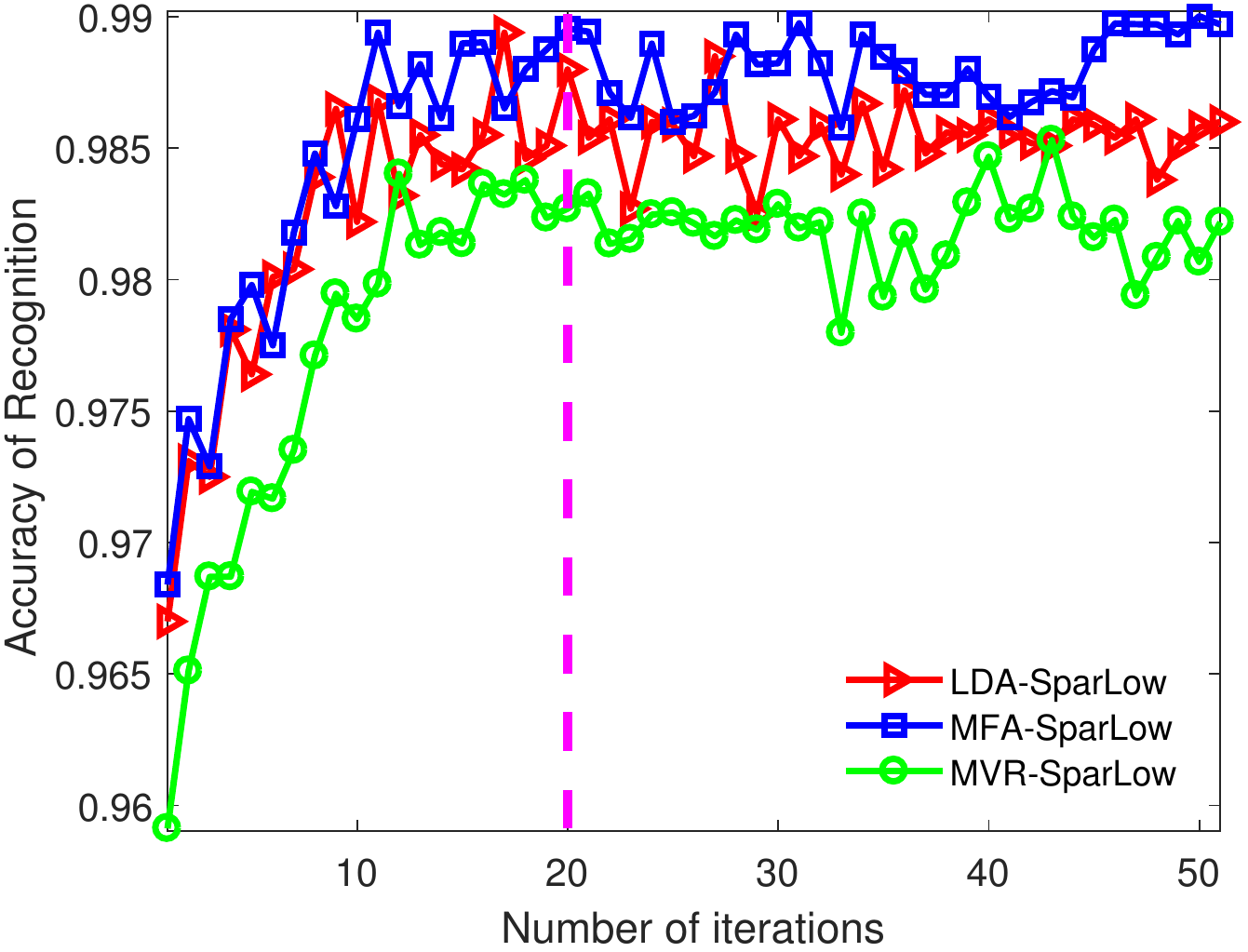}} 
	\hspace{0.001in}
	%
	\subfigure[USPS: with $g_{c}$ and $g_{d}$] 
	{
		\label{fig_USPS_results_supervised1}
		\includegraphics[width=1.65in]{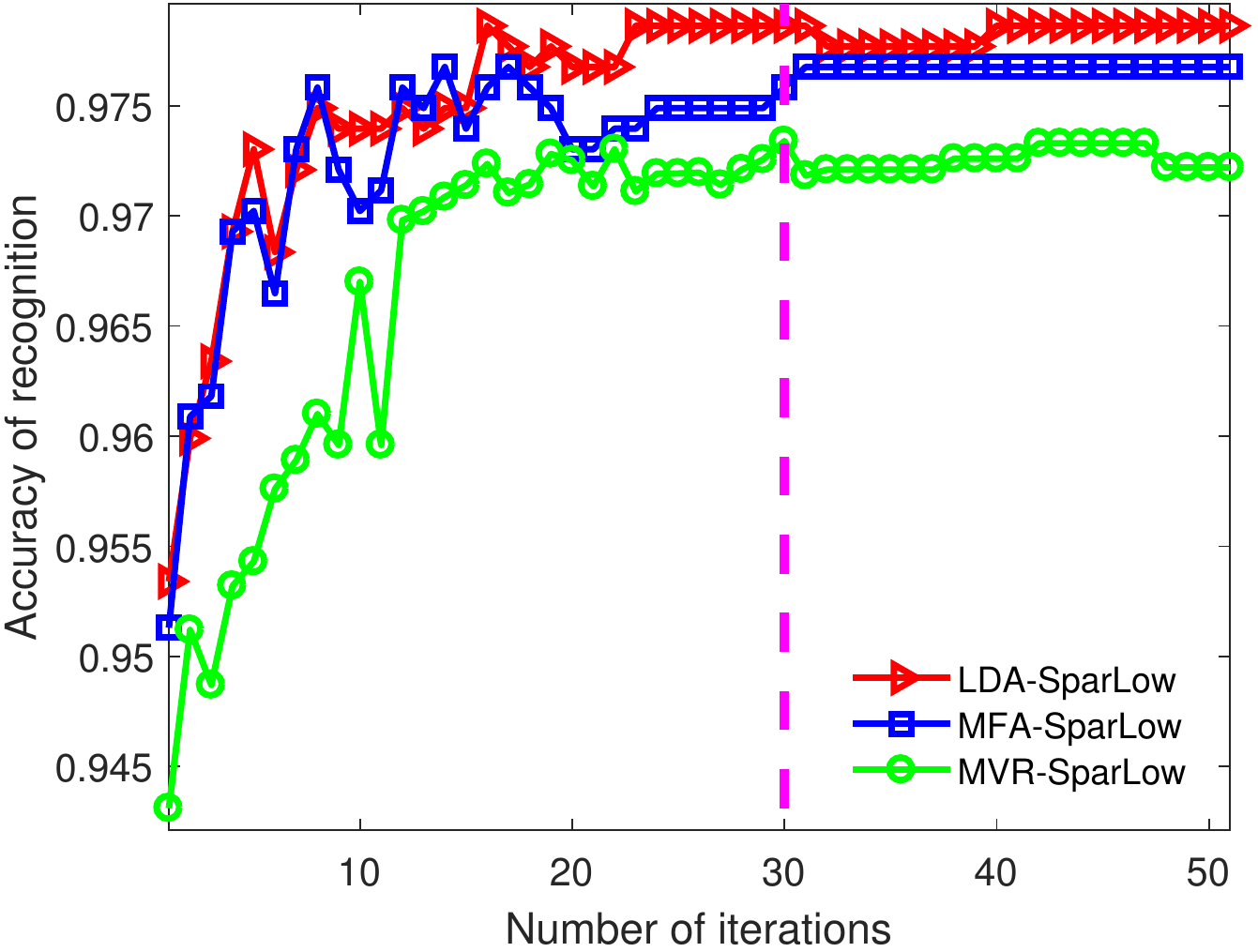}  
	} 
	\hspace{0.001in} 
	\subfigure[USPS: without $g_{c}$ and $g_{d}$]
	{
		\label{fig_USPS_results_supervised2}
		\includegraphics[width=1.65in]{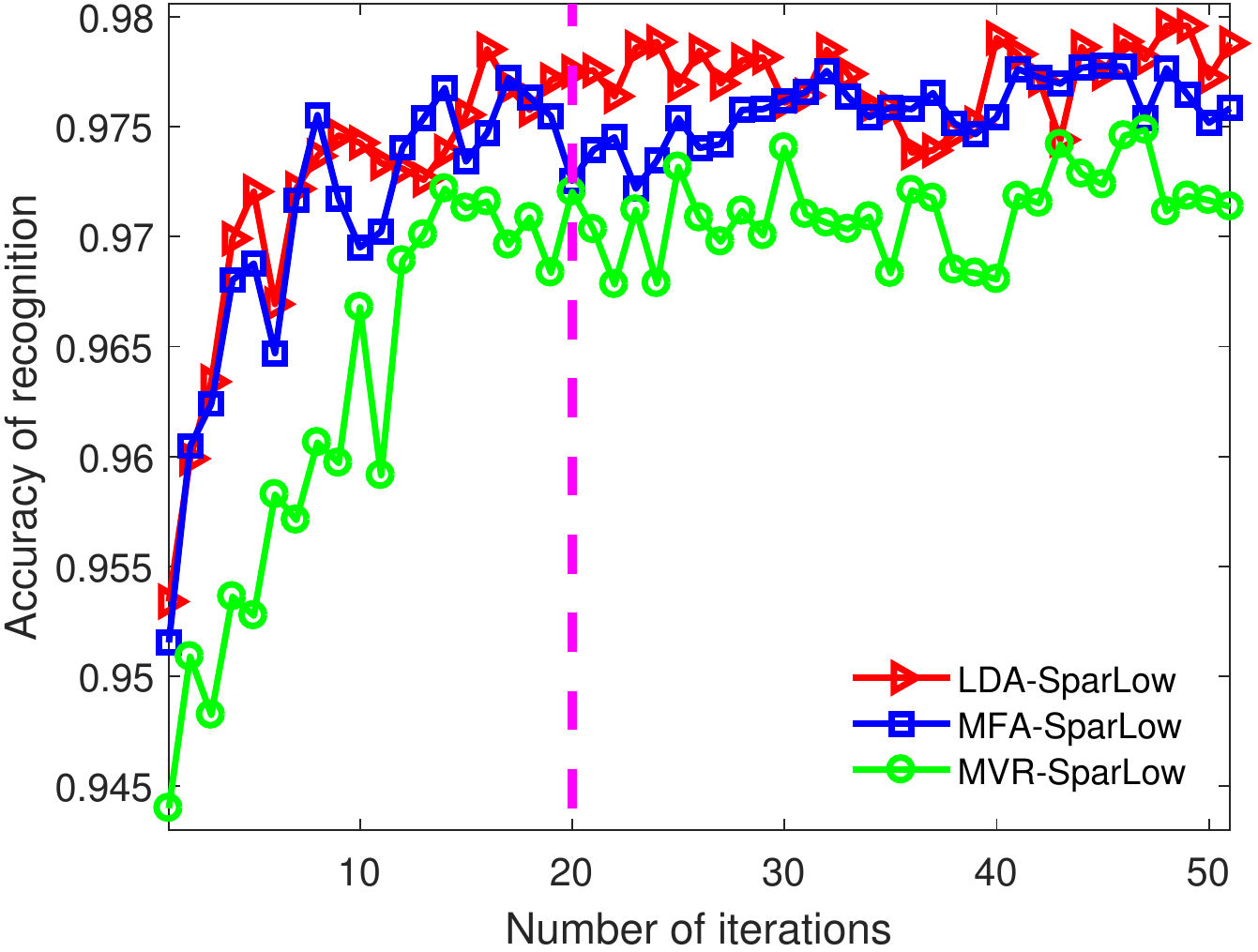} 
	}
	\caption{ Trace of performance over optimization process of supervised \emph{SparLow} with or without regularizers. 
		\label{fig_SupervisedSparLow_convergence} 
	}
\end{figure*}

We further investigate the influence of different weighing factors $\mu_{1}$ and $\mu_{2}$
for both the \emph{PCA-SparLow} and \emph{LDA-SparLow} on the 1NN classification problem.
The experiments are performed with $r = 1000$, and $l = 50$ for \emph{PCA-SparLow}, $l = 9$ for \emph{LDA-SparLow}.
 Fig.~\ref{fig:regu2} and Fig.~\ref{fig:regu3} depict the
 1NN classification results of \emph{PCA-SparLow} and \emph{LDA-SparLow}
 with respect to different weighing factors $\mu_{1}$ and $\mu_{2}$.
It is clear from Fig.~\ref{fig:regu2} and Fig.~\ref{fig:regu3} that for both \emph{PCA-SparLow} and \emph{LDA-SparLow}, 
suitable choices of $\mu_{1}$ and $\mu_{2}$ can improve the performance of the 
\emph{SparLow} system.
%

\subsubsection{Targeted Low Dimensionality}
\label{sec:522}
The main goal of this work is to learn appropriate low dimensional representations of
images.
In this experiment, we investigate the impact of the choice of the targeted low 
dimensionality $l$ to the performance of \emph{SparLow}.
As depicted in Fig.~\ref{fig:redu1}, \emph{SparLow} models of three classic 
unsupervised learning algorithms are examined in terms of recognition accuracy 
with respect to different targeted low dimensionality $l$.
It is clear that, when $l\ge 32$, all three tested unsupervised \emph{SparLow} methods 
perform almost equally well for this specific task.
Similar trends are also observed in the supervised \emph{SparLow}'s as in 
Fig.~\ref{fig:redu2}.
Hence, we conclude that, after bypassing a threshold of the targeted 
low dimensionality, the performance of the corresponding \emph{SparLow} method
is stable and reliable.

\subsubsection{Number of Labelled Samples} 
\label{sec:523}
For semi-supervised learning, one important aspect is surely the number of labelled 
samples.
We compare the \emph{SparLow} counterparts of three state of the art semi-supervised 
learning algorithms, i.e., \emph{SDA-SparLow} versus SDA \cite{caid:iccv07},
\emph{SLap-SparLow} versus SDE \cite{yugx:pr12},
\emph{SMVR-SparLow} versus label propagation methods, e.g., LapRLS \cite{belk:jmlr06}
and LGC \cite{zhou:nips04}.

More specifically, the dictionary $\widehat{\D}$ is initialized by Laplacian 
\emph{SparLow}.
For label propagation methods, e.g., LapRLS and LGC, we use the same settings as 
in \cite{belk:jmlr06, zhou:nips04} for classification.
We choose $\mu_1 = 2.5 \times 10^{-4}, \mu_2 = 5 \times 10^{-3}$, 
$\lambda_1 = 0.2, \lambda_2 = 10^{-3}$ for all tests.
Moreover, we set $\alpha = 0.1$ for \emph{SDA-SparLow} and \emph{SMVR-SparLow}, 
$\alpha_1 = 0.1, \alpha_2 = 0.01$ for \emph{SMFA-SparLow}.
The neighborhood size is set to $20$. 
Our results in Fig.~\ref{fig:semi3} show that the three semi-supervised
\emph{SparLow} methods consistently outperform all other state of the art methods.
It is also worth noticing that with an increasing number of labeled samples,
semi-supervised \emph{SparLow} methods demonstrate greater advantages over
their conventional counterparts. 

\subsection{Evaluation of Disentanglability}
\label{sec:53}
In this subsection, we investigate the disentanglability of the proposed \emph{SparLow}
framework in the setting of unsupervised learning,
which is arguably to be a challenging
scenario to test the ultimate goal of representation learning, i.e., 
to automatically disentangle underlying discriminant factors within the data.
%
%
%
%
	    Hereby, the disentangling factors of variation could be
		discerned consistently across a set of images,   
		such as the class information, various levels of illuminations,
		resolutions, sharpness, camera orientations, the expressions and the poses of faces, etc.
%
%
    In the following, we take the disentangling factors of the class information, 
    the illuminations, the expressions and the poses of faces
      as examples to evaluate the disentanglability of the proposed \emph{SparLow} framework.
%

	Our first experiments are performed on the CMU-PIE face dataset \cite{caid:iccv07},
	which consists of $68$ human subjects with $41,368$ face images in total.
 	In our experiments, we set  
	$\lambda_1 = 10^{-2}, \lambda_2 = 10^{-5},$ $\mu_1 = 2.5 \times 10^{-4}, \mu_2 = 5 \times 10^{-3}$. 
	It is often considered to possess strong class information, compared to other classic
	recognition benchmark datasets.
	We follow the same experimental setting as described in \cite{caid:iccv07} and  
	choose the frontal pose and use all the images under different illuminations, 
	so that we get $64$ images for each subject with the resized scale $32 \times 32$.
	We compare three classic unsupervised learning methods, namely, PCA, ONPP, and OLPP, with
	its sequential sparsity based methods, and our proposed \emph{SparLow} methods.
	Fig.~\ref{fig_pie_compare} reports the performance of the three families of methods,
	applied to $1$NN classification problems on the CMU-PIE dataset.
	It is clear that the \emph{SparLow} methods outperform by far the state of 
	the art algorithms consistently.
\begin{figure}[h]
	\centering
	\includegraphics[width=0.42\textwidth]{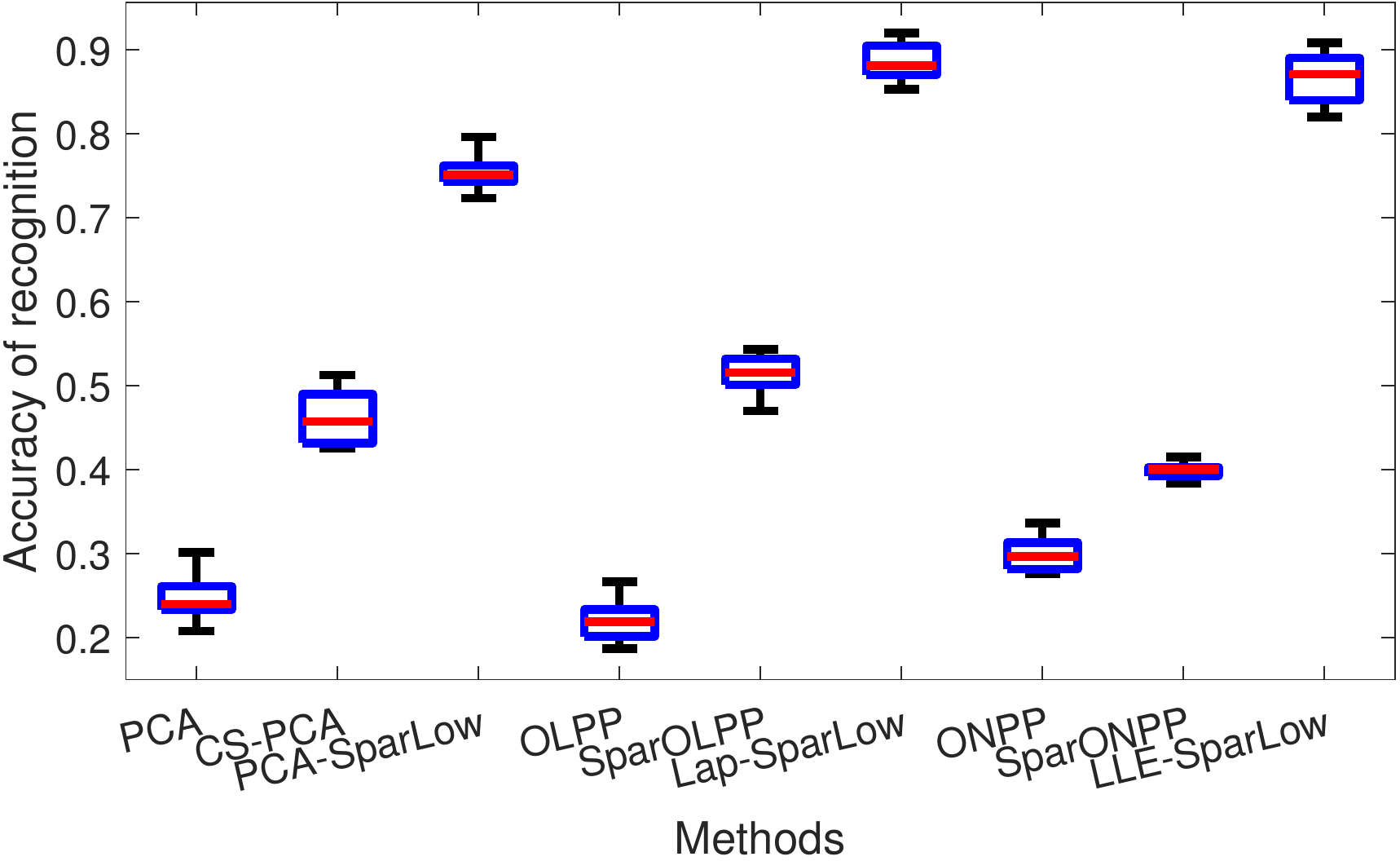} 
	\vspace{-3mm}
	\caption{\label{fig_pie_compare}
		Face recognition on $68$ class PIE faces. The classifier is $1$NN. 
		Randomly choose $8160$ training samples and $3394$ testing samples.
	}
	\vspace*{-3mm}
\end{figure}

	A popular intuitive approach to evaluate disentanglability of representation learning 
	methods is via a visualization of the extracted representations.
	Similar to the concepts of \emph{Fisherfaces} in \cite{belh:pami97}, 
	\emph{eigenfaces} in \cite{belh:pami97}, \emph{laplacianfaces} in \cite{hexf:pami05}, 
	\emph{orthogonal laplacianfaces} in \cite{caid:tip06}, and \emph{orthogonal LLEfaces} 
	in \cite{rowe:science00}, 
%
%
 we construct the $j^{\mathrm{th}}$ \emph{SparLow} 
	facial disentanglement $\bm{\upsilon}_{j}$ as
	\begin{equation}
	\label{eq_data_features}
	\bm{\upsilon}_j = {\D} {\mathbf{u}}_j \in \mathbb{R}^{m},
	\end{equation}
	with $\mathbf{u}_j$ being the $j^{\mathrm{th}}$ column vector of projection matrix ${\U}$.
	Fig.~\ref{fig_pie_3d_facial_features:a} shows the first ten \emph{eigenfaces}, 
	\emph{laplacianfaces}, and \emph{LLEfaces} from top to bottom,
	while Fig.~\ref{fig_pie_3d_facial_features:b} depicts the first ten basis vectors 
	of learned disentangling factors of variation for \textit{PCA-SparLow}, \textit{Lap-SparLow} and 
	\textit{LLE-SparLow}, accordingly. 
 Clearly, our learned facial features in Fig.~\ref{fig_pie_3d_facial_features:b}
	capture  more factors of variation in faces, such as varying poses and expressions (e.g., smile), 
	than the state of the arts in Fig.~\ref{fig_pie_3d_facial_features:a}.
	
	
%
\begin{figure*}
	\centering
	\subfigure[Features extracted from original data]{
		\begin{minipage}[t]{0.455\textwidth}
			\includegraphics[width=\textwidth]
			{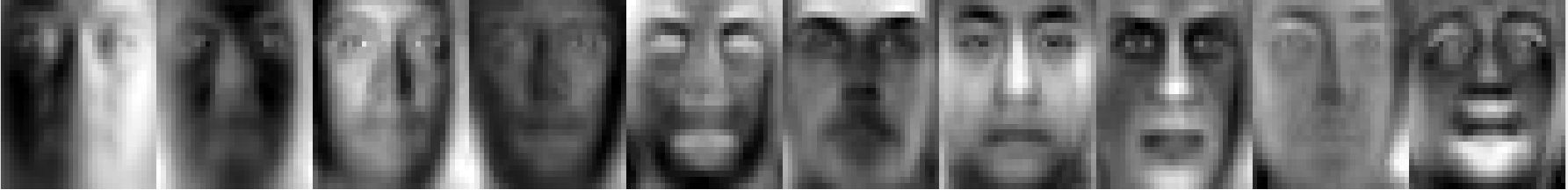}
			\includegraphics[width=\textwidth]
			{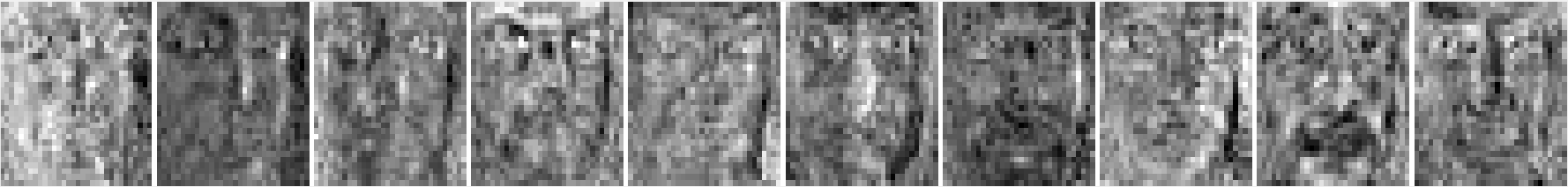}
			\includegraphics[width=\textwidth]
			{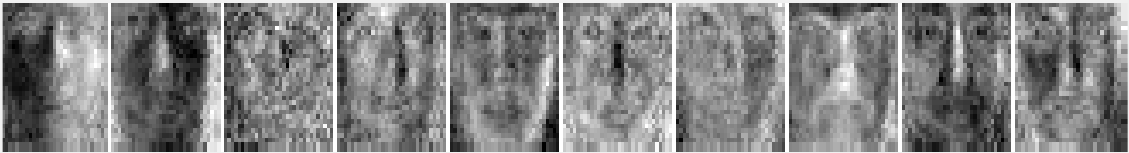}
			\vspace{-1mm}
		\end{minipage}
		\label{fig_pie_3d_facial_features:a}
	}
	%
	\begin{minipage}[t]{0.05\textwidth}
		\centering
		\vspace{-5.5mm}
		(1) \\[5.5mm]		
		(2) \\[5.5mm]		%
		(3) 
	\end{minipage}
	\subfigure[Features extracted from sparse representations]{
		\begin{minipage}[t]{0.455\textwidth}
			\includegraphics[width=\textwidth]
			{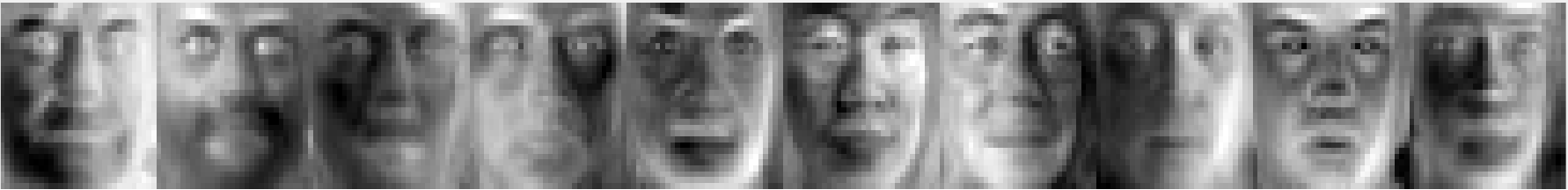}
			\includegraphics[width=\textwidth]
			{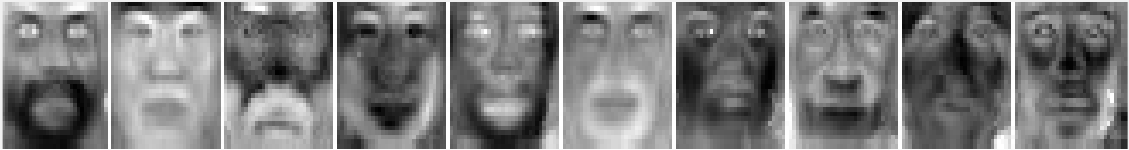}
			\includegraphics[width=\textwidth]
			{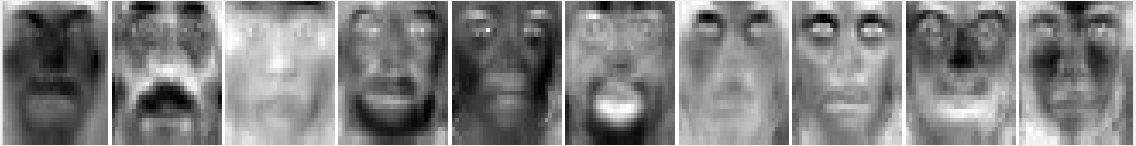}
			\vspace{-1mm}
		\end{minipage}
		\label{fig_pie_3d_facial_features:b}
	}
	\vspace{-3mm}
	\caption{Visualisation of facial features on PIE faces \cite{caid:iccv07}. 
		The presented features are generated via Eq.~\eqref{eq_data_features}.
		From top to bottom: (1) PCA eigenfaces; (2) Laplacianfaces; (3) LLEfaces.
		\label{fig_pie_3d_facial_features}
	}
	\vspace{-3mm}
\end{figure*}
\begin{figure*}[!t]
	\centering
	\subfigure[\emph{OLPP Family}]{
		\begin{minipage}[t]{0.27\textwidth}
			\includegraphics[width=1.04\textwidth]{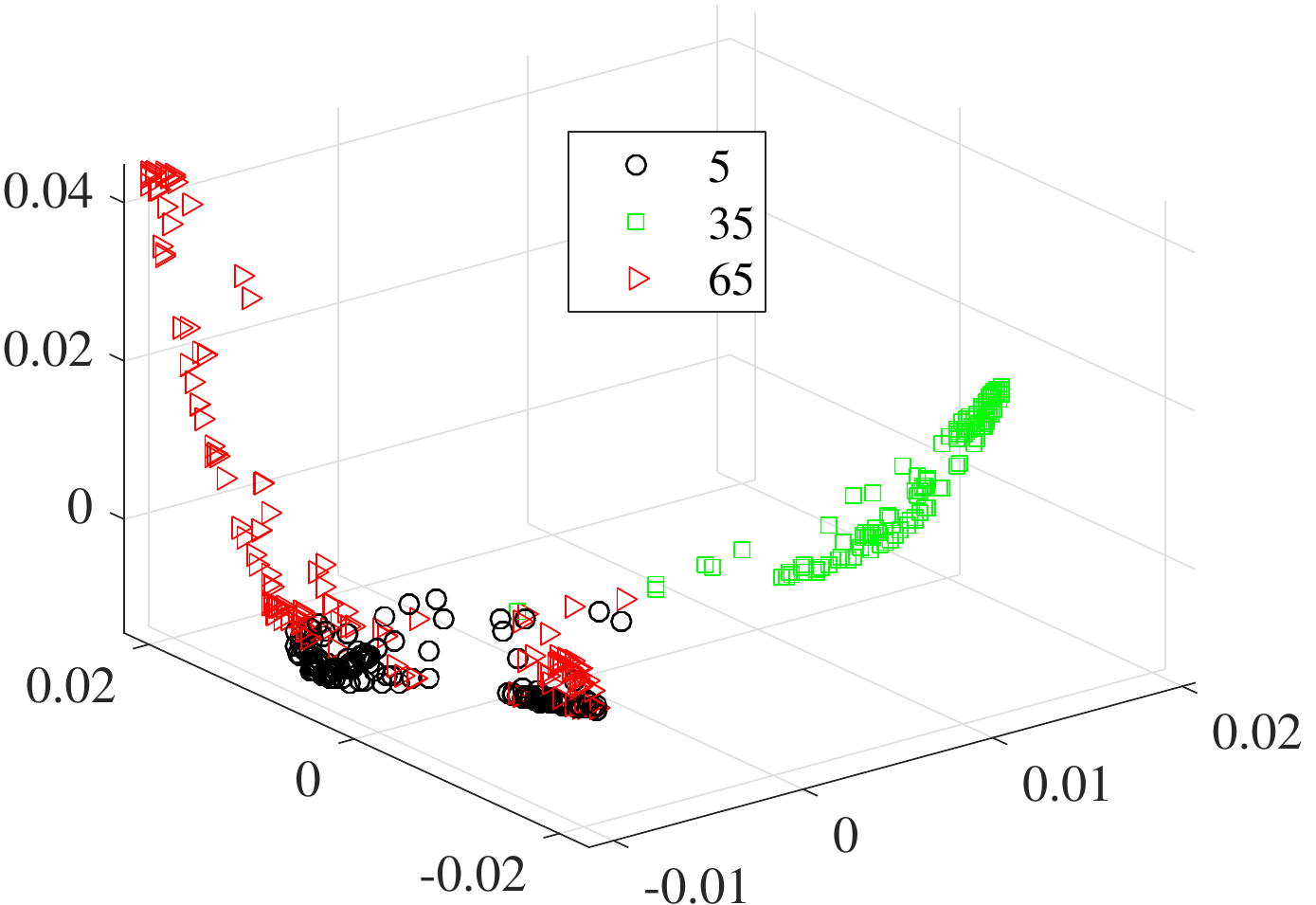}
			\includegraphics[width=1.0\textwidth]{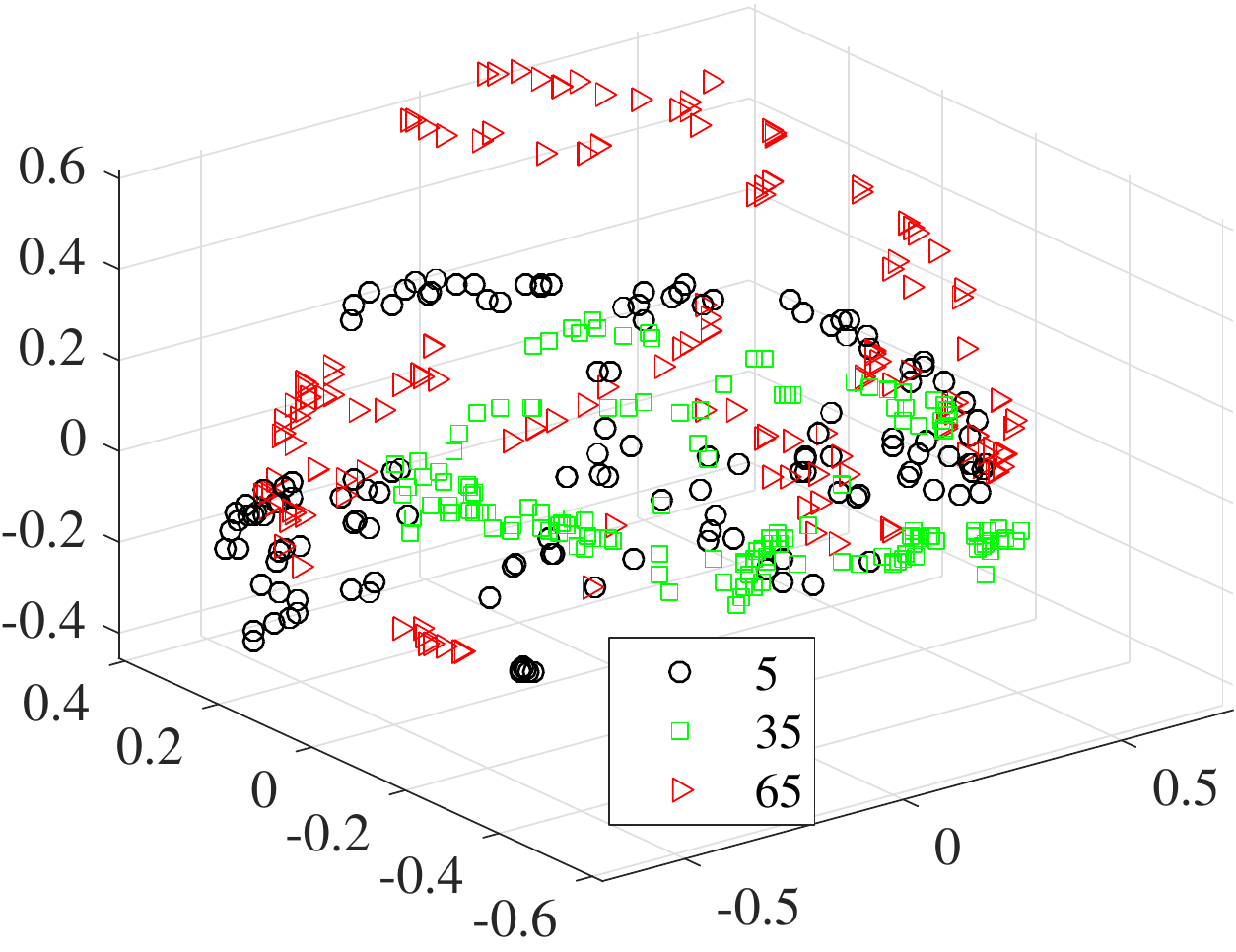}
			\includegraphics[width=1.04\textwidth]{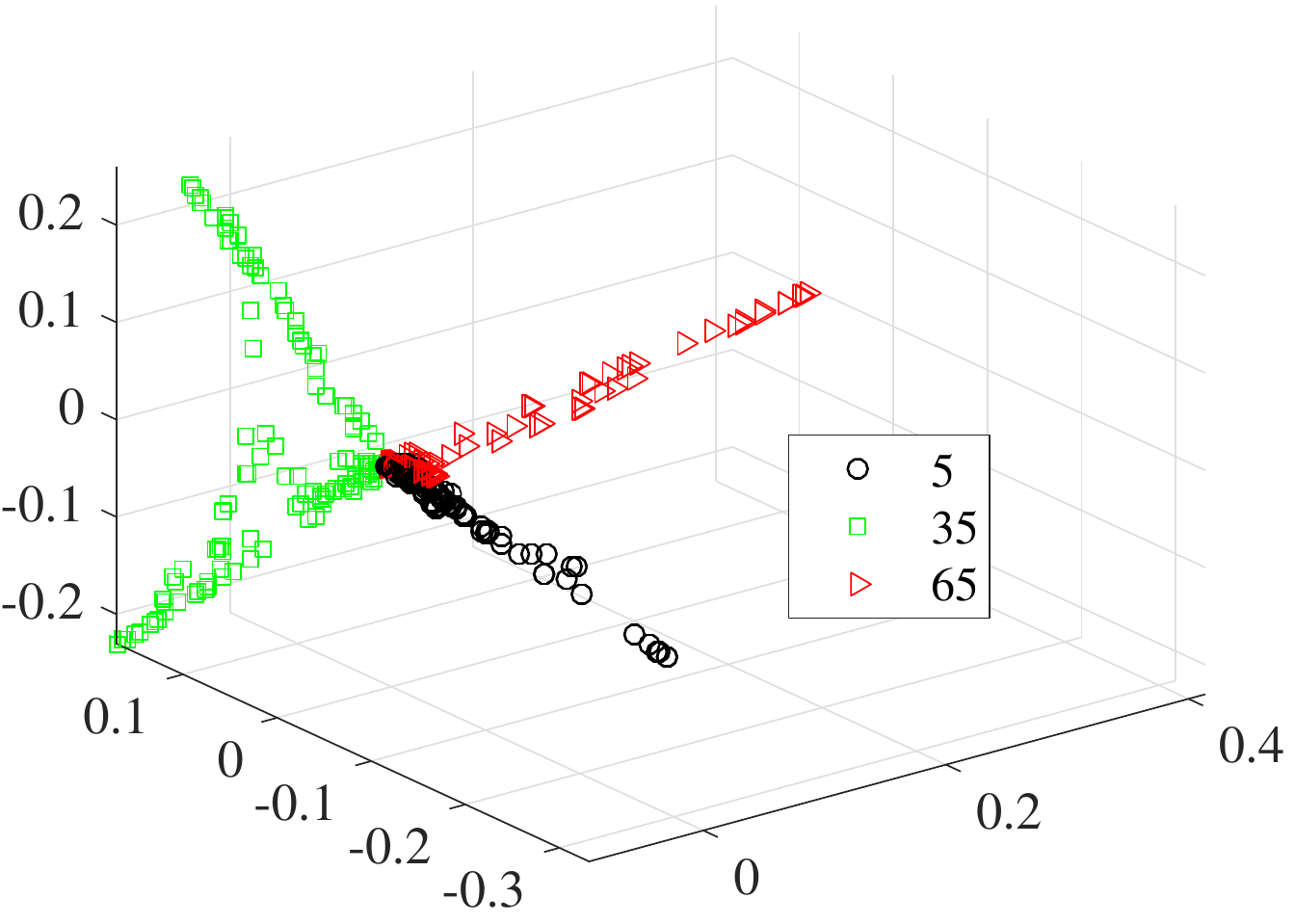}
			\vspace{0.2mm}
		\end{minipage}
	}
	\hspace{5mm}
	\subfigure[\emph{PCA Family}]{
		\begin{minipage}[t]{0.27\textwidth}
			\includegraphics[width=1.05\textwidth]{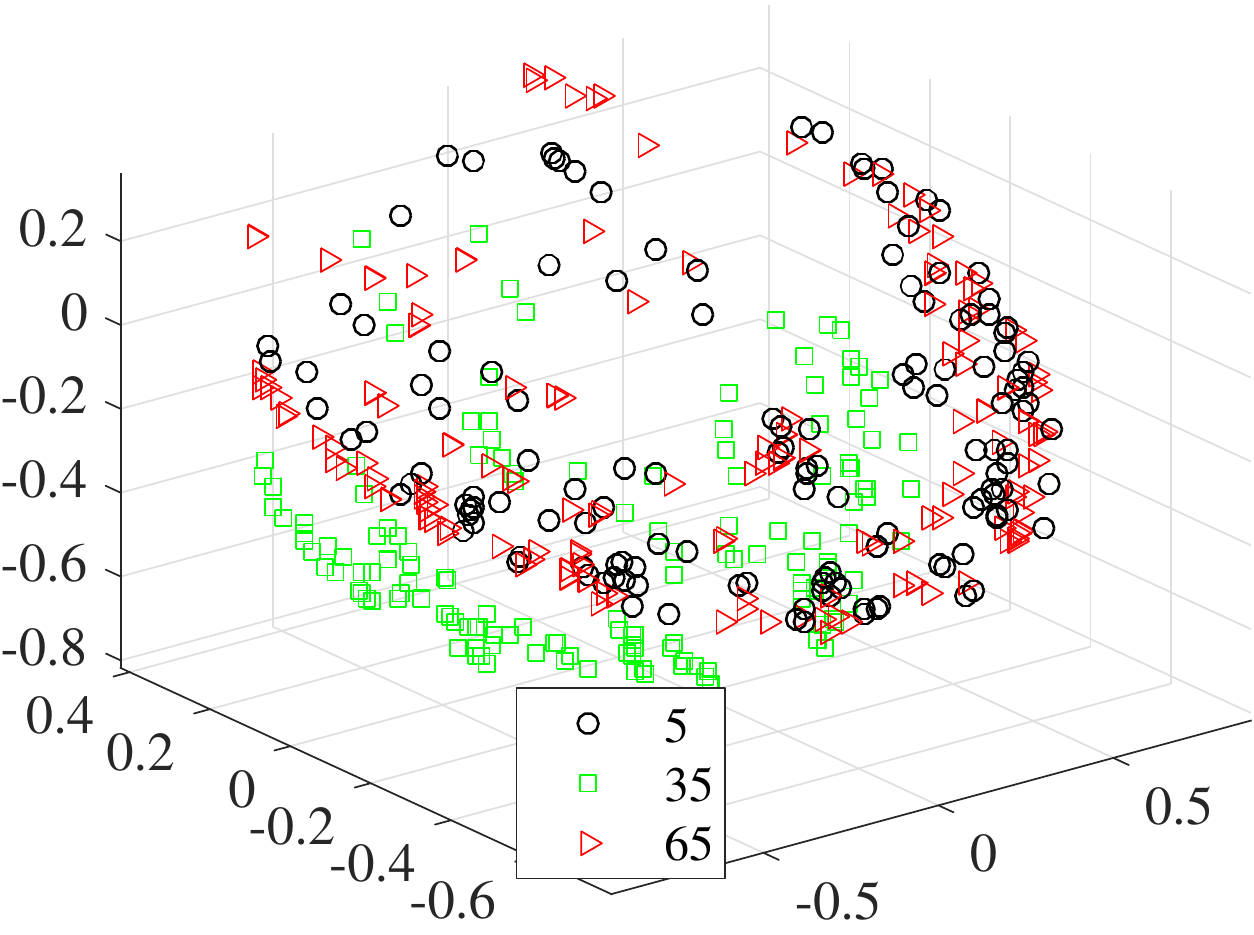}
			\includegraphics[width=1.05\textwidth]{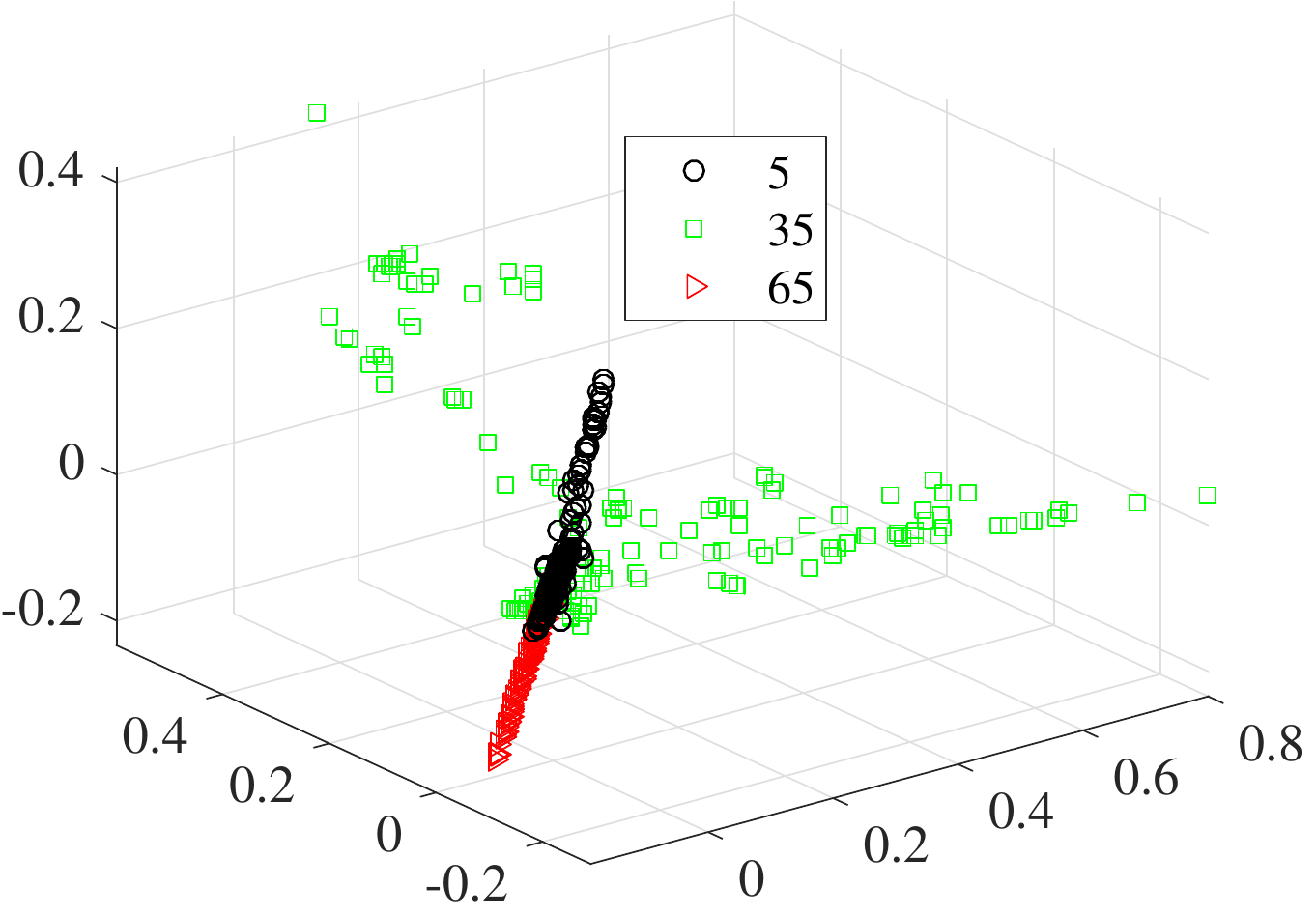}
			\includegraphics[width=1.05\textwidth]{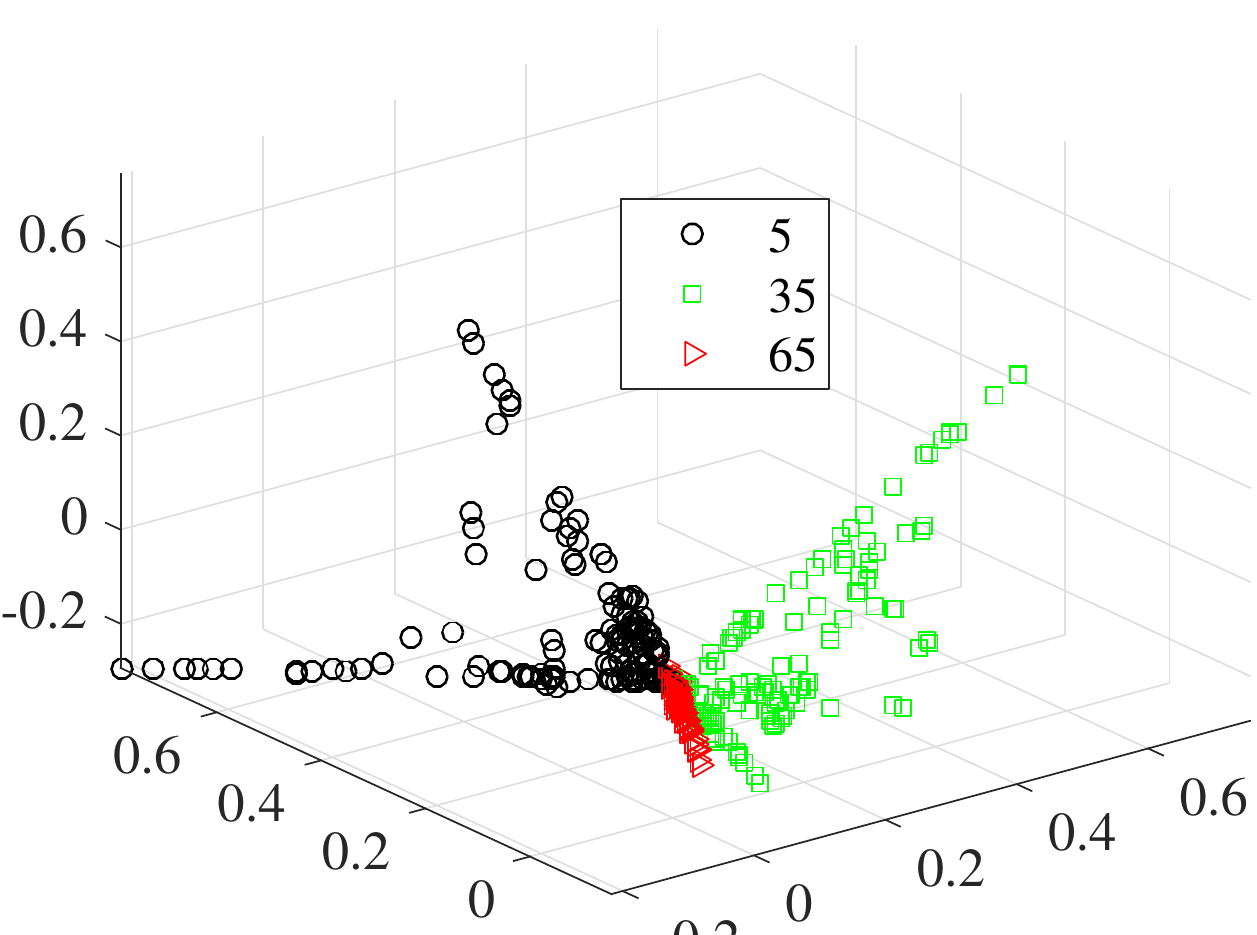}
			\vspace{0.2mm}
		\end{minipage}
	}
	\hspace{5mm}
	\subfigure[\emph{ONPP Family}]{
		\begin{minipage}[t]{0.27\textwidth}
			\includegraphics[width=1.0\textwidth]{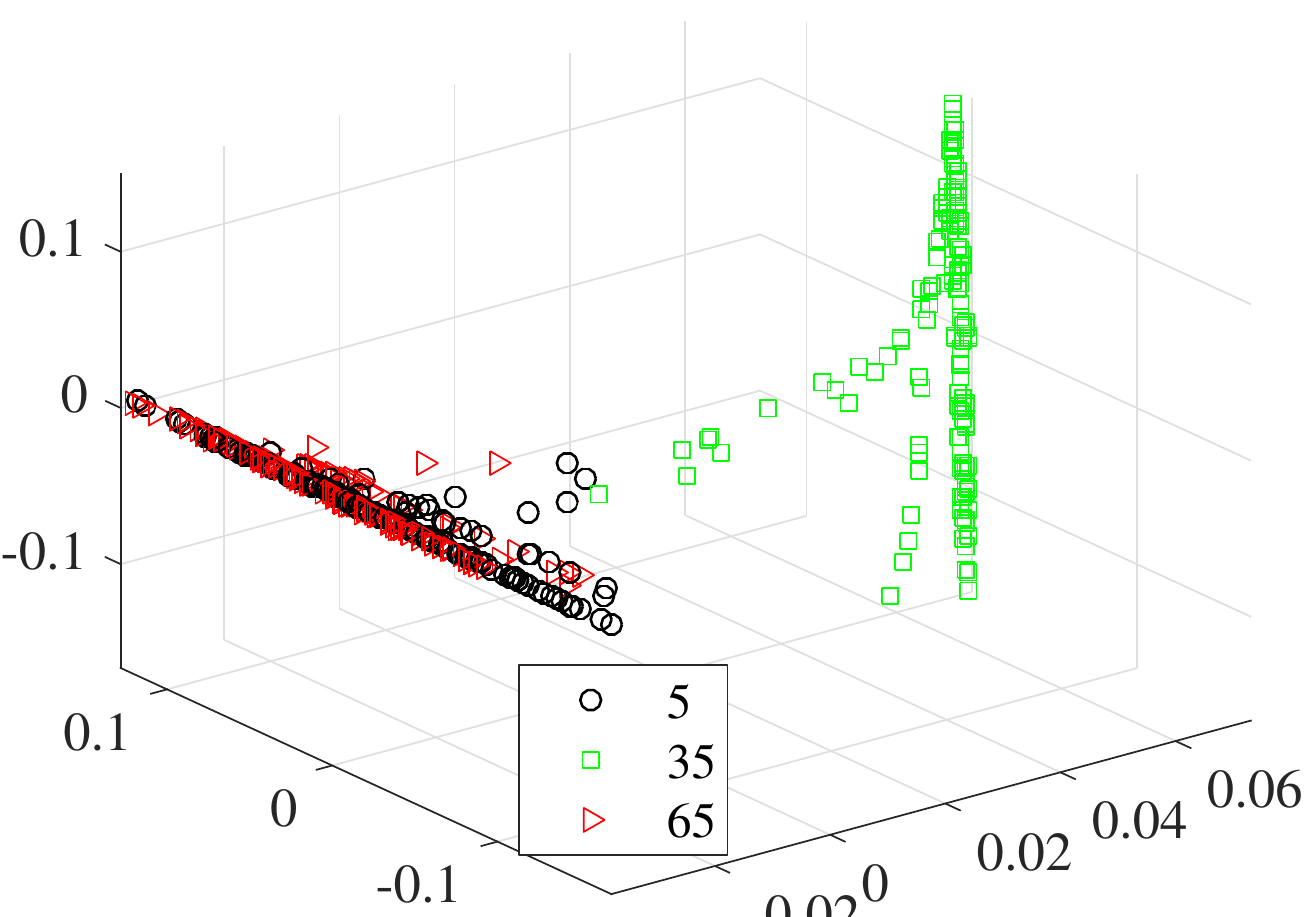}
			\includegraphics[width=1.0\textwidth]{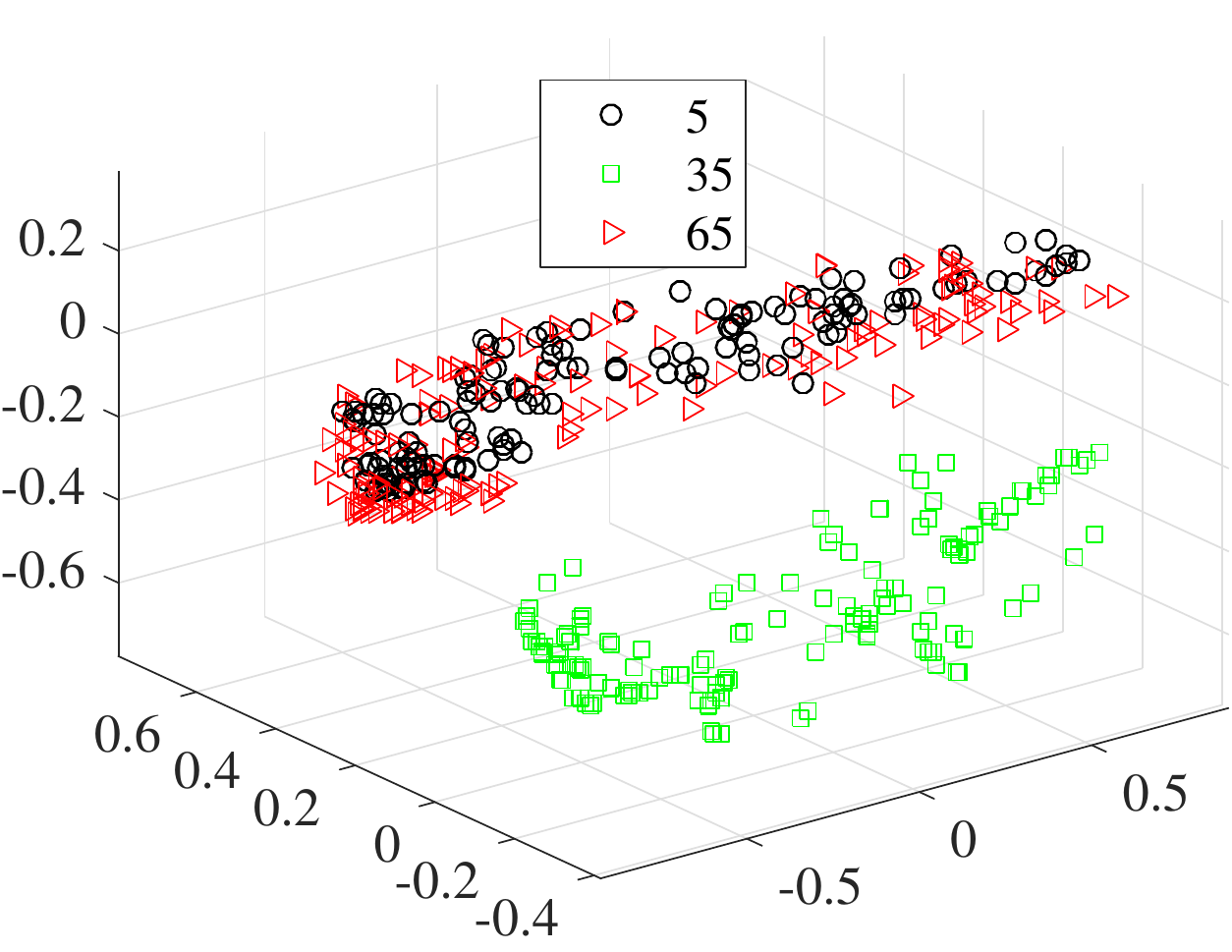}
			\includegraphics[width=1.0\textwidth]{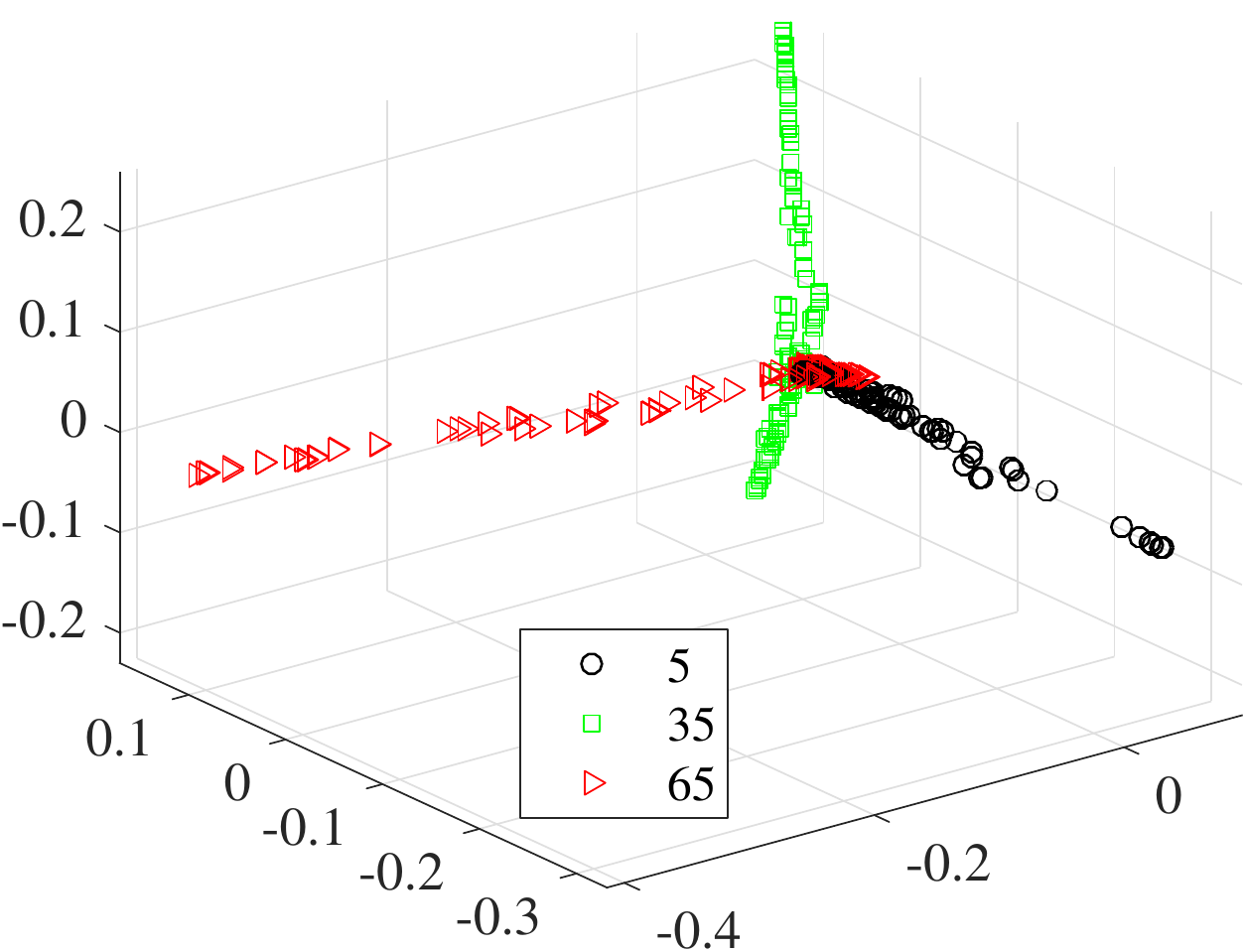}
			\vspace{0.02mm}
		\end{minipage}
	}
	\vspace{-4mm}
	\caption{
		\label{fig_pie_3d_onpp_pca_olpp}
		3D visualisation of PIE faces (class 5, 35, 65).  
		From top to bottom: Applying 
		OLPP/PCA/ONPP in original space, in sparse space with respect to initial 
		dictionary $\widehat\D$, and in sparse space with respect to learned dictionary 
		via \textit{SparLow}, respectively.
	}
	\vspace*{-4mm}
\end{figure*}

Furthermore, we aim to intuitively illustrate the  disentangled factor by 
visualizing the low dimensional image representations, i.e., $3D$.
As suggested in \cite{caid:iccv07}, a subset containing $11,554$ PIE faces with five 
near frontal poses (C05, C07, C09, C27, C29) and different illuminations are chosen, 
thus we nearly get $170$ images for each individual.
Our experiments of 3D visualization were conducted on PIE faces (class 5, 35, 65), 
compared to their classic counterparts.
As depicted in Fig.~\ref{fig_pie_3d_onpp_pca_olpp}, the $3D$ representations 
captured in the original data space, shown in the first row in 
Fig.~\ref{fig_pie_3d_onpp_pca_olpp}, are hardly possible to cluster or group. 
In particular, the boundary between each pair of faces are completely entangled. 
It is evidential that visualization powered by the \emph{SparLow},
i.e., the third row in {Fig.}~\ref{fig_pie_3d_onpp_pca_olpp}, leads to direct 
clustering of the faces.
In short, the class information is clearly disentangled, while the other approaches
fail.

Then we perform 3D visualization on PIE faces without class information. 
We choose $70$ faces from the class $5$ with $2$ factors of variations, i.e., poses and illuminations.
As can be seen from the {Fig.}~\ref{fig_pie_class5_2d_onpp_pca_olpp}, the information referred to poses and illuminations could be
clearly disentangled.
In {Fig.}~\ref{fig_pie_class5_2d_onpp_pca_olpp_a}, from left to right, the illumination become stronger.
From top to bottom, the poses of faces change from left to right.
The similar results are also shown in {Fig.}~\ref{fig_pie_class5_2d_onpp_pca_olpp_b} and {Fig.}~\ref{fig_pie_class5_2d_onpp_pca_olpp_c}.
\begin{figure*}
	\centering
	\subfigure[\emph{PCA-SparLow}.]{
		\label{fig_pie_class5_2d_onpp_pca_olpp_a} 
		\includegraphics[width=2.3in]{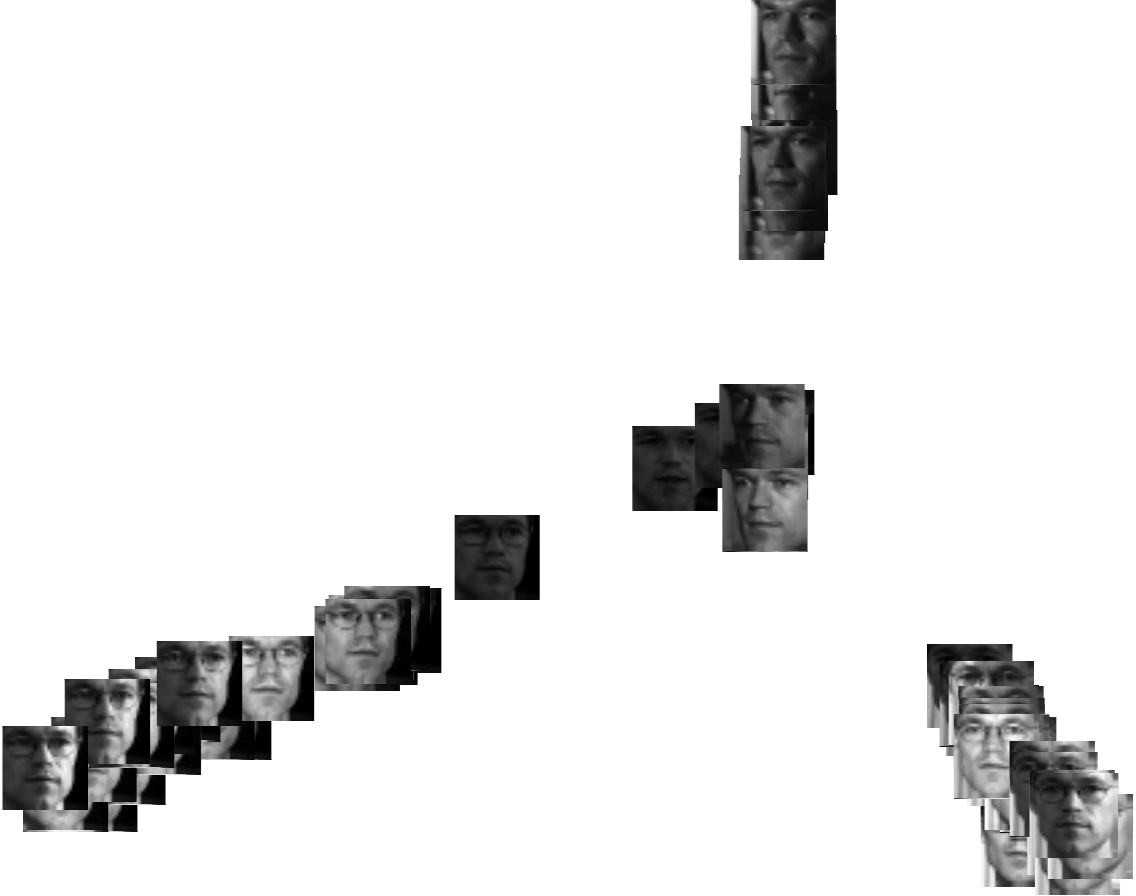}}
	\hspace{0.001in}
	\centering
	\subfigure[\emph{OLPP-SparLow}] {
		\label{fig_pie_class5_2d_onpp_pca_olpp_b} 
		\includegraphics[width=2.3in]{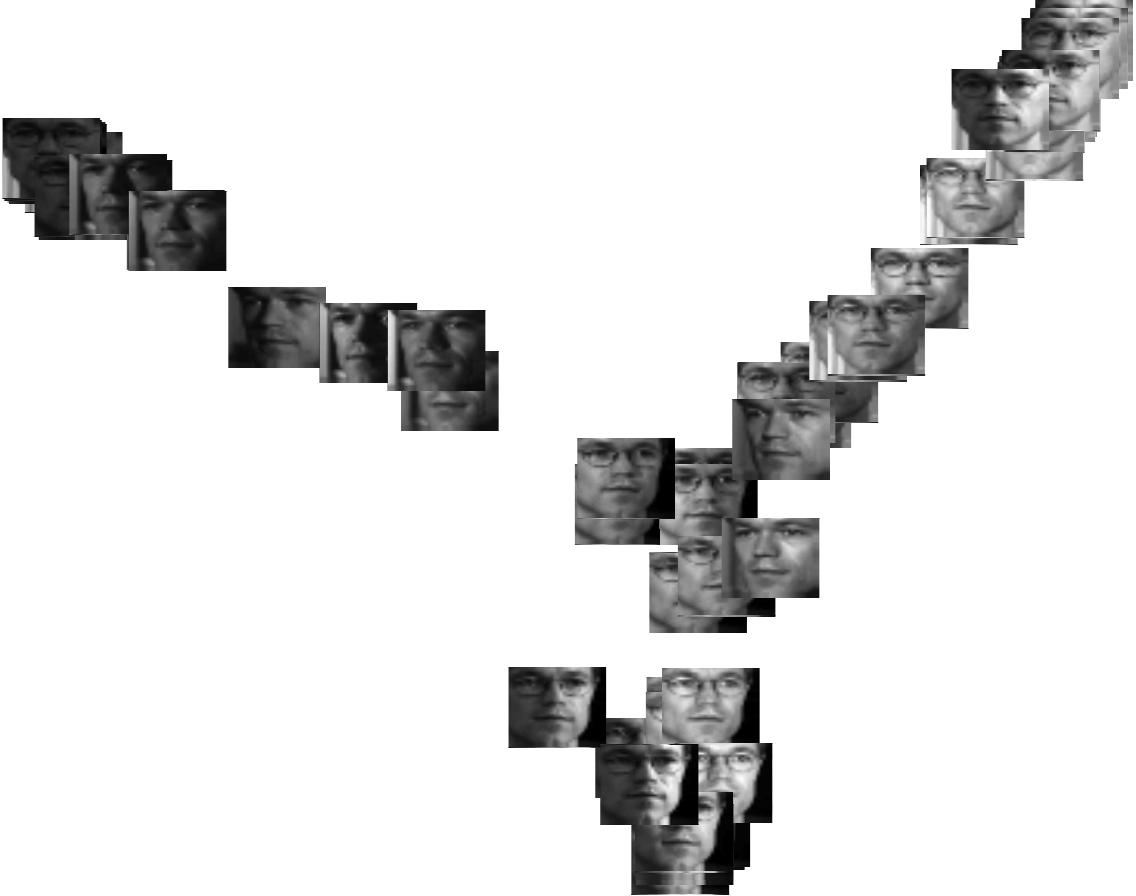}}
	\hspace{0.001in}
	\centering
	\subfigure[\emph{ONPP-SparLow}] {
		\label{fig_pie_class5_2d_onpp_pca_olpp_c} 
		\includegraphics[width=2.3in]{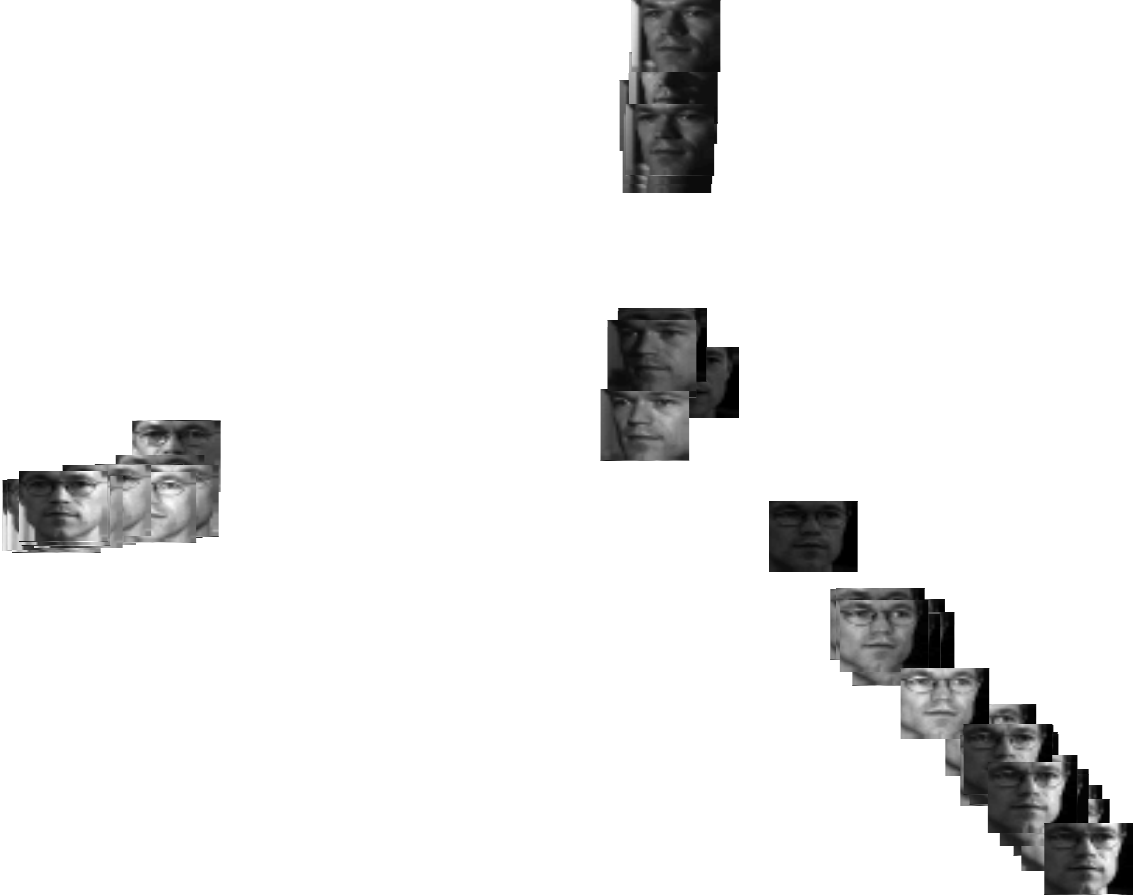}}
	\hspace{0.01in}
	\centering
	\caption{
		2D visualization of PIE faces (class 5).  
		\label{fig_pie_class5_2d_onpp_pca_olpp} 
	}
\end{figure*}

Our last experiments in this subsection are performed on handwritten digits, i.e., 
the MNIST dataset \footnote[1]{\url{http://yann.lecun.com/exdb/mnist/}} and the
USPS dataset.
The MNIST dataset consists of $60000$ handwritten digits images for training 
and $10000$ digits images for testing. 
 The parameters for elastic net are set to be 
	$\lambda_1 = 0.2$, $\lambda_2 = 2\times 10^{-5}$, 
	and $\mu_1 = 5 \times 10^{-3}$, $\mu_2 = 4 \times 10^{-4}$, for both experiments on the MNIST and USPS datasets. 
We compare the \emph{SparLow} methods to several state of the art methods, on 
the task of $1$NN and Gaussian SVM (GSVM) classification.
For PCA, {KPCA} and \textit{PCA-SparLow}, we set $l = 50$, for other methods, 
we set $l = 20$.
For USPS, we use the full training and testing dataset.
For MNIST, we randomly choose $30000$ images for training, and use standard $10000$ testing dataset.
Our results in Table~\ref{table_PR_digital} suggest that the \emph{SparLow} 
methods consistently outperform the state of the arts.
\begin{table}[!t]
	\centering
	\begin{tabu}{|c !{\vrule width 2pt} c|c !{\vrule width 2pt} c|c|}
		\hline  
		{{Methods}}	&   \tabincell{c}{USPS \\ (1NN)} &   \tabincell{c}{USPS \\ (GSVM)}    &  \tabincell{c}{MNIST \\ (1NN)} &   \tabincell{c}{MNIST \\ (GSVM)}  \\ 
		\hline  
		PCA  \cite{gaoj:prl12} & \tabincell{c}{$86.40\% $, \\ $l = 50$} & \tabincell{c}{ $92.43\%$, \\ $l = 50$ }  & \tabincell{c}{$84.62\% $, \\ $l = 50$} & \tabincell{c}{$94.63\% $, \\ $l = 50$} \\ 
	    \hline  
		Sparse PCA  \cite{zouh:cgs06} & \tabincell{c}{$87.66\% $} & \tabincell{c}{ $88.91\%$ }  & \tabincell{c}{$86.10\% $} & \tabincell{c}{$90.88\% $} \\ 
		\hline  
		OLPP \cite{caid:tip06}						& $84.11 \% $		                   	 & $91.48 \% $				 & $83.12\%$  	& $94.76\%$ \\ 
		\hline  
		ONPP \cite{koki:nlaa10}				& $87.39\%$ 	  	                     & $92.73\%$ 		& $85.01\%$  	&  $95.21\%$  \\ 
		\hline  
		KPCA \cite{koki:nlaa10}					&\tabincell{c}{$89.19\%$, \\ $l = 50$}  &\tabincell{c}{$93.27\%$, \\ $l = 50$} 	& $-$ 	& $-$ \\ 
		\hline  
		LLE    \cite{rowe:science00} 					& $68.81\%$                             & $90.43\%$ 			& $66.09\%$ 	&    $93.11\%$ \\ 
		\hline  
		LE 	\cite{koki:nlaa10}						& $71.85\%$ 	                        & $91.93\%$ 		& $68.16\%$   	&   $93.90\%$ \\ 
		\hline  
		ISOMAP   \cite{koki:nlaa10}     			& $64.80\%$ 	& $90.13\%$	   	        & $60.51\%$ 	          &  $91.67\%$ \\ 
		\hline 
		CS-PCA \cite{gaoj:prl12}                    & $87.84\%$ 	      & $94.22\%$				& $87.65\%$  & $96.04\%$\\
		\hline  
		\textit{PCA-SparLow} 	&\tabincell{c}{{$92.18\%$}, \\ $l = 50$}  &\tabincell{c}{{$96.82\%$}, \\ $l = 50$}	   	& \tabincell{c}{{$91.23\%$}, \\ $l = 50$}  	& $97.12\%$ \\ 
		\hline 
		\textit{Lap-SparLow} 							        & {$91.83\%$}                 & {$96.26\%$}  	          &  $89.32 \%$ &   $96.91\%$ \\ 
		\hline 
		\textit{LLE-SparLow} 							        & {$90.78\%$}                 & {$96.16\%$}	             &  $89.10 \%$ &   $96.93\%$ \\ 
		\hline 
	\end{tabu} 
	\vspace{2mm}
	\caption{
		\label{table_PR_digital}
		Classification Performance (Accuracy (\%)) for the MNIST \& USPS datasets 
		of the Proposed \textit{SparLow} methods, with comparisons to some classical 
		unsupervised approaches. \vspace{-8mm}
	}
\end{table}

\subsection{Performance in Large Scale Image Processing}
\label{sec:54}
Finally, we investigate performance of the \emph{SparLow} model in 
large scale image processing applications, specifically, the problem of 
object categorization on large scale image dataset with complex 
backgrounds, such as the images from Caltech-101 \cite{liff:cviu07}, Caltech-256 
\cite{grif:cit07} and 15-Scenes \cite{laze:cvpr06}.
\begin{table*}[htb!]
	\centering
	\caption{\label{table_PR_Caltech} Classification Performance (Average accuracy (\%)) on Caltech-101 \& Caltech-256 datasets.
	}
	\vspace{-2mm}
	\centering
	\begin{tabu}{|c|c|c|c|c|c|c|c !{\vrule width 3pt} c|c|c|c|c|}
		\hline  
		\tabincell{c}{(a) Caltech-101}     & 1   & 5  & 10 & 15 & 20 & 25 &  30   
		&  \tabincell{c}{(b) Caltech-256}  & 15 &  30 & 45 & 60 \\ 
		\hline  
		KSPM \cite{laze:cvpr06} & $-$ & $-$  & $-$ & $56.40$   & $-$ & $-$ & {$64.40$ } 
		& KSPM  \cite{laze:cvpr06}  & $-$ & $34.10$ &$-$  & $-$  \\ 
		ScSPM+SVM \cite{yang:cvpr09}& $-$ & $-$ & $-$ & {$67.0$ }   & $-$ & $-$ &${73.2}$ 
		& ScSPM+SVM \cite{yang:cvpr09}  &$27.73$ & $34.02$  & $37.46$ & $40.14$ \\ 
		LLC+SVM \cite{wang:cvpr10} & $-$ & $51.15$ & $59.77$ & $65.43$    & $67.74$ &  $70.16$ & ${73.44}$ 
		& LLC+SVM \cite{wang:cvpr10}  & $34.36$& $41.19$  & $45.31$ & $47.68$ \\
		Griffin \cite{grif:cit07} & $-$ & $44.2$ & $54.5$ & $59.0$    & $63.3$ &  $65.8$ & $67.60$ 
		& Griffin \cite{grif:cit07}	 & $28.30$ & $34.10$ &  $-$ & $-$ \\
		SRC \cite{wrig:pami09, jian:pami13} & $-$ & $48.8$ & $60.1$ & {$64.9$}    & $67.5$ &  $69.2$ & {$70.7$}
		& SRC \cite{wrig:pami09, jian:pami13} & $27.86$	& $33.33$ &  $-$ &   $-$  \\
		D-K-SVD \cite{zhan:cvpr10, jian:pami13} & $-$ & $49.6$ & $59.5$ & {$65.1$}    & $68.6$ &  $71.1$ & {${73.0}$}
		& D-K-SVD \cite{zhan:cvpr10, jian:pami13} &   $-$ & $33.72$  &  $-$ & $48.88$ \\
	 BMDDL \cite{Zhou:tip17}  & $31.88$ & $56.10$ & $66.01$ & {$69.56$}    & $71.32$ &  $72.28$ & {${75.54}$}
        & BMDDL\cite{Zhou:tip17}  &  $35.41$ & $41.56$  &  $46.90$ & $52.32$ \\
		%
		LC-K-SVD \cite{jian:pami13} & $28.9$ & $54.0$ & $63.1$ & {$67.7$}    & $70.5$ &  $72.3$ & {${73.6}$}
		&  LC-K-SVD \cite{jian:pami13}  		 & $28.9$	& $34.32$  & $-$ &  $-$ \\
		FDDL \cite{yang:ijcv14}	& $-$ & $51.1$ & $60.6$ & $65.6$    & $68.5$   & $70.4$ &  $71.0$ 
		&  FDDL \cite{yang:ijcv14}  & $-$ & $-$ & $-$ & $-$ \\
		SSPIC  \cite{srin:tip15} & $-$ & \tabincell{c}{$55.1$  }  & \tabincell{c}{$62.1$ } & \tabincell{c}{$65.0$  } 
		& $67.5$ &  \tabincell{c}{$68.9$  } & \tabincell{c}{$71.5$  }
		& LSc 	\cite{gaos:pami13}  & $29.99$	& $35.74$ & $38.47$  & $40.32$  \\
		LDA+GSVM  & $-$ & $50.49$ & $60.39$ & {$64.28$}    & $67.10$ &  $71.24$ & {$72.40$}
		& \emph{TDDL} \cite{mair:pami12}  & $33.14$ & $39.05$ & $44.16$ & $49.05$ \\
		\Xhline{2\arrayrulewidth}
		\emph{SparLDA}      & $-$    & $54.60$  &  $65.26$ & $70.05$	& $72.12$  & $73.2$ & $75.82$ 
		& \emph{SparLDA}   & $35.62$      &  $39.53$ & $47.08$ & $51.90$ \\
		\emph{LDA-SparLow}  & $31.23$ & $56.44$  &  $67.12$ & $73.82$	& $74.70$  & $76.20$ & $76.86$ 
		& \emph{LDA-SparLow}   & $38.05$  & {$43.26$} & {$50.32$} & {$55.74$} \\
		\emph{SDA-SparLow}  & $46.12$ & $67.42$  &  $72.01$ & $76.12$	& $76.64$  & $77.40$ & $78.25$ 
		& \emph{SparMFA}       & $36.42$  & $40.89$ & $47.62$ & $51.76$ \\
		\emph{MFA-SparLow}  & $32.43$ & $57.52$  &  $68.44$ & $73.95$	& $75.56$  & $76.63$ & $77.32$
		& \emph{MFA-SparLow}  & $38.82$ & $44.05$ & $51.34$ & $56.82$ \\
		\emph{SMFA-SparLow}  & $46.02$ & $68.66$  &  $72.92$ & $76.02$	& $77.24$  & $77.80$ & $78.42$
		& \emph{SparMVR}      & $32.90$ & $37.91$ & $45.03$ & $50.64$ \\
		\emph{MVR-SparLow}  & $29.30$ & $54.79$  &  $65.63$ & $70.56$	& $73.33$  & $75.41$ & $76.15$
		& \emph{MVR-SparLow}  & $36.29$ & $40.66$ & $47.92$ & $52.56$ \\
		\emph{SMVR-SparLow}  & $44.82$ & $66.43$  &  $70.80$ & $75.48$	& $76.32$  & $77.14$ & $77.76$ 
		&    & $-$ & $-$ & $-$ & $-$ \\
		\hline
	\end{tabu} 
\end{table*}

We adopt a popular approach of object categorization to firstly detect certain local image 
features, such as dense SIFT or dense DHOG (a fast SIFT implementation) \cite{lowe:ijcv04},  then to quantize them into discrete ``visual words'' over a codebook, and finally
to compute a fixed-length Spatial Pyramid Pooling (SPP) vector of acquired 
``visual words'' \cite{yang:cvpr09, wang:cvpr10, laze:cvpr06}.
We refer to such an approach as the \emph{SIFT/DHOG-SPP} representation.
In our experiments, the local descriptor is extracted from $s \times s$ pixel patches 
densely sampled from each image, specifically, we choose 
$s = 16$ for SIFT and $s = 16, 25, 31$ for DHOG. 
The dimension of each SIFT/DHOG descriptor is $128$.
A codebook with the size of $k = 1024$ or $k = 2048$, is learned for coding 
SIFT/DHOG descriptors.
We then divide the image into $ 4\times 4$, $3\times 3$ and $1 \times 1$ 
subregions, i.e., $21$ bins. 
The spatial pooling procedure for each spatial sub-region is applied via the max pooling function associated with an ``$\ell_2$ normalisation'', 
e.g.,~\cite{yang:cvpr09, wang:cvpr10, jian:pami13, yang:cvpr10}.
The final SPP representations are computed with the size $m = 21504$ or $m= 43008$, and 
hence are reduced into a low-dimensional PCA-projected subspace. 
In what follows, we denote by $m, m_{\mathrm{PCA}}, r, l$ the dimension of SPP 
representation, PCA projected subspace, sparse codes, 
and learned low dimensional representation, respectively.
%

\subsubsection{Number of Reduced Features}
\label{sec:541}
It is known that the computational complexity of sparse coding mainly 
depends on the choice of dictionary size \cite{weix:icpr16}.
It is hence necessary to investigate the impact of the number of features
to the performance of \emph{SparLow}.
We employ a popular approach to firstly apply a classic PCA transformation on the
SPP features, and then learn a dictionary on reduced features~\cite{wrig:pami09}.
In this experiment, we deploy the Caltech-101 dataset \cite{liff:cviu07}, which
contains $9144$ images from $102$ classes.
Most images are {in} medium resolution (about $300\times 300$ pixels). 
Fig.~\ref{fig_101_lda_mfa_mlr_convergence} show the recognition results of several
supervised \emph{SparLow} and semi-supervised \emph{SparLow} methods with PCA projected
SPP features.
It is obvious that after reaching a certain number of reduced features, i.e., $l > 1024$,
all methods in test show no significant improvement.
Moreover, it is worth knowing that the performance of two semi-supervised \emph{SparLow} 
methods consistently outperform other methods. 
\begin{figure}[!t]
	\centering
	\includegraphics[width=0.4\textwidth]	{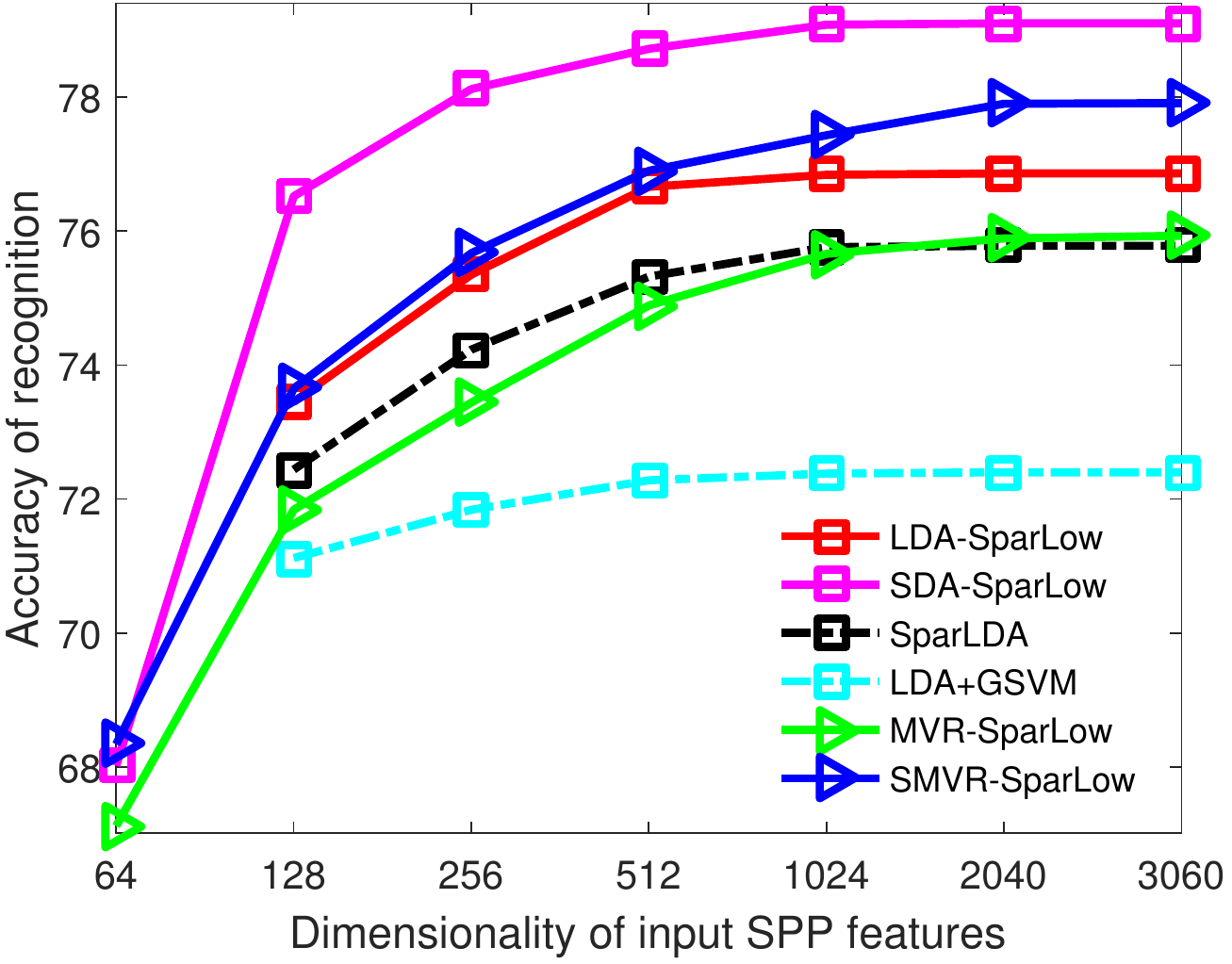}
	\vspace{-2mm}
	\caption{
		\label{fig_101_lda_mfa_mlr_convergence}
		Recognition results using proposed \emph{MFA-SparLow} in PCA projected 
		subspace on Caltech-101 dataset. $n_{\mathrm{train}} = 30$, $r = 3060$.
	}
	\vspace{-3mm}
\end{figure}



\subsubsection{Caltech-101 dataset}
\label{sec:542}
Learning on Caltech-101 dataset is often considered to be hard, 
since the number of images per category varies significantly from $31$ to $800$. 
In our experiments, we set $\lambda_1 = 5\times 10^{-2}, \lambda_2 = 10^{-5}$, 
$k = 1024$, $r = 1020$, and choose $l \in \{101, 287, 512\}$ are set for LDA, MFA, MVR 
related methods, respectively.
We confine ourselves to the same experimental settings as used in 
\cite{yang:cvpr09, wang:cvpr10, laze:cvpr06}.
We randomly select $1$, $5$, $10$, $15$, $20$ and $30$ labeled images per category for training and the rest images for testing. 
For semi-supervised \textit{SparLow}, the training set includes all labeled and unlabelled images.
Table \ref{table_PR_Caltech} (a) gives a comparison of \emph{LDA-SparLow}, \emph{SLDA-SparLow} with approaches from the literatures.
Note that, for the number of labeled samples $n_{l} = 1$, LDA and \emph{LDA-SparLow} 
are not applicable.
It shows that our proposed approaches consistently outperform all the competing 
approaches. 
Especially, the semi-supervised \emph{SparLow} could significantly improve the 
recognition accuracy when the labeled training samples are limit.
The possible reason is that some categories have large samples, e.g., the category 
\textit{airplanes} has $800$ samples, the $n_{l} \leq 30$ is too much limit for 
training such a category.
Semi-supervised \emph{SparLow}'s take advantage of all the data for training sparsifying dictionary,
which is the key factor to promote the discrimination of sparse representations.

\vspace*{-2mm}
\subsubsection{Caltech-256 dataset}
\label{sec:543}
The Caltech-256 dataset consists of $30607$ images from $256$ 
categories with various resolutions from $113 \times 150$ to $960 \times 1280$. 
Each category has at least $80$ images. 
Unlike the Caltech-$101$ dataset, this dataset contains multiple objects in various 
poses at different locations within the images.
Existence of background clutter and occlusion result in higher intraclass diversity, 
which makes the categorization task even harder.

We apply our \emph{DHOG-SPP} \textit{SparLow} on randomly selected $15$, $30$, $45$, $60$ training images per category, respectively.
We set $\lambda = 0.1$, $m_{\mathrm{PCA}} = 2560$, $k = 2048$, $r = 3128$.
For LDA-like \textit{SparLow}, MFA-like \textit{SparLow} and MVR-like \textit{SparLow}, we set 
$l = 255$, $361$ and $387$, respectively. 
Finally, we use GSVM for classifying the low dimensional representations.
In this experiment, we use the sparse coding formulation associated with KL-divergence.
Table \ref{table_PR_Caltech} (b) shows that our results outperform the state of the art 
methods under all the cases.
Moreover, we also implemented \emph{TDDL} for comparison, and the each sub-dictionary size of \emph{TDDL} 
is fixed as $r = 200$.
It shows that \emph{TDDL}'s perform worse than the \emph{SparLow}'s. 
The possible reason is that \emph{TDDL}'s associated with a binary classifier may suffer the huge number of classes.

\subsubsection{15-Scenes dataset}
\label{sec:544}
We finally evaluate the \emph{SparLow} framework on the 15-Scenes dataset \cite{laze:cvpr06}. This dataset contains totally $4485$ images falling into $15$ categories, 
with the number of images {in } each category ranging from $200$ to $400$ and image size around $300 \times 250$ pixels. 
The image content is diverse, containing not only indoor scenes, such as bedroom, kitchen, but also outdoor scenes, such as building and country views, etc.

Following the common experimental settings, we use \emph{SIFT-SPP} as input with $k = 1024$ and $m_{\mathrm{PCA}} = 2000$.
For MFA, we set $k_1 = 70, k_2 = 100$, and $l = 50$. For \emph{SparMFA} and \emph{MFA-SparLow}, we set $k_1 = 30, k_2 = 100$ and $l = 60$.
For all supervised and semi-supervised \emph{SparLow} methods, the dictionary size $r = 750$.
Table \ref{table_scene_15} compares our results with several sparse coding methods in 
\cite{yang:cvpr09, wrig:pami09, wang:cvpr10, gaos:pami13, zhan:cvpr10, jian:pami13, Zhou:tip17, lobe:pami15}, GSVM,
and the method in \cite{laze:cvpr06}, which are all using SPP features as input data.
As shown in Table \ref{table_scene_15}, our approaches significantly outperform all state of the art approaches.
Note that, the bottom three lines are all semi-supervised methods.
\begin{table}[htb!]
	\centering
	\centering
	\begin{tabu}{|c|c !{\vrule width 2pt} c|c |}
		\hline  
		Methods     						  			& Accuracy  		    &  Methods  							        &  Accuracy  \\
		\hline  
    BMDDL \cite{Zhou:tip17}  & $96.9$    			&   Lobel \cite{lobe:pami15}                         	   	&  $86.3\pm 0.5$  \\
		GSVM   	                               & $72.2$              & LDA        									& {$91.69$ }  \\
		KSPM \cite{laze:cvpr06} 			&  $83.50$ 				& \emph{SparLDA}                                		& {$95.89$ } \\
		ScSPM+GSVM \cite{yang:cvpr09} 			& {$80.28$ } 			& \emph{LDA-SparLow}                                	& {$97.47$ } \\
		LLC+GSVM \cite{wang:cvpr10} 			& $89.2$    			& MFA            								& {$92.82$ } \\
		SRC \cite{wrig:pami09, jian:pami13} 	& $91.8$				& \emph{SparMFA}                           			& {$96.65$ } \\

		LSc 	\cite{gaos:pami13}    		&  $89.7$ 				& \emph{MFA-SparLow}                           		& {$98.46$ } \\
		K-SVD\cite{ahar:tsp06} + LDA            & $92.6$ 				& MVR ($l = 512$)                           	& {$93.10$ } \\
		D-K-SVD \cite{zhan:cvpr10}     & $89.01$        		& \emph{SparMVR}                           		    & {$96.32$ } \\
		LC-K-SVD \cite{jian:pami13}                  & $92.9$    			& \emph{MVR-SparLow}                           		& {$97.55$ } \\
	    FDDL \cite{yang:ijcv14} 	& $90.2$ & $-$ & $-$ \\
  		%
		\hline
		\Xhline{2\arrayrulewidth}
		SDA \cite{caid:iccv07}     				& $97.28$    			& \emph{SDA-SparLow}                           		& {$99.18$ } \\
		SDE \cite{yugx:pr12}        				& $97.66$    			& \emph{SLap-SparLow}                           		& {$99.25$ } \\
		LapRLS \cite{belk:jmlr06}         & $94.86$    			& \emph{SMVR-SparLow}                           		& {$99.12$ } \\
		\hline
	\end{tabu} 
	\vspace{2mm}
	\caption{\label{table_scene_15}
		Averaged classification Rate (\%) comparison on 15-Scenes dataset. 
		The classifier is $1$NN for the third column if not specified.
		\vspace{-8mm}
	}
\end{table}

\vspace*{-1mm}
\section{Conclusion}
\label{sec:07}
In this work, we present a low dimensional representation learning approach, 
coined as \emph{SparLow}, which leverages both sparse representation and the trace 
quotient criterion.
It can be considered as a two-layer disentangling mechanism, which applies the trace quotient
criterion on the sparse representations.
Our proposed generic cost function is defined on a sparsifying dictionary and
an orthogonal transformation, which form a product Riemannian manifold.
A geometric CG algorithm is developed for optimizing the \emph{SparLow} function.
Our experimental results depict that in comparison with the state of the art unsupervised, supervised and semi-supervised 
representation learnings methods, our proposed \emph{SparLow} framework delivers promising 
performance in data visualization and classification. 
Moreover, the proposed \emph{SparLow} is flexible and can be extended to more general cases of 
low dimensional representation learning models with orthogonal constraints.

\vspace*{-10mm}
\begin{IEEEbiography}[{\includegraphics[width=1in,height=1.25in,clip,keepaspectratio]{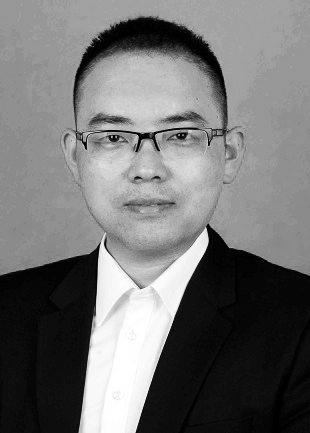}}] 
	{Xian Wei} (S'14-M'18) received the Ph.D. degree in Engineering from the Technical University of Munich, Munich, Germany, in 2017.
In July 2017, he joined Fujian Institute of Research on the 
Structure of Matter, Chinese Academy of Sciences, China,
as a leading Researcher of Machine Vision and Pattern Recognition Lab.
His research interests focus on sparse coding, deep learning and geometric optimization. 
The applications include robotic vision, videos or images modeling, synthesis, recognition and semantics. 
\end{IEEEbiography}

\vspace*{-10mm}
\begin{IEEEbiography}[{\includegraphics[width=1in,height=1.25in,clip,keepaspectratio]{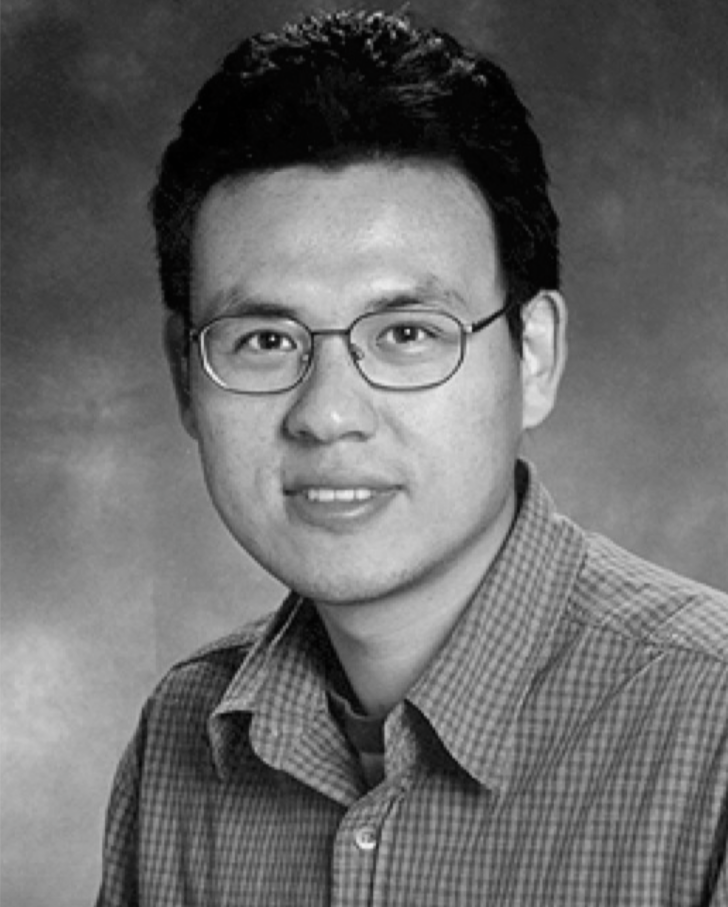}}]{Hao Shen}(S'04-M'08)
received his PhD in Engineering from the Australian National University, Australia, in 2008.
From December 2008 to September 2017, he was a post-doctoral researcher at the 
Institute for Data Processing, Technische Universit{\"a}t M{\"u}nchen, Germany.
In October 2017, he joined fortiss, 
State Research institute of Bavaria, Germany,
as the leader of Machine Learning Lab.
His research interests focus on machine learning for signal processing, 
e.g., deep representation learning, and reinforcement learning.
\end{IEEEbiography}

\vspace*{-10mm}
\begin{IEEEbiography}[{\includegraphics[width=1in,height=1.25in,clip,keepaspectratio]{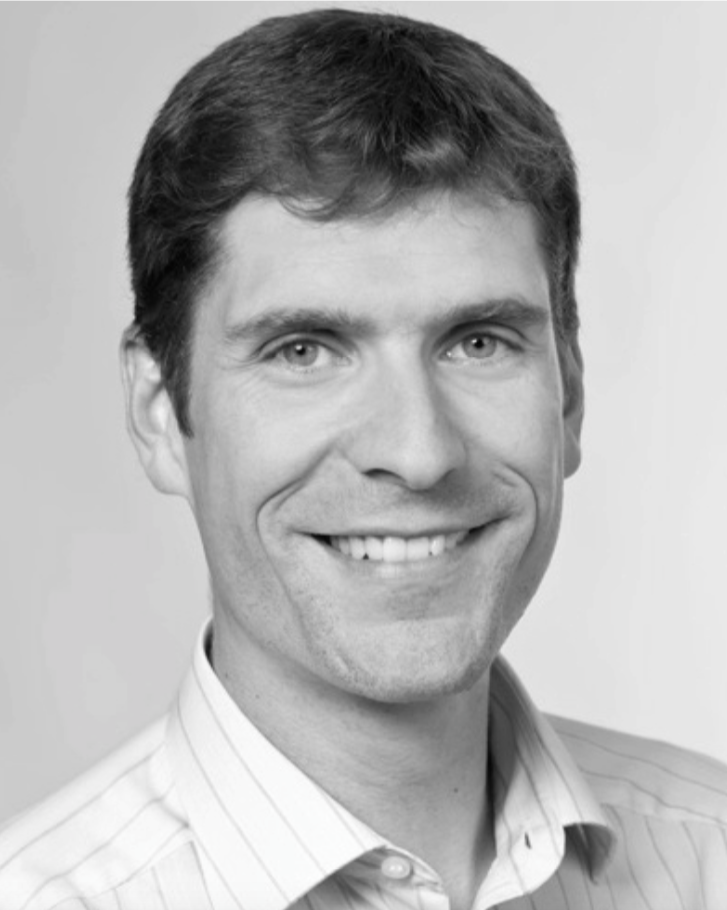}}]{Martin Kleinsteuber}
received his Ph.D. in Mathematics from the University of W\"urzburg, Germany, in 2006. After post-doc positions at National ICT Australia Ltd., the Australian National University, Canberra, Australia, and the University of W\"urzburg, he has been appointed assistant professor for geometric optimization and machine learning at the Department of Electrical and Computer Engineering, TU M\"unchen, Germany, in 2009. 
Since 2016, he is leading the Data Science Group at Mercateo AG, Munich. 
\end{IEEEbiography}

%
%
%
%



\appendix
\section{Proof of Proposition~1}
\label{app:01}
Note, that the numbering of equations in Appendix continues from 
the numeration in the manuscript.
The proof of Proposition~1 requires the following two lemmas.
\setcounter{equation}{62}

\begin{lemma}
	\label{lem:01}
	Let $g \colon \mathbb{R}^{r} \to \mathbb{R}$ satisfy
	Assumption~1.
	Then, for a given $\x \in \mathbb{R}^{m}$ and a dictionary 
	$\D \in \mathbb{R}^{m \times r}$, a vector $\phib^{*} := 
	[\varphi_{1}^{*}, \ldots, \varphi_{r}^{*}]^{\top} \in \mathbb{R}^{r}$ is
	the unique solution to the sparse regression problem as in 
	Eq.~(5), if and only if, the following conditions
	hold
	\begin{equation}
	\left\{\!\!
	\begin{array}{ll}
	\d_{i}^{\top} (\x - \D \phib^{*}) = \nabla g_{i}(\varphi_{i}^{*}),
	& \quad~\text{for}~\varphi_{i} \neq 0, \\[0.5mm]
	\d_{i}^{\top} (\x - \D \phib^{*}) \in \partial g_{i}(0),
	& \quad~\text{for}~\varphi_{i} = 0.
	\end{array}
	\right.
	\end{equation}
	Here $\partial g_{i}(0) = [-b_{1},b_{2}]$ with $b_{1},b_{2} > 0$ is 
	the subgradient of element-wise regulariser $g_{i}$ at $0$.
\end{lemma}
\begin{proof}
	Since the cost function Eq.~(5) is strictly convex as a result of $g$
	being strictly convex, the
	unique solution $\phib^{*}$ satisfies the subdifferential optimality condition, i.e.,
	\begin{equation}
	0 \in \partial f_{s}(\phib^{*}),
	\end{equation}
	where $\partial f_{s}(\phib^{*})$ is the subdifferential of $f_{s}$ at $\phib^{*}$.
	By the construction that the cost function is separable, the 
	subdifferential of the cost function is computed as
	\begin{equation}
	\d_{i}^{\top} (\x - \D \phib^{*}) \in \partial g_{i}(\varphi_{i}).
	\end{equation}
	By the fact that $\partial g_{i}(\varphi_{i}) = \nabla_{\!g_{i}\!}(\varphi_{i})$ for
	all $\varphi_{i}^{*} \neq 0$, the result follows.
\end{proof}

\begin{lemma}
	\label{lem:02}
	Let $g \colon \mathbb{R}^{r} \to \mathbb{R}$ satisfy
	Assumption~1, and Assumption~2 hold true.
	Then, for a given $\x \in \mathbb{R}^{m}$ and a dictionary 
	$\D \in \mathbb{R}^{m \times r}$, the unique solution $\phib^{*} := 
	[\varphi_{1}^{*}, \ldots, \varphi_{r}^{*}]^{\top} \in \mathbb{R}^{r}$ 
	to the sparse regression problem as in 
	Eq.~(5) is smooth in an open 
	neighborhood around $\D^{*}$.
\end{lemma}
\begin{proof}
	Recall the definition of support $\mathfrak{P}(\x,\D^{*})$ as in Eq.~(42),
	we define its complement set as 
	\begin{equation}
	\mathfrak{Q}(\x,\D^{*}) := \{i\in\{1,\ldots,r\} |
	\varphi^{*}_{i} = 0\}.
	\end{equation}
	Let $k = |\mathfrak{P}(\x,\D)|$, $\phib_{\L} \in \mathbb{R}^{k}$ 
	and $\D_{\L} \in \mathbb{R}^{m \times k}$ being the subset of 
	$\D \in \mathbb{R}^{m \times r}$, in which 
	the indices of columns fall into the support $\L$.
	We construct a function as
	\begin{equation}
	\begin{split}
	\eta& \colon \mathbb{R}^{m \times r} \times \mathbb{R}^{r} \to \mathbb{R}^{r}, 
	\\[-1mm]
	\eta&(\D,\phib) :=  \D_{[\L,\J]}^{\top} (\x - \D_{[\L,\J]} \phib) - 
	\begin{bmatrix} 
	\nabla g_{\L}(\phib_{\L}) \\ \phib_{\J} 
	\end{bmatrix}, \\[-1mm]
	\end{split}
	\end{equation}
	with $\phib = [\phib_{\L}^{\top}~\phib_{\J}^{\top}]^{\top}$,
	$\phib_{\L} \in \mathbb{R}^{k}$, and
	$\phib_{\J} \in \mathbb{R}^{r-k}$.
	The function $\eta$ is smooth in both $\D$ and $\phib$, and
	$\eta(\D^{*},\phib^{*}) = 0$.
	Taking the directional derivative of $\eta$ with respect to $\D$ and $\phib$ leads to
	\begin{equation}
	\begin{split}
	& \Dr \eta(\D, \phib)(\H_{\D},\mathbf{h}_{\phib}) \\[-1mm]
	= & \underbrace{
		\begin{bmatrix}
		\chi_{1}(\D_{[\L,\J]},\phib)
		&\!\!
		\chi_{2}(\D_{[\L,\J]},\phib)
		\end{bmatrix}
	}_{=: J_{\eta}(\D, \phib) \in \mathbb{R}^{r \times (r+m\cdot r)}}
	\cdot
	\begin{bmatrix}
	\mathbf{h}_{\phi} \\
	\operatorname{vec}(\mathbf{H}_{\D})
	\end{bmatrix}, \\[-1mm]
	\end{split}
	\end{equation}
	where \vspace{-1mm}
	\begin{equation}
	\chi_{1}(\D_{[\L,\J]},\phib) \!:=\! 
	\begin{bmatrix}
	\mathsf{H} g_{\L}(\phib_{\L}) \!&\! \mathbf{0}_{k,r-k} \\
	\mathbf{0}_{r-k,k} \!&\! \I_{r-k}
	\end{bmatrix}
	\!+\! \D_{[\L,\J]}^{\top} \D_{[\L,\J]},
	\end{equation}
	and $\chi_{2}(\D_{[\L,\J]},\phib) \in \mathbb{R}^{r \times (m \cdot r)}$ is a 
	tedious term without clear knowledge on its rank.
	Here, $\mathsf{H} g_{\L}(\phib_{\L}) \in \mathbb{R}^{k \times k}$ is the Hessian 
	matrix of the separable sparsifying function restricted on the support $\L$, i.e., 
	$g_{\L} := \{g_{i}\} \in \mathbb{R}^{k}$ with $i \in \L$.
	The matrix $J_{\eta}(\D, \phib)$ is known as the \emph{Jacobian matrix} of $\eta$.
	Our aim is to ensure the full rankness of the Jacobian matrix to apply 
	the implicit function theorem \cite{aman:book08}.
	Often, it is assumed that $m \le r$.
	Let $\kappa(\D)$ be the \emph{spark} of $\D$, i.e., the smallest number of 
	columns from $\D$ that are linearly dependent \cite{elad:book10}.
	If $k < \kappa(\D)$, then the Jacobian matrix $J_{\eta}(\D, \phib)$ is of full rank.
	Unfortunately, the spark is often very difficult to control during an optimization 
	procedure.
	Thus, a simple but general fix to make $J_{\eta}(\D, \phib)$ have full rank 
	is to have Hessian $\mathsf{H} g_{\L}(\phib_{\L})$ non-degenerate.
	In other words, by Assumption~2 and the implicit function theorem, 
	there exist two open neighborhood containing $\D^{*}$ and $\phib^{*}$, i.e.,
	$\D^{*} \in \mathfrak{U}$ and $\phib^{*} \in \mathfrak{W}$, and
	a unique continuously differentiable function $\widehat{\phib} \colon 
	\mathfrak{U} \to \mathfrak{W}$, so that $\widehat{\phib}(\D^{*}) = \phib^{*}$ and 
	$\eta(\D,\widehat{\phib}(\D)) = 0$ for all $(\D,\phib) \in \mathfrak{U}
	\times \mathfrak{W}$.
	
	Let us denote $\widehat{\phib}(\D) := [\widehat{\varphi}_{1}(\D), \ldots, 
	\widehat{\varphi}_{r}(\D)]^{\top} \in \mathbb{R}^{r}$.
	Since $\widehat{\phib}$ is continuously differentiable in $\D$ and 
	$\widehat{\phib}(\D^{*}) = \phib^{*}$, there exists an open subset 
	$\mathfrak{U}_{\epsilon} \subset \mathfrak{U}$, so that the following
	holds true with a gap $\epsilon > 0$ and $\epsilon < \min \{b_{1}, b_{2}\}$
	\vspace{-1mm}
	\begin{equation}
	\left\{\!\!
	\begin{array}{ll}
	|\widehat{\varphi}_{i}(\D) - \varphi_{i}^{*}| < \epsilon, & 
	\text{~for~}i \in \mathfrak{P};
	\\[0.5mm]
	|\widehat{\varphi}_{i}(\D)| < \epsilon, & \text{~for~}i \notin \mathfrak{P}.
	\end{array}
	\right.
	\end{equation}
	We then construct the following projection
	\begin{equation}
	\Pi_{\L} \colon \mathbb{R}^{r} \to \mathbb{R}^{r},
	\qquad
	\phib \mapsto \widetilde{\phib} := [\widetilde{\varphi}_{1}, \ldots, 
	\widetilde{\varphi}_{r}]^{\top}, \vspace{-1mm}
	\end{equation}
	where
	\begin{equation}
	\widetilde{\varphi}_{i} := \left\{
	\begin{array}{ll}
	\varphi_{i}, \qquad & \text{~for~}\varphi_{i} \notin \partial g_{i}(0);
	\\[0.5mm]
	0, & \text{~for~}\varphi_{i} \in \partial g_{i}(0).
	\end{array}
	\right.
	\end{equation}
	It is clear that the projection $\Pi_{\L}$ does not change support, and 
	is hence smooth in $\phib$.
	Consequently, the composition $\Pi_{\L}(\widehat{\phib}(\D))$ is a smooth function in $\D$,
	and is the unique solution of the sparse regression problem for given $(\D, \x)$ 
	with $\D \in \mathfrak{U}_{\epsilon}$ by Lemma~\ref{lem:01}, 
	i.e., $\Pi_{\L}(\widehat{\phib}(\D)) = \phib^{*}(\D)$. 
	Thus, the result follows.
\end{proof}

\noindent Finally, the proof of Proposition~1 is straightforward.
\begin{proof}
	For each sample $\x_{i}$ with $i = 1, \ldots, n$, let us denote by $\phib_{i}^{*}$
	the sparse representation with respect to a common dictionary $\D^{*}$.
	We further denote by $\mathfrak{U}_{\epsilon_{i}}$
	an open neighborhood containing $\D^{*}$, so that the unique
	solution $\phib_{i}^{*}(\D)$ is smooth.
	Then a finite number of unions of sets $\mathfrak{U}_{\epsilon_{i}}$ is a non-empty 
	set as \vspace{-1mm}
	\begin{equation}
	\mathfrak{U}_{\epsilon^{*}} := \bigcap_{i=1}^{n} \mathfrak{U}_{\epsilon_{i}}
	\end{equation}
	with $\epsilon^{*} := \min_{i} \epsilon_{i}$.
	It is straightforward to conclude that the complete collection of 
	$\Phib^{*} = [\phib_{1}^{*},\ldots,\phib_{n}^{*}]$ is smooth in 
	$\mathfrak{U}_{\epsilon^{*}}$.
\end{proof}

In the rest of this appendix, we show a lemma, which enables development of
gradient based algorithms.
\begin{customprop}{2}
	\label{prop:02}
	Let $g \colon \mathbb{R}^{r} \to \mathbb{R}$ satisfy
	Assumption~1, and 
	for a given $\x \in \mathbb{R}^{m}$ and a dictionary 
	$\D^{*} \in \mathbb{R}^{m \times r}$, let $\phib^{*} \in \mathbb{R}^{r}$ be the 
	unique solution to the sparse regression problem as in Eq.~(5). 
	Then, the first derivative of $\phib_{\L}^{*}$ has a 
	close form expression as
	\begin{equation}
	\label{eq:derivative_sparse}
	\operatorname{D}\!\phi_{\L}^{*}(\D_{\L})\H =\!
	\big( K(\D_{\L}) \big)^{\!\!\!^{-1}} \!\! \Big(\! \H^{\top}\! \x \!-\! 
	\big( \H^{\!\top}\! \D_{\L} 
	\!+\! \D_{\L}^{\top} \H \big) \phib_{\L}^{*} \!\Big).\!
	\end{equation}
\end{customprop}
\begin{proof}
	Recall the results from Lemma~\ref{lem:02}, the unique solution 
	$\phib^{*} = [\varphi_{1}^{*}, \ldots, \varphi_{r}^{*}]^{\top} \in 
	\mathbb{R}^{r}$ is smooth in an open neighborhood around $\D^{*}$. 
	Since the support $\L(\x_{i},\D^{*})$ stays unchanged, the vector $\phib_{\L}$
	is smooth in an appropriate open neighborhood containing $\D_{\L}$,
	i.e., we can compute
	\begin{equation}
	\label{eq:critical}
	\nabla_{g_{\L}}(\phib_{\L}) = \D_{\L}^{\top} 
	(\x - \D_{\L} \phib_{\L}).
	\end{equation}
	We then take the derivative on the both sides of Eq.~\eqref{eq:critical} with
	respect to $\D_{\L}$ in direction $\H \in T_{\D_{\L}}\S(m, k)$ as
	\begin{equation}
	\label{eq:derivative2}
	\begin{split}
	\!\!\!\operatorname{D}\! \left(\nabla_{\!g_{\L}}(\phib_{\L}(\D_{\L})) \right)\!\H =\, &
	\H^{\top}\! (\x - \D_{\L} \phib_{\L})\!-\!\D_{\L}^{\top} \H \phib_{\L}\\
	&  \!-\! 
	\D_{\L}^{\top} \D_{\L} \operatorname{D}\phib_{\L}(\D_{\L})\H.
	\end{split}
	\end{equation}
	The left hand side in Eq.~\eqref{eq:derivative2} can be computed by
	\begin{equation}
	\label{eq:hessian}
	\operatorname{D}\! \left(\nabla_{\!g_{\L}}(\phib(\D)) \right)\!\H
	= \mathsf{H} g_{\L}(\phib_{\L})
	\cdot \operatorname{D}\phib_{\L}(\D_{\L}) \H,
	\end{equation}
	where $\mathsf{H} g_{\L} (\phi_{\L}) \in \mathbb{R}^{k \times k}$ is the 
	Hessian matrix of function $g_{\L}$.
	It is positive definite by Assumption~1.
	Substituting Eq.~\eqref{eq:hessian} into Eq.~\eqref{eq:derivative2} leads to
	a linear equation in $\operatorname{D}\phi_{\L}(\D_{\L})\H$ as
	\begin{equation}
	\label{eq:derivative_sparse_imp}
	K(\D_{\L}) \cdot \operatorname{D}\phib_{\L}(\D_{\L}) \H 
	\!=\! \H^{\top} \x - \big( \H^{\top} \D_{\L} + \D_{\L}^{\top} \H \big)
	\phib_{\L}.
	\end{equation}
	where $K(\D_{\L}) := \mathsf{H} g_{\L}(\phib_{\L}) + \D_{\L}^{\top} \D_{\L}$ 
	is positive definite.
	%
	Thus, the closed form expression as in Eq.~\eqref{eq:derivative_sparse} follows.
\end{proof}


\end{document}